\documentclass[runningheads]{llncs}
\usepackage[T1]{fontenc}
\usepackage{graphicx}
\usepackage{amssymb}
\usepackage{booktabs}
\usepackage{amsmath}
\usepackage{xcolor}
\usepackage{multirow}
\usepackage{subcaption}
\usepackage{makecell}

\usepackage[ruled]{algorithm}
\usepackage[noend]{algpseudocode}
\algrenewcommand{\algorithmicrequire}{\textbf{Input:}}
\algrenewcommand{\algorithmicensure}{\textbf{Output:}}
\usepackage{bm}
\algnewcommand\algorithmicinput{\textbf{Symbols:}}
\algnewcommand\Symbols{\item[\algorithmicinput]}
\algrenewcommand\textproc{}

\usepackage{cite}
\usepackage[colorlinks=true, citecolor=blue, linkcolor=blue, urlcolor=blue]{hyperref}

\begin{document}
\title{Learning Safe Agent Behaviour from Human Preferences and Justifications via World Models}
\titlerunning{Learning Safe Behaviour from Preferences \& Justifications via World Models}
%
\author{Ilias Kazantzidis\inst{1}\orcidID{0000-0002-1127-3843} \and
Timothy J. Norman\inst{1}\orcidID{0000-0002-6387-4034} \and
Yali Du\inst{2}\orcidID{0000-0001-5683-2621} \and
Christopher T. Freeman\inst{1}\orcidID{0000-0003-0305-9246}}
\authorrunning{I. Kazantzidis et al.}
%
\institute{University of Southampton, Southampton SO17 1BJ, United Kingdom\\
\email{\{ik3n19,t.j.norman,ctf1\}@soton.ac.uk} \and
King's College London, London WC2R 2LS, United Kingdom\\
\email{yali.du@kcl.ac.uk}}
\maketitle              

\begin{abstract}

We address the problem of safely training an agent policy and deploying a good and safe policy, in settings where the environment dynamics are unknown and no suitable reward function is available. In the context of safety-critical environments, we consider traditional reinforcement learning impractical and resort to the resource of human input. We introduce \textit{DROPJ}, a human-centred method for both safe training and deployment. We first learn a world model (a learned simulator) from a dataset of prior real-world trajectories. A human then \textit{plays the game} in this learned simulator to extract several informative simulated trajectories. From these, we sample pairs of simulated trajectory segments and elicit from a human their preference over these segments, as well as a reason (\textit{justification}) for their choice. We then train a reward model from these justified preferences and use it, together with the world model, to directly deploy the agent using model predictive control. Running real-user experiments, we find that generating informative simulated trajectories from a user significantly reduces the computational cost during training compared to other strategies, and can also improve the performance during deployment. In the context of training within a learned simulator, we show that the use of preferences rather than other types of feedback substantially improves the performance during deployment. We further demonstrate that safety justifications accompanying preferences can significantly enhance safety or prioritise user-prescribed aspects of safety associated with them during deployment.

\end{abstract}

\section{Introduction}
\label{sec:introduction}

Over the last couple of years, Reinforcement Learning (RL) has been proven to be an efficient way of training agents for sequential decision-making \cite{mnih2015human,silver2016mastering,vinyals2019grandmaster}. However, widespread use of end-to-end RL for many useful real-world and safety-critical applications, such as self-driving cars and robotics, is limited. One important reason is that in practical settings that the dynamics function of the environment may be unknown, and thus we do not have access to a reliable (near-)perfect simulator, safety becomes a serious concern if training has to take place in the real world. For instance, an autonomous RL car would have to bump into a post in order to learn not to bump into a post. This problem, known as safe exploration, has been targeted by several methods in the Safe RL literature \cite{garcia2015comprehensive,achiam2017constrained,chow2018lyapunov,cheng2019end,turchetta2020safe,wachi2023safe,manishapo}. Although these methods try to provide safety guarantees by setting up constraints, they assume access to a well-defined environmental reward function. However, in complex domains, \textit{reward functions} may also be unavailable or poorly aligned with human objectives \cite{leike2018scalable}. Learning a model from human feedback \cite{knox2009interactively}, particularly a reward model from pairwise preferences \cite{christiano2017deep}, offers a promising solution, as the learned policy is better aligned with human objectives. Yet, this alone does not tackle safe learning, as the agent would still encounter unsafe states if trained in the real world, risking damage. Moreover, methods that the human continuously oversees the agent during training \cite{saunders2018trial,kazantzidis2022train} would typically impose an infeasible human burden, as well as their use would be impractical in domains where continuous operation is essential.

\begin{figure}[t]
\centerline{\includegraphics[width=1.3\textwidth]{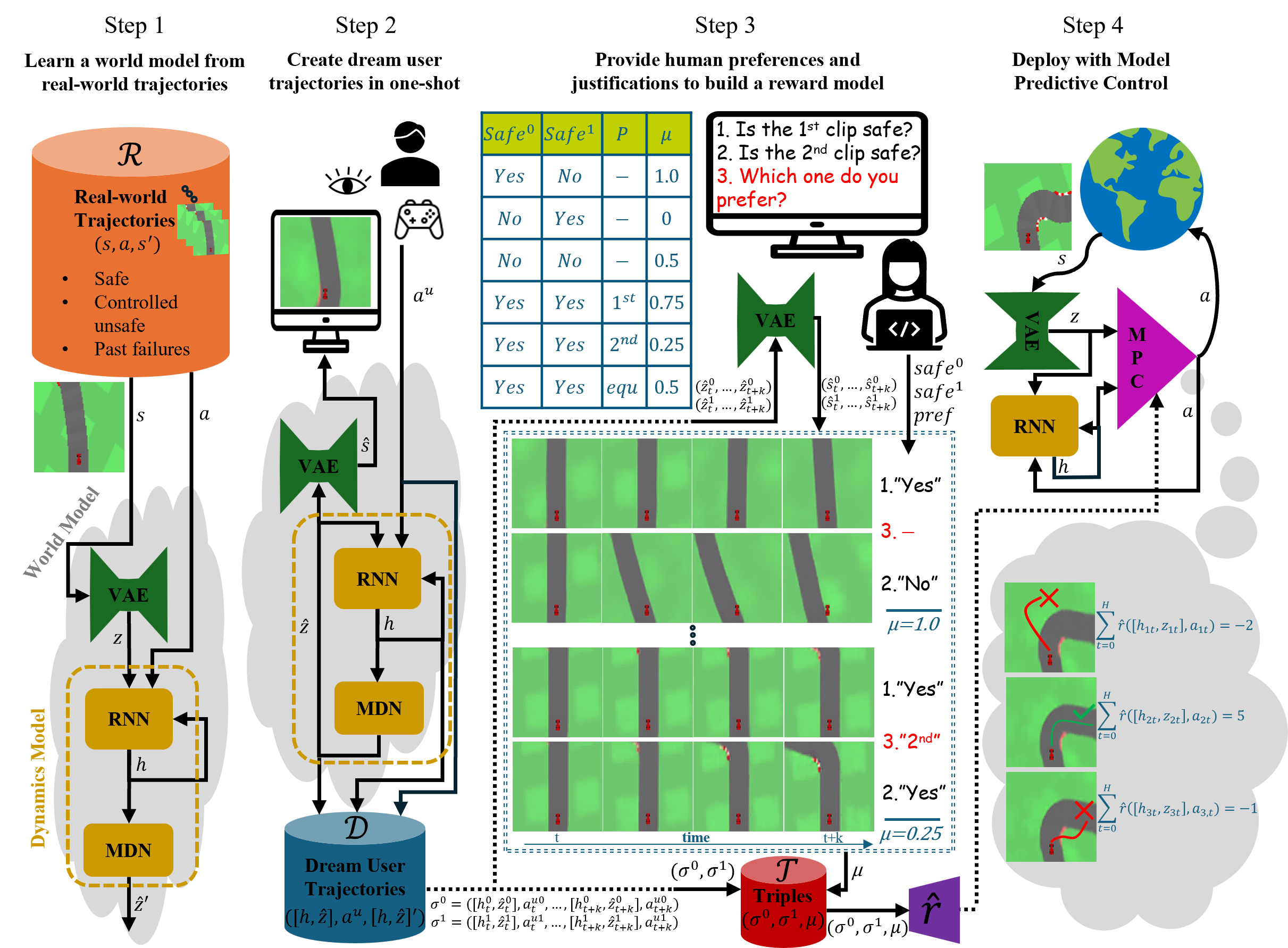}}
\caption{The DROPJ framework \cite{kazantzidis2026safe}. In Step 3, we show the case with a single safety justification (associated with unsafe steps of the car driving onto grass).}
\label{fig:method}
\end{figure}

An approach that tries to tackle the issue of safe learning under both unavailable dynamics and reward functions is to build a dynamics model from an offline dataset of past real-world reward-free trajectories --- effectively an imperfect simulator (also referred to as a \textit{world model}) --- and construct a reward model from human feedback on trajectories extracted safely inside that simulator. Using this reward model, either an RL policy could be learned inside that simulated environment and transferred to the real world, or, better still, an agent could be deployed directly to the real world using Model Predictive Control (MPC) \cite{garcia1989model}. However, works that have followed this approach \cite{reddy2020learning,rahtz2022safe} have two limitations. First, an iterative, computationally expensive process, of generating trajectory segments, seeking human feedback on them, training a reward model, and then generating new segments based on the improved reward model, is used to generate informative queries for human feedback. This hypothetical query generation is particularly expensive and difficult to scale in complex tasks, as the user would have to stay on a long-lasting loop providing feedback. Second, although training is carried out safely inside the imperfect simulator, unsafe-state violations can still occur during deployment if safety requirements are not captured by the reward model.

We present \textbf{DROPJ}: \textbf{D}ream-World \textbf{R}eward Learning from \textbf{O}ne-Shot Human \textbf{P}references and \textbf{J}ustifications \cite{kazantzidis2026towards}. As illustrated in Figure \ref{fig:method}, in Step 1, DROPJ learns a world model (a learned simulator) from prior real-world reward-free trajectories (episodes). In Step 2, a user `plays the game' for a few minutes in order to generate \textit{dream user trajectories}, i.e.\ not real but trajectories extracted with the learned simulator. Here, exploration and unsafe-state visits are allowed. In Step 3, pairs of segments are sampled from the dream user trajectories and a user expresses their preference and a \textit{justification} for it. A reward model is then learned with preferences and justifications similar to \cite{christiano2017deep}. In Step 4, equipped with the learned dynamics and reward models, we directly deploy the agent with MPC. Thus, DROPJ solves the two limitations of the previous paragraph: in Step 2, it facilitates the generation of trajectory segments by leveraging the understanding of a user who safely takes actions via the learned simulator, and creates through it diverse examples in \textit{one-shot}; and in Step 3, it tackles safety during deployment, using \textit{safety justifications} to accompany preferences.

Our method, in contrast to many previous works, \textbf{assumes no access to predefined environment dynamics or handcrafted reward functions}, reflecting realistic real-world conditions. We evaluate, with real users, not only performance and safety, but also human burden and computational cost. The two core contributions are:
\begin{enumerate}
\item A technique, we call \textit{one-shot}, where examples for human feedback are generated within a learned simulator from a user in `one-shot'. This technique improves both computational cost (time) and performance compared to previous techniques.
\item The extension of the \textit{justifications} idea in three important ways: (i) the user feedback (preferences with justifications) is provided \textit{offline} over trajectory segments, instead of online over state-action pairs; (ii) preferences with justifications are used to train a reward model, instead of directly optimising a policy; and (iii) multiple safety justifications can be provided, instead of only a single one. We find that justifications, during deployment, can enhance user-prescribed aspects of safety, with potential trade-offs in performance or other aspects.\footnote{This paper is an extended and substantially revised version of the conference paper \cite{kazantzidis2026safe}, which was presented at ICAART 2026 (\url{https://icaart.scitevents.org/?y=2026}). Figures \ref{fig:obst_car}, \ref{fig:impact_preferences}--\ref{fig:mpc_dreams}, and Tables \ref{tab:equivalence}, \ref{tab:times} and \ref{tab:mdnrnn} are reproduced from \cite{kazantzidis2026safe}, while Figures \ref{fig:method}, \ref{fig:obs_env_ex}--\ref{fig:tsne}, and Tables \ref{tab:vae_configuration}, \ref{tab:sparse_model} and \ref{tab:pref_mod} are adapted from \cite{kazantzidis2026safe}, as indicated in their captions.}
\end{enumerate}

\section{Related Work}
\label{sec:related_work}

Our research regarding safety has a similar motivation with the large body of literature in Safe RL which does not rely on human input \cite{garcia2015comprehensive,gu2024review,achiam2017constrained,ge2019safe,chow2018lyapunov,cheng2019end,luo2021learning,turchetta2020safe,wachi2021safe,yu2022towards,wachi2023safe,manishapo}. These methods formulate safety requirements as constraints that must hold during both training and deployment, using the Constrained Markov Decision Process (CMDP) framework \cite{altman1999constrained}. However, they assume, in contrast to our approach, the availability of an explicitly specified (well-defined) environmental reward that defines the optimisation objective. Our method also connects with techniques that integrate human input \cite{ng2000algorithms,knox2009interactively,brown2018risk,warnell2018deep,brown2019extrapolating,guo2022mtirl,li2025reinforcement}, and particularly human preferences, the most common type of feedback for aligning agent behaviour with human goals \cite{christiano2017deep,biyik2018batch,lee2021pebble,park2022surf,kim2023preference,karlaus2025tempo,liu2026safe}. However, most of these assume learning within perfect simulators (known dynamics function), without studying safety. Conversely, methods that the user actively oversees the agent helping it to avoid unsafe states while learning \cite{saunders2018trial,goecks2019efficiently,frye2019parenting,kazantzidis2022train} struggle with scaling and generalisation. Learning from human feedback has also received remarkable attention in large language models over the last couple of years, including safety aspects \cite{ouyang2022training,gu2023human,rafailov2023direct,shi2024human,lou2024safe,dai2024safe}.

Our work is related to \textit{world models}. As explained in \cite{ha2018recurrent}, a world model can be built from an offline dataset of reward-free trajectories, and consists of a \textit{vision} component that compresses a high-dimensional image input into a condensed latent-space representation (encoder), and a \textit{memory} component that predicts the next latent state based on historical information (predictor), essentially forming a dynamics model inside a virtual environment. It is a learned model of the real world (effectively a learned simulator), with the special characteristic of having the two aforementioned components. As such, we also refer to it here as \textit{dream world} or even \textit{dream} (depicted in Step 1 of Figure \ref{fig:method} with a cloud). With the dynamics model we can: simulate a lot of actions internally without needing costly real-world interactions; and also plan, i.e.\ `imagine' future states and evaluate strategies before acting.\footnote{We note that a world model is not a perfect simulator or a video generation system that models too many details and is therefore poor for planning; instead, it predicts in an abstract representation space and enables efficient planning \cite{lecun2022path}.} Furthermore, in \cite{ha2018recurrent} it is shown that an RL policy could be trained in the dream world and successfully transferred to the real. However, the RL reward function was defined as a simple time-dependent function. In contrast, we learn a fresh reward model from human feedback in the dream. More recent developments in world models have enhanced long-term prediction stability and scalability in complex environments \cite{hafner2019dream,hafner2019learning,hafner2020mastering,as2022constrained,hafner2025mastering,huang2024safedreamer}. However, instead of using an offline dataset of prior trajectories as in our case, they rely on potentially-unsafe online data collection from interactions of the agent in the real environment. Moreover, they often assume that the environment can provide handcrafted rewards, instead of having to learn a reward model from human input as in our setting. Similarly, approaches in offline RL, such as \cite{rafailov2021offline}, demonstrate that policies can be trained entirely from offline datasets using latent space models and transferred to the real environment, though they assume the handcrafted rewards are available in the dataset trajectories. This differs once again from our approach, where we learn a reward model from human input. Finally, we note that although world models constitute a powerful and promising tool for prediction and planning, their training typically demands vast amounts of data and substantial computational resources. For instance, the state‑of‑the‑art V‑JEPA 2 model \cite{assran2025v} was pretrained on over one million hours of video and required refinement on robot interaction data.

The method closest to ours is ReQueST \cite{reddy2020learning,rahtz2022safe}, which involves learning a dynamics model from an offline dataset of prior reward-free trajectories, and then learning a reward model from human feedback on trajectory segments generated through the dynamics model. However, in ReQueST, the iterative process of generating trajectory segments, providing human feedback on them, training a reward model, and then generating new segments based on the improved reward model, is continuously repeated. The trajectory segments in order to be informative are synthesised by optimising proxies that maximise the value of information. This is conducted by gradient-based optimisation, and is computationally expensive, and as such hard to scale in complex tasks that require long-length segments for planning. Thus, the user has to remain on a long-lasting loop. Also, MoP-RL \cite{liu2023efficient} adopts the idea from ReQueST using preferences. However, it maintains the same issue as it follows the same iterative process with the difference that it generates trajectory segments through model-based RL in the simulated environment using the Cross-Entropy Method (CEM) \cite{botev2013cross}. That adds another significant overhead. Instead, our approach expedites the trajectory segment generation by enabling the user to interact with the learned simulator and safely generate informative trajectories in an one-shot manner. Additionally, the above methods, although they learn safely, they do not target safety during deployment. In contrast, DROPJ does that by accompanying preferences with safety justifications.

\section{Preliminaries}
\label{sec:preliminaries}

\subsubsection{Learning a Reward Model from Preferences}
In the traditional approach of learning a reward model from human preferences \cite{christiano2017deep}, a segment $\sigma$ of length $k$ is defined as a sequence of interleaved states and actions $(s_t, a_t, \dots, s_{t+k}, a_{t+k})$. For any two segments, $\sigma^0$ ($1^{\mathit{st}}$ segment) and $\sigma^1$ ($2^{\mathit{nd}}$ segment), a human provides a preference $P\in \{ 1^{\mathit{st}}, \mathit{equ}, 2^{\mathit{nd}}\}$, where $P=1^{\mathit{st}}\iff \sigma^0 \succ \sigma^1$ (preference for $\sigma^0$), $P=2^{\mathit{nd}}\iff \sigma^1 \succ \sigma^0$ (preference for $\sigma^1$) and $P=\mathit{equ}\iff \sigma^0 \sim \sigma^1$ (indifference).  $P$ is mapped to the preference ground-truth label $\mu \in \{0, 0.5, 1\}$, where $1$ and $0$ denote preference for $\sigma^0$ and $\sigma^1$, respectively, and $0.5$ denotes indifference. The answer is stored in a dataset $\mathcal{T}$, as a triple $(\sigma^0, \sigma^1, \mu)$. Using the Bradley-Terry model \cite{bradley1952rank}, the predicted preference probabilities can be modelled as:
\[
p[P= 1^{\mathit{st}}] =
\frac{\exp \sum_t \hat{r}(s_t^0, a_t^0)}{\sum\limits_{i \in \{0,1\}} \exp \sum_t \hat{r}(s_t^i, a_t^i)} \ \ \text{and} \ \
p[P=2^{\mathit{nd}}]=\frac{\exp \sum_t \hat{r}(s_t^1, a_t^1)}
{\sum\limits_{i \in \{0, 1\}} \exp \sum_t \hat{r}(s_t^i, a_t^i)}.
\]
\noindent Then, the preference reward model $\hat{r}=\hat{r}(s,a)$ can be learned by minimising the cross-entropy loss between these predicted probabilities and the preference labels:

\begin{equation}
\mathcal{L}^{\hat{r}}  = -  \sum_{ (\sigma^0, \sigma^1, \mu) \in \mathcal{T}} 
 \mu \log 
p[P= 1^{\mathit{st}}]
+ 
(1 - \mu) 
\log 
p[P= 2^{\mathit{nd}}]
\label{eq:rew_loss}
\end{equation}

\noindent Equation \ref{eq:rew_loss} will be used in Step 3 of DROPJ (Figure \ref{fig:method}) to train a reward model, but this time with justifications being integrated in $\mu$, and segments extracted from the dream world.

\subsubsection{Justifications in Prior Work}
\label{par:just_prior}
Justifications over preferences $J$ provide effectively an explanation for the preferred choice \cite{kazantzidis2022train,kazantzidis2022learning}. This results in a more information-rich preference label $\mu$, allowing for a larger set of discrete values within the $[0,1]$ range. In the case of a single justification (instead of multiple ones) related to safety (in addition to a \textit{default} justification that always exists) and state-action pairs (instead of segments), the user, besides a preference $P$, provides a $J\in \{\mathit{warning}, \mathit{no\_warning}\}$, where $\mathit{warning}$ means at least one proposed action is unsafe, and $\mathit{no\_warning}$ means both actions are safe. If $w_s, w_{\mathit{def}}\in(0.5,1]$ are the weights for the safety and the default justification (e.g.\ preference because the action simply takes the agent closer to a goal state), respectively, then a mapping $P \times J \rightarrow \mu$, where $\mu \in \{1-w_s, 1-w_{\mathit{def}}, 0.5, w_{\mathit{def}}, w_{s}\}$ and  $w_{s}> w_{\mathit{def}}$, can promote safer behaviour than the typical $P \rightarrow \mu$ mapping. For instance, a sensible choice of weights at $(w_{s}=1, w_{\mathit{def}}=0.75)$ would lead to: $\mu=1$ denoting total preference for the $1^{\mathit{st}}$ action because the $2^{\mathit{nd}}$ action is unsafe (i.e.\ $P=1^{\mathit{st}},J=\mathit{warning}$); $\mu=0.75$ denoting simple preference for the $1^{\mathit{st}}$ action because of a default reason (i.e.\ $P=1^{\mathit{st}},J=\mathit{no\_warning}$); and $\mu=0$ and $\mu=0.25$, conversely. Step 3 of DROPJ uses justifications with three key extensions: (i) user feedback (preferences with justifications) is provided offline over dream trajectory segments, instead of online over state-action pairs; (ii) $\mu$ is used as a preference label to train a reward model (Equation \ref{eq:rew_loss}), instead of directly optimising a policy; and (iii) multiple safety justifications can be provided, instead of only a single one.

\section{(Obstacle) Car Racing and Problem Setting}
\label{sec:car_problem}

In this research, we investigate safety using the Car Racing (CR) environment illustrated in Figure \ref{fig:car} \cite{brockman2016openai} and a more challenging version we created, which we call Obstacle Car Racing (OCR) (Figure \ref{fig:obst_car}). Both state and action spaces are continuous with states being RGB images and actions controlling steering, acceleration (aka.\ gas) and braking. The track is randomised in each trial, and the reward structure we used gives +10 for stepping on a new patch (tile) and -1 for veering onto grass or the kerb. This reward structure is not used for training, but \emph{only} for evaluation. In OCR, we include static potholes (aka.\ chuckholes) where the car slows down when it drives over them, and other `scripted' cars it can collide with. The goal is to visit as many new patches as possible, with the fewest number of steps on the grass/kerb, and fewest accidents with chuckholes and scripted cars.

\begin{figure}
     \centering
     \begin{subfigure}[b]{0.49\textwidth}
         \centering
         \includegraphics[width=\textwidth]{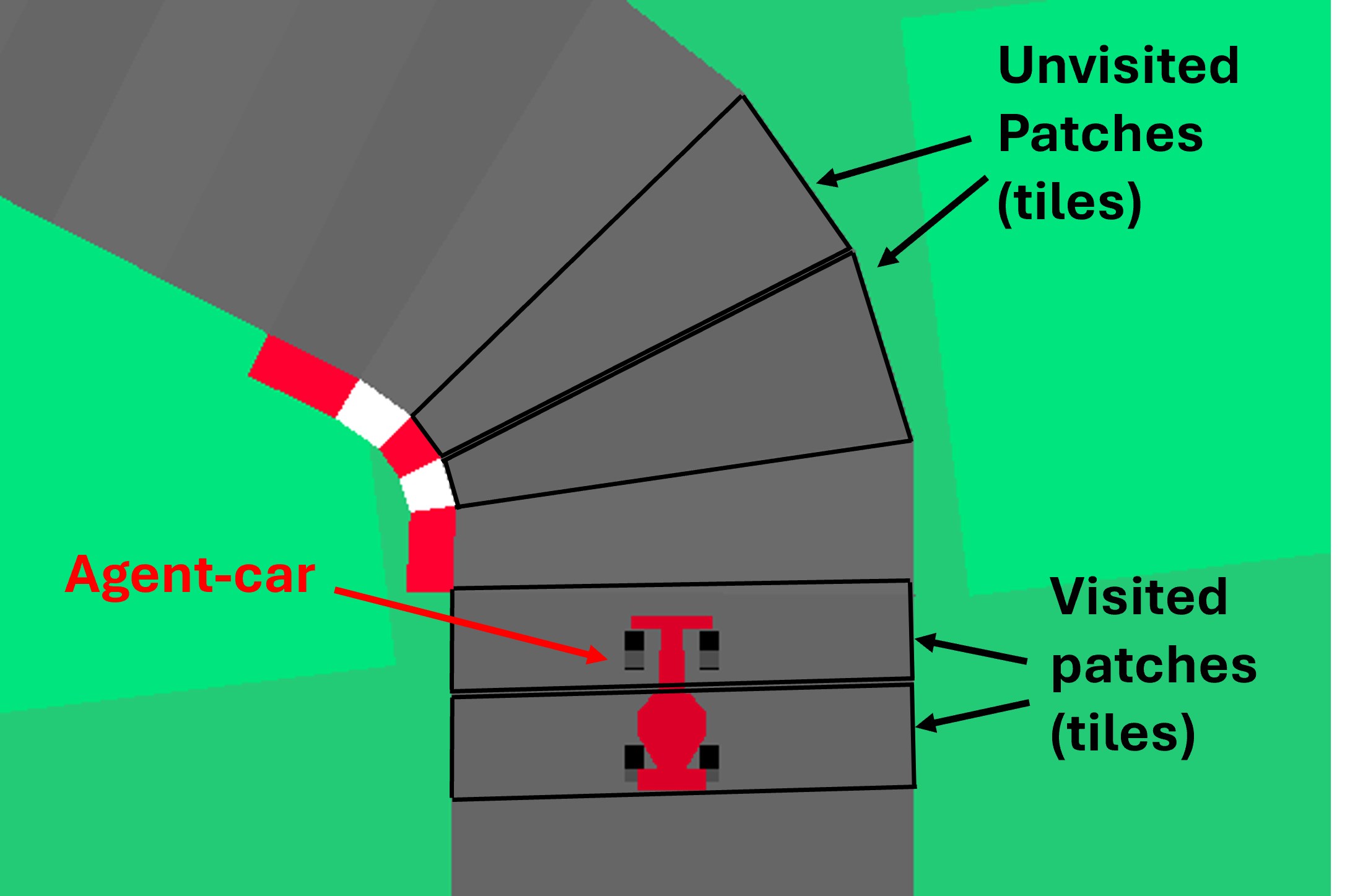}
         \caption{Car Racing environment}
         \label{fig:car}
     \end{subfigure}
    \hfill     
     \begin{subfigure}[b]{0.49\textwidth}
         \centering
         \includegraphics[width=\textwidth]{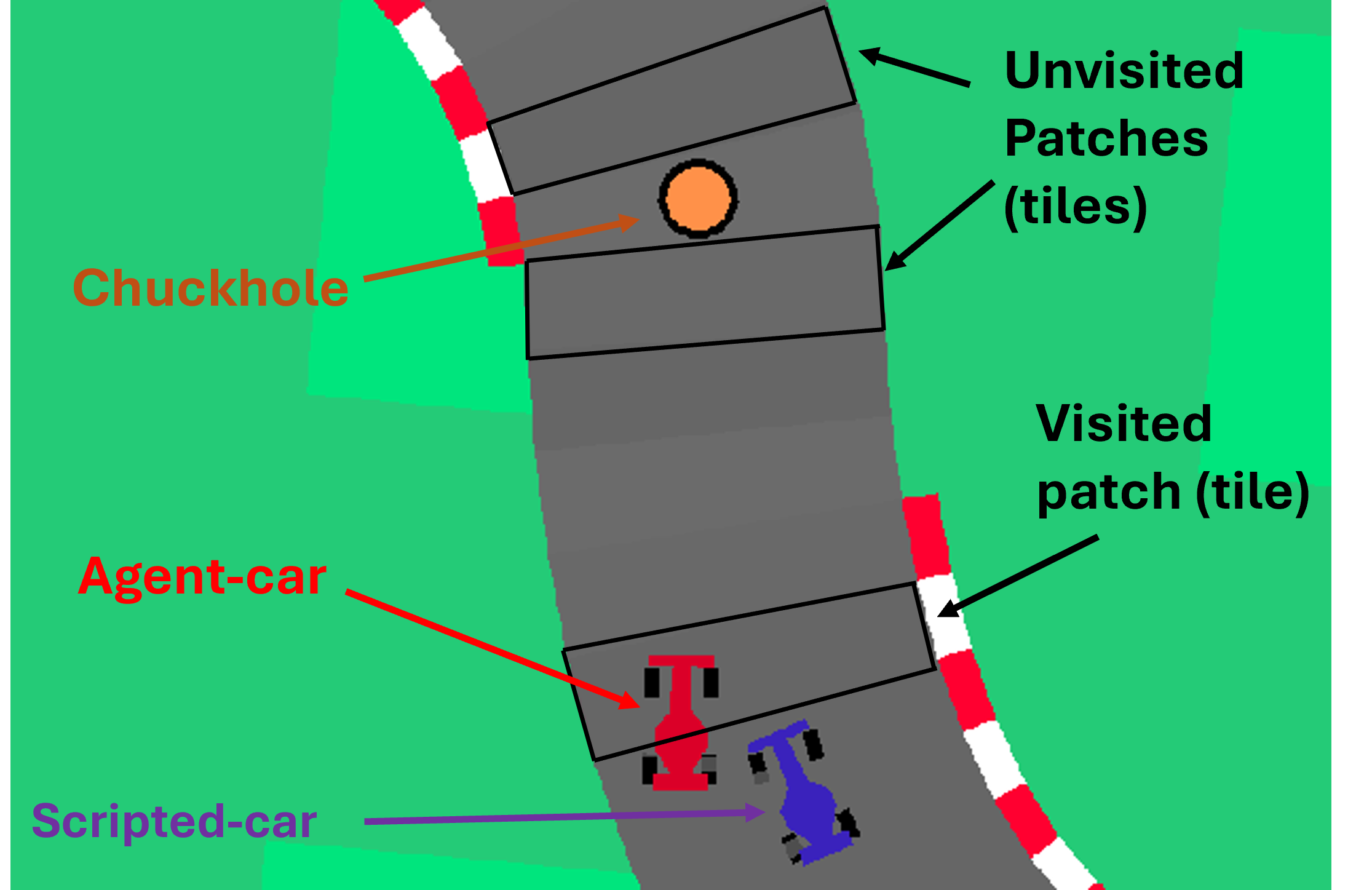}
         \caption{Obstacle Car Racing environment}
         \label{fig:obst_car}
     \end{subfigure}
     \hfill
        \caption{Environments used in this research \cite{kazantzidis2026safe}.}
        \label{fig:dropj_envs}
\end{figure}

Our \textbf{problem setting} is based on the following assumptions. The unsafe states (driving over grass or chuckholes, or hitting other cars) are unknown to the agent, and we have access to neither the state transition dynamics function of the environment (i.e.\ there is no way to sample experience for free), nor to a suitable (oracle) reward function. The state and action spaces are continuous, and the input is image-based. We assume recordings of prior real-world reward-free trajectories and access to a user with reasonable understanding of the environment who can answer queries \textit{offline} generated through a learned simulator. Querying the user is expensive, and hence we aim to bother them as little as possible. 

\section{DROPJ}
\label{sec:dropj}

Our aim in DROPJ is to build a reward model from human input and use it for planning during deployment. Table \ref{tab:notation_dropj} provides a quick reference for the symbols used in the following sections. DROPJ is split into four steps illustrated in Figure \ref{fig:method} and detailed in Algorithm \ref{alg:dropj}:
\begin{description}
    \item[Step 1] Learn a world model from real-world trajectories
    \item[Step 2] Create dream user trajectories in \textit{one-shot}
    \item[Step 3] Learn a reward model from user preferences and \textit{justifications}
    \item[Step 4] Deploy with MPC
\end{description}

\begin{table} [t]
\centering
\caption[Symbols used in DROPJ]{Symbols used in DROPJ.}
\label{tab:notation_dropj}  
\resizebox{\textwidth}{!}{ 
\begin{tabular}{lll}
\toprule
\textbf{Symbol} & \textbf{Description} & \textbf{Domain/Range} \\
\midrule
$s$ & Agent's state & $\mathcal{S} = \mathbb{R}^n$ \\
$a$ & Agent's action & $\mathcal{A}$ (continuous) \\
$\mathcal{R}$ & Offline dataset of prior real-world trajectories & - \\
$M$ & Number of prior real-world trajectories in $\mathcal{R}$ & $\mathbb{N}$ \\
$L$ & Length of a trajectory in $\mathcal{R}$ or $\mathcal{D}$ & $\mathbb{N}$ \\
$z$ & Latent state & $\mathcal{Z} = \mathbb{R}^d$ \\
$h$ & Memory state & $\mathcal{H}=\mathbb{R}^m$  \\
$\hat{s}$  & Predicted state & $\mathcal{S} = \mathbb{R}^n$  \\
$\hat{z}$ & Predicted latent state & $\mathcal{Z} = \mathbb{R}^d$ \\
$u$ & User & - \\
$a^{u}$ & User's action & $\mathcal{A}$ (continuous) \\
$\mathcal{D}$ & Dataset of dream user trajectories & - \\
$T$ & Number of dream user trajectories in $\mathcal{D}$ & $\mathbb{N}$ \\
$\hat{r}$ & Reward model & $\mathbb{R}$ \\
$(\sigma^{0},\sigma^{1})$ & Pair of dream trajectory segments & - \\ 
$K$ & Number of pairs of dream trajectory segments & $\mathbb{N}$ \\
$k$ & Length of dream trajectory segments & $\mathbb{N}$ \\
$P$ & Human preference & $\{1^{\mathit{st}}, \mathit{equ}, 2^{\mathit{nd}}\}$ \\
$\mathit{Safe}^0, \mathit{Safe}^1$ & Justifications for single safety justification case & $\{\mathit{Yes}, \mathit{No}\}$ \\
$w_s, w_{\mathit{def}}$ & Justification weights for single safety justification case & $(0.5, 1]$ \\
$N$ & Number of justifications for multiple justifications case & $\mathbb{N}$ \\
$J_1,\dots,J_N$ & Justifications for multiple justifications case & $\{0,1\}$ \\
$w_1, \dots, w_N$ & Justification weights for multiple justifications case & $(0.5, 1]$\\
$\mu$ & Preference label for single safety justification case & $\{1-w_s, 1-w_{\mathit{def}}, 0.5, w_{\mathit{def}}, w_{s}\}$ \\
& \phantom{Preference label} for multiple justifications case & $\{1 - w_1,  \dots, 1-w_N, 0.5, w_N, \dots,  w_1 \}$  \\
$\mathcal{T}$ & Dataset with triples $(\sigma^0,\sigma^1,\mu)$ & - \\
$S$ & Number of sensible random action sequences for MPC & $\mathbb{N}$ \\
$H$ & Planning horizon & $\mathbb{N}$ \\
$R$ & Number of steps executed until replanning & $\mathbb{N}$ \\

\bottomrule
\end{tabular}}

\end{table}

\begin{algorithm}
   \fontsize{8.1pt}{9.2pt}\selectfont
   \caption{DROPJ}
   
\begin{algorithmic}[1]

  \Require \
   
   \noindent \textbf{Training (Steps 1-3):} offline dataset $\mathcal{R}$ of $M$ real-world trajectories, each of $L$ transitions $(s, a, s')$, with $s, s'\in\mathcal{S}=\mathbb{R}^n, a\in\mathcal{A}$; user $u$; number of dream user trajectories $T$; number of queries to the user (pairs of dream trajectory segments created) $K$; length of segment $k$; justification weights $w_s, w_{\mathit{def}} \in(0.5, 1]$ for single safety justification, or $w_1, \dots, w_N \in(0.5, 1]$ for multiple justifications 
   
   \noindent \textbf{\textbf{Deployment (Step 4)}:} planning horizon $H$; number of steps until replanning $R$; number of sensible random action sequences $S$
   
   \Ensure reward model $\hat{r}$ (\textbf{training}); MPC policy $\pi_{\mathit{mpc}}$ (\textbf{deployment})

\vspace{1ex}
   
   \State \textbf{// Step 1: Learn a World Model}
   
   \State Train VAE encoder $f_{\mathit{enc}}\!:\! \mathcal{S}\! \to\! \mathcal{Z}$ and decoder $f_{\mathit{dec}}\!:\! \mathcal{Z} \!\to \!\mathcal{S}$ from $\mathcal{R}$, with $\mathcal{Z}=\mathbb{R}^d$ and $d \ll n$
   
   \State Train dynamics model from $\mathcal{R}$: an RNN $f_{\mathit{RNN}}\!:\! \mathcal{H} \times \mathcal{Z} \times \mathcal{A} \to \mathcal{H}$ and an MDN $f_{\mathit{MDN}}\!:\! \mathcal{H} \to \mathcal{Z}$

\vspace{1ex}
   
   \State \textbf{// Step 2: Collect Dream User Trajectories in One-shot}
   \For{\_ from $1$ to $T$}
      \For{\_ from $1$ to $L$}
      
       \State Reconstruct frame of last predicted latent state and display it to monitor $\hat{s}=f_{\mathit{dec}}(\hat{z})$

      \State Let user $u$ take an action $a^{u}$ 
      
       \State Run a step in simulation and generate next latent state $\hat{z}'=f_{\mathit{MDN}}(f_{\mathit{RNN}}(h, \hat{z}, a^{u}))$

      \State Store dream transition $\mathcal{D} \gets \mathcal{D} \cup ([h, \hat{z}], a^{u}, [h, \hat{z}]')$

      \EndFor
   \EndFor

\vspace{1ex}
    
   \State \textbf{// Step 3: Learn a Reward Model from Preferences and Justifications}
   \For{\_ from $1$ to $K$}
   
        \State Sample a pair of segments (uniformly) from the dream user trajectories $(\sigma^{0}, \sigma^{1}) \sim \mathcal{\mathcal{D}}$ 
        \State Reconstruct segments and show clips to user $(\hat{s}_t^i, \dots, \hat{s}_{t+k}^i)\!=\!f_{\mathit{dec}}((\hat{z}_t^i, \dots, \hat{z}_{t+k}^i)), \ \!\!i\!=\!0,\!1$
        \State Query user based on: Table \ref{tab:equivalence} to infer $\mu\!=\!\mu(\mathit{Safe}^0,\mathit{Safe}^1,P,w_s,w_{\mathit{def}})$ for single safety justification; Equation \ref{eq:gen_just} to infer $\mu\!=\!\mu(J_1,\dots,J_N,P,w_1,\dots,w_N)$ for multiple justifications\label{line:drop}
        \State Store the triple $\mathcal{T} \gets \mathcal{T} \cup (\sigma^0,\sigma^1,\mu)$
   \EndFor
    \State Train the reward model by minimising with gradient descent Equation \ref{eq:rew_loss}

\vspace{1ex}

   \State \textbf{// Step 4: Deployment with MPC}

   \State $s\gets$ initial state
   \State $\mathit{done}=\mathit{False}$
   \State $h \gets$ initialise with  0's
   \State \textbf{//} $\mathit{return}=0$ for evaluation
   \While{$\mathit{not \ done}$}
       \State $z=f_{\mathit{enc}}(s)$
        \State \textbf{//} Running sample-based MPC planning
        \If{$\mathit{plan\_step} \% R == 0$}
             \State Generate $S$ sensible random action  sequences: \State \hspace{1em} $\mathcal{A}_{\mathit{mpc}} = \{ \mathbf{a}_1, \mathbf{a}_2, \dots, \mathbf{a}_S \}$, where $\mathbf{a}_i = (a_{i0}, a_{i1}, \dots, a_{iH})$
             \For{each $\mathbf{a}_i \in \mathcal{A}_{\mathit{mpc}}$}
                \State $h_{i0}=h, \ z_{i0}=z$
                \State Use the dynamics model to simulate the trajectory:
                 \State \hspace{1em}$([h_{i0}, z_{i0}],a_{i0}, \ldots, [h_{iH}, z_{iH}],a_{iH})$
                 \State Compute the predicted return of the trajectory: 
                 \State \hspace{1em} $R_i = \sum_{t=0}^{H} \hat{r}([h_{it}, z_{it}], a_{it})$
            \EndFor
            \State Select the action sequence of the trajectory with the highest predicted return:
            \State \hspace{1em}$\mathbf{a}^* = \arg\max_{\mathbf{a}_i \in \mathcal{A}_{\mathit{mpc}}} R_i$                
            \State Keep the first $R$ actions of the chosen action sequence:
            \State \hspace{1em} $\pi_{\mathit{mpc}}(s) = (a^*_1, a^*_2, \ldots, a^*_R)$
            \State Reset the planning step counter:
            \State \hspace{1em} $plan\_step=0$
        \EndIf
        \State$a = \pi_{\mathit{mpc}}(s)[\mathit{plan\_step}]$ \textbf{//} Execute the next action from the planned sequence
        \State$\mathit{plan\_step} = \mathit{plan\_step} + 1$ 
        \State $s, \mathit{done} \gets \mathit{env\_step}(a)$ \textbf{//} return $\mathit{reward}$ for evaluation
        \State  \textbf{//} $\mathit{return}=\mathit{return} + \mathit{reward}$ for evaluation
        \State $h=f_{\mathit{RNN}}([h, z], a)$
   \EndWhile

\end{algorithmic}
\label{alg:dropj}
\end{algorithm}

\subsection{Learning a World Model}
\label{sec:learn_world_model}

More formally, in \textbf{Step 1} (Line 1 in Algorithm \ref{alg:dropj}), we assume an available offline dataset $\mathcal{R}$ of $M$ past real-world trajectories of $L$ state-action transitions $(s, a, s')$ each, with states $s$ and next states $s'$ in a high-dimensional, continuous space (such as images) $\mathcal{S} = \mathbb{R}^n$ and actions $a$ also in a continuous space $\mathcal{A}$.\footnote{The image-based trajectories extracted for Step 1 from the (O)CR original environment are treated as analogous to trajectories that could be collected in a real-world scenario, such as by teleoperating a physical robot with a top-down camera --- thus, we use the term \textit{real-world} trajectories, to distinguish them from \textit{dream-world} trajectories that will be extracted in Step 2 inside the learned world model.} The overall number of state-action transitions $M\times L$ should be large enough and $\mathcal{R}$ should have enough state and action coverage to build a sufficiently good world model.\footnote{A sufficiently good world model is one that captures the environment dynamics well enough to support planning during deployment. In Section \ref{sec:world_model_abl}, we describe how we determined that.} For this, apart from safe trajectories, it should also include past failures or some controlled unsafe trajectories (e.g.\ a vehicle gently stopping against foam barriers or a robot dropping plastic instead of porcelain cups).\footnote{For environments with catastrophic states, such as a car close to a cliff, unsafe states are considered the ones `close to the cliff'. Thus, they can still be observed via controlled procedures.} Initially, using $(s,s')$, the vision component of the world model is trained, which is a variational auto-encoder (VAE) \cite{kingma2013auto} with: an encoder $f_{\mathit{enc}}:\mathcal{S} \to \mathcal{Z}$, transforming $s$ to a latent state $z \in \mathcal{Z}$, where $\mathcal{Z}=\mathbb{R}^d$  is a low-dimensional latent space ($d \ll n$); and a decoder $f_{dec}:\mathcal{Z} \to \mathcal{S}$, reconstructing an image $s$ from a latent state $z$ (Line 2). Then, the encoded $(s, a, s')$ transitions from $\mathcal{R}$ train the memory component, which is a Recurrent Neural Network (RNN) \cite{hochreiter1997long} $f_{\mathit{RNN}}:\mathcal{H} \times \mathcal{Z} \times \mathcal{A} \to \mathcal{H}$, providing the memory state $h\in\mathcal{H}=\mathbb{R}^m$. Further, $h$ is fed into a Mixture Density Network (MDN) \cite{bishop1994mixture} $f_{\mathit{MDN}}:\mathcal{H} \to \mathcal{Z}$. The MDN models a probability distribution over the next latent state $p(z' \mid h)$ by parameterising a mixture of Gaussians, from which the next latent predicted state $\hat{z}'$ can be sampled in its output. If $\hat{z}'$ is fed back as input to the RNN, we essentially have a learned dynamics model (\textit{predictor}) in the latent space $\mathcal{Z}$ (Line 3).\footnote{We note that our focus in this work was the evaluation of contributions from Steps 2 and 3 of DROPJ, so we decided to use a simple architecture of the dynamics model.}

\subsection{Collecting Dream User Trajectories (the One-shot Technique)}
\label{sec:learning_oneshot}

In \textbf{Step 2} (Line 4), instead of creating trajectory segments for user feedback via an iterative computationally-expensive process, we use a more efficient idea. Since a prediction $\hat{z}$ can be fed, not only back to the RNN as the next latent state, but also back to the VAE decoder to be reconstructed as $\hat{s}=f_{\mathit{dec}}(\hat{z})$ and observed by a user $u$ in a monitor (Line 7), $u$ can take actions $a^{u}$ in the dream world (`play the game' in the learned simulator; Lines 8--9), and in one-shot create $T$ dream user trajectories of $L$ state-action transitions $([h,\hat{z}],a^{u},[h,\hat{z}]')$, stored in a dataset $\mathcal{D}$ (Line 10). The memory $h$ should also be stored, by being concatenated with $\hat{z}$ to form the dream state $[h,\hat{z}]$. During this step, $u$ can explore states that would otherwise be tough to discover, gathering a sufficiently diverse set of examples. They should also deliberately visit unsafe states. That is, since $u$ can use their understanding of the task, the dream user trajectories in $\mathcal{D}$ are gathered in one-shot and quickly (no need for a large $T$), and they are also more informative than the ones generated from a costly iterative process based on trajectory optimisation.

Considering this one-shot technique and its scalability in practice, for robot navigation, the same way a user teleoperates a real robot via the keyboard or joystick, they can also do it in the learned simulator. For instance, in the (O)CR environment, we simply used the arrow keys from the keyboard to play the game in the learned simulator, exactly as in the real environment. For more advanced scenarios, such as robotic manipulation, leap motion technologies can enable a user to guide the robot's arm within the learned simulator. For large-scale vehicles, full teleoperation driving setups widely used in autonomous vehicle testing could be employed. In general, any setup, used in the real world to gather real-world trajectories in Step 1, would then be used safely in the learned dream world to gather diverse dream-world trajectories in Step 2.

\subsection{Learning a Reward Model from User Preferences and Justifications}
\label{sec:justifications_dream}

In \textbf{Step 3} (Line 11), the goal is to learn a reward model $\hat{r}=\hat{r}([h,z],a)$ from human feedback. In theory, we can use any form of feedback, exploiting the benefits of the one-shot technique from Step 2. However, to achieve aligned and safe agent behaviour during deployment, we propose learning from preferences with \textit{safety justifications}. Initially, we create $K$ pairs of dream trajectory segments $(\sigma ^{0},\sigma ^{1})$. A dream segment of length $k$ is now $\sigma = ( [h_t, \hat{z}_t], a_t^u, \dots, [h_{t+k}, \hat{z}_{t+k}], a_{t+k}^u)$, sampled uniformly at random from $\mathcal{\mathcal{D}}$ (Line 13).\footnote{Other sampling techniques were possible \cite{lee2021b}, but, from initial investigation, they offered no significant advantage.} The video clips of a pair of dream segments are shown to the human after being passed through the VAE decoder: $(\hat{s}_t^i, \dots, \hat{s}_{t+k}^i)=f_{\mathit{dec}}((\hat{z}_t^i, \dots, \hat{z}_{t+k}^i)), \ i=0,1$ (Line 14). Now, a preference $P$ could have been elicited from the human as in the typical reward learning from human preferences \cite{christiano2017deep}. 

However, now we also elicit safety justifications in order to infer the preference label $\mu$ (Line 15). For a single safety justification, we use an offline query protocol for trajectory segments, equivalent to the online one described in Section \ref{sec:preliminaries} for state-action pairs. Specifically, to infer the ground-truth $\mu$ we query the user: (i) whether $\sigma ^{0}$ is safe, i.e.\ $\mathit{Safe}^0 \in \{\mathit{Yes}, \mathit{No} \}$; (ii) whether $\sigma ^{1}$ is safe, i.e.\ $\mathit{Safe}^1  \in \{\mathit{Yes}, \mathit{No} \}$; and only if both answers are $\mathit{Yes}$, we query (iii) which one they prefer, i.e.\ $P\in \{1^{\mathit{st}}, \mathit{equ}, 2^{\mathit{nd}}\}$. Table \ref{tab:equivalence} shows the equivalence between the mappings $P \times J \rightarrow \mu$ and $\mathit{Safe}^0 \times \mathit{Safe}^1 \times P \rightarrow \mu$. We have $\mu \in \{1 - w_s, 1 - w_{\mathit{def}}, 0.5, w_{\mathit{def}}, w_s\}$, where $w_s$ is the justification weight of the only safety justification and $w_{\mathit{def}}$ is the justification weight of a default justification related to any other reason. When not both $\mathit{Safe}^0$ and $\mathit{Safe}^1$ are $\mathit{Yes}$, we can automatically infer $\mu$. Although the two protocols seem similar in terms of human burden and the online one could have been used by referring to trajectory segments instead of state-action pairs, we chose the offline one, since it is more intuitive for queries answered offline. Figure \ref{fig:method} shows in Step 3 two examples with the offline protocol and a single safety justification. In the first of these examples, the $1^{\mathit{st}}$ (upper) clip is automatically preferred \textit{because} the car stays safe inside the road, in contrast to the $2^{\mathit{nd}}$ that veers onto the grass; in the second example, the $2^{\mathit{nd}}$ (bottom) clip is preferred \textit{because} the car covers a longer distance.

\begin{table}[t]
\caption{Equivalence between the online (warning ($w$) when at least one action is unsafe and no warning ($n$) when they are both safe) and offline (querying explicitly the safety of each segment) protocol for one safety justification \cite{kazantzidis2026safe}.} 
\centering
{
\setlength{\tabcolsep}{6mm}
\begin{tabular}{ll|lll|l}
\toprule
\multicolumn{2}{c|}{\textbf{Online}} & \multicolumn{3}{c|}{\textbf{Offline}} &  \\ \midrule
          $P$ & $ J$ & $ \mathit{Safe}^0$  & $ \mathit{Safe}^1$ & $P$ & $\mu$  \\ 
          \midrule
        $1^{\mathit{st}}$ & $w$ & $\mathit{Yes}$ & $\mathit{No}$ & -- & $w_s$ \\ 
        $2^{\mathit{nd}}$ & $w$ & $\mathit{No}$ & $\mathit{Yes}$ & -- & $1-w_s$ \\ 
        $\mathit{equ}$ & $w$ & $\mathit{No}$ & $\mathit{No}$ & --  & $0.5$ \\ 
        $1^{\mathit{st}}$ & $n$ & $\mathit{Yes}$ & $\mathit{Yes}$ & $1^{\mathit{st}}$ & $w_{\mathit{def}}$ \\ 
        $2^{\mathit{nd}}$ & $n$ & $\mathit{Yes}$ & $\mathit{Yes}$ & $2^{\mathit{nd}}$ & $1-w_{\mathit{def}}$ \\
        $\mathit{equ}$ & $n$ & $\mathit{Yes}$ & $\mathit{Yes}$ & $\mathit{equ}$ & $0.5$ \\
          \bottomrule
\end{tabular}
}
\label{tab:equivalence}
\end{table}

\subsubsection{A General Framework for Multiple Justifications} For environments with multiple safety factors, we can generalise the framework to $N$ justifications $J_1, \dots, J_N \in \{0,1\}$ (meaning whether a justification is active or not), with corresponding weights $w_1, \dots, w_N \in(0.5, 1]$. We can combine different justifications in one formula, prioritising them by importance (severity). For example, in the Obstacle Car Racing environment we can consider three safety justifications $J_{\mathit{grass}}, J_{\mathit{chuck}}, J_{\mathit{car}}$ related to avoiding contact with grass, chuckholes and other cars. Additionally, we also always include a default justification $J_{\mathit{def}}$ related to any other reason for preferring one segment over another, such as driving faster to cover more tiles or taking a challenging sharp turn. Then, if we believe that avoiding crashing cars is at least as important (severe) as driving over a chuckhole, avoiding hitting chuckholes is at least as important as driving on grass, and avoiding driving on the grass is at least as important as any other default objective, we could set the justification weights in the following way: $w_{\mathit{car}} \geq w_{\mathit{chuck}} \geq w_{\mathit{grass}} \geq w_{\mathit{def}}$. Although different formulas that combine the justification weights of active justifications could be devised, a straightforward and reasonable setup is to consider only the most severe active justification. That is, in an example that the agent-car is shown to end up hitting another car while being on top of grass during some steps, the reason (justification) for rejecting that clip over another one, would be explicitly the car crash. In that case the preference label $\mu$ should be determined explicitly from $w_{\mathit{car}}$. Using the indicator function $\mathbb{I}(\cdot)$ for the above example, this can be translated to the formula which gives the preference label:

{
\fontsize{9.6pt}{12pt}\selectfont
\begin{equation}
\begin{aligned}
\mu = & \, \mathbb{I}(J_{\mathit{car}} = 1) \\
& \hspace{1em} \cdot \left[ w_{\mathit{car}} \cdot \mathbb{I}(P = 1^{\mathit{st}}) + 0.5 \cdot \mathbb{I}(P = \mathit{equ}) + (1 - w_{\mathit{car}}) \cdot \mathbb{I}(P = 2^{\mathit{nd}}) \right] \\
& + \, \mathbb{I}(J_{\mathit{car}} = 0) \cdot \mathbb{I}(J_{\mathit{chuck}} = 1) \\
& \hspace{1em} \cdot \left[ w_{\mathit{chuck}} \cdot \mathbb{I}(P = 1^{\mathit{st}}) + 0.5 \cdot \mathbb{I}(P = \mathit{equ}) + (1 - w_{\mathit{chuck}}) \cdot \mathbb{I}(P = 2^{\mathit{nd}}) \right] \\
& + \, \mathbb{I}(J_{\mathit{car}} = 0) \cdot \mathbb{I}(J_{\mathit{chuck}} = 0) \cdot \mathbb{I}(J_{\mathit{grass}} = 1) \\
& \hspace{1em} \cdot \left[ w_{\mathit{grass}} \cdot \mathbb{I}(P = 1^{\mathit{st}}) + 0.5 \cdot \mathbb{I}(P = \mathit{equ}) + (1 - w_{\mathit{grass}}) \cdot \mathbb{I}(P = 2^{\mathit{nd}}) \right] \\
& + \, \mathbb{I}(J_{\mathit{car}} = 0) \cdot \mathbb{I}(J_{\mathit{chuck}} = 0) \cdot \mathbb{I}(J_{grass} = 0) \cdot \mathbb{I}(J_{\mathit{def}} = 1) \\
& \hspace{1em} \cdot \left[ w_{\mathit{def}} \cdot \mathbb{I}(P = 1^{\mathit{st}}) + 0.5 \cdot \mathbb{I}(P = \mathit{equ}) + (1 - w_{\mathit{def}}) \cdot \mathbb{I}(P = 2^{\mathit{nd}}) \right]
\end{aligned}
\label{eq:gen_just_example}
\end{equation}
}

\noindent With this formula, $\mu$ is determined from the most severe active justification. For example, if there is a car crash in one of the two clips (regardless of the presence of any other active justification), then $\mathbb{I}(J_{\mathit{car}} = 1)=1$, $\mathbb{I}(J_{\mathit{car}} = 0)=0$ and $\mu$ is determined only from the second row of Equation \ref{eq:gen_just_example}. We note that it is always $\mathbb{I}(J_{\mathit{def}} = 1) = 1$, as there should be at least a default reason for preferring a segment over another. This formula can be generalised to an arbitrary number $N$ of justifications. With $J_N$ being the default justification (i.e.\ $\mathbb{I}(J_{\mathit{N}} = 1) = 1$; not necessarily related to safety), $\mathbb{I}(\cdot)$ the indicator function, and with decreasing safety severity ($w_1 \geq w_2 \geq \dots \geq w_N$), the preference label is:

{
\fontsize{9.6pt}{12pt}\selectfont
\begin{equation}
\begin{aligned}
\mu = & \, \mathbb{I}(J_1 = 1) \\
& \hspace{1.1em} \cdot \left[ w_1 \cdot \mathbb{I}(P = 1^{\mathit{st}}) + 0.5 \cdot \mathbb{I}(P = \mathit{equ}) + (1 - w_1) \cdot \mathbb{I}(P = 2^{\mathit{nd}}) \right] \\
& + \, \mathbb{I}(J_1 = 0) \cdot \mathbb{I}(J_2 = 1) \\
& \hspace{1.1em} \cdot \left[ w_2 \cdot \mathbb{I}(P = 1^{\mathit{st}}) + 0.5 \cdot \mathbb{I}(P = \mathit{equ}) + (1 - w_2) \cdot \mathbb{I}(P = 2^{\mathit{nd}}) \right] \\
& + \, \mathbb{I}(J_1 = 0) \cdot \mathbb{I}(J_2 = 0) \cdot \mathbb{I}(J_3 = 1) \\
& \hspace{1.1em} \cdot \left[ w_3 \cdot \mathbb{I}(P = 1^{\mathit{st}}) + 0.5 \cdot \mathbb{I}(P = \mathit{equ}) + (1 - w_3) \cdot \mathbb{I}(P = 2^{\mathit{nd}}) \right] \\
& +  \dots \\
& + \, \mathbb{I}(J_1 = 0) \cdot \mathbb{I}(J_2 = 0) \cdot \dots \cdot \mathbb{I}(J_{N-1} = 0) \cdot \mathbb{I}(J_N = 1) \\
& \hspace{1.1em} \cdot \left[ w_N \cdot \mathbb{I}(P = 1^{\mathit{st}}) + 0.5 \cdot \mathbb{I}(P = \mathit{equ}) + (1 - w_N) \cdot \mathbb{I}(P = 2^{\mathit{nd}}) \right]
\end{aligned}
\end{equation}
}

\noindent or more concisely:

{
\fontsize{9.6pt}{12pt}\selectfont
\begin{equation}
\begin{split}
\mu = \sum_{k=1}^{N}  \Bigg[
    &\left( \prod_{i=1}^{k-1} \mathbb{I}(J_i = 0) \right) 
    \cdot  \mathbb{I}(J_k = 1) \cdot \\
    & \left( 
        w_k  \cdot  \mathbb{I}(P = 1^{\mathit{st}}) + 
        0.5 \cdot \mathbb{I}(P = \mathit{equ}) + 
        (1 - w_k) \cdot \mathbb{I}(P = 2^{\mathit{nd}}) 
    \right)
\Bigg],
\end{split}
\label{eq:gen_just}
\end{equation}
}

\noindent with $\mu \in \{1 - w_1, 1-w_2, \dots, 1-w_N, 0.5, w_N, \dots, w_2, w_1 \}$. Equivalently, and to minimise human burden, in our experiments we elicit only the most severe justification's \textit{name}. For instance, Figure \ref{fig:obs_env_ex} shows the GUI of the Obstacle Car Racing environment with multiple justifications, in an example where the correct answer $J = \mathit{Car}$, $P=2^\mathit{nd}$ (Right clip), makes $\mu=1-w_{\mathit{car}}$. 

\begin{figure}[ht]
    \centering
    \includegraphics[width=0.65\textwidth]{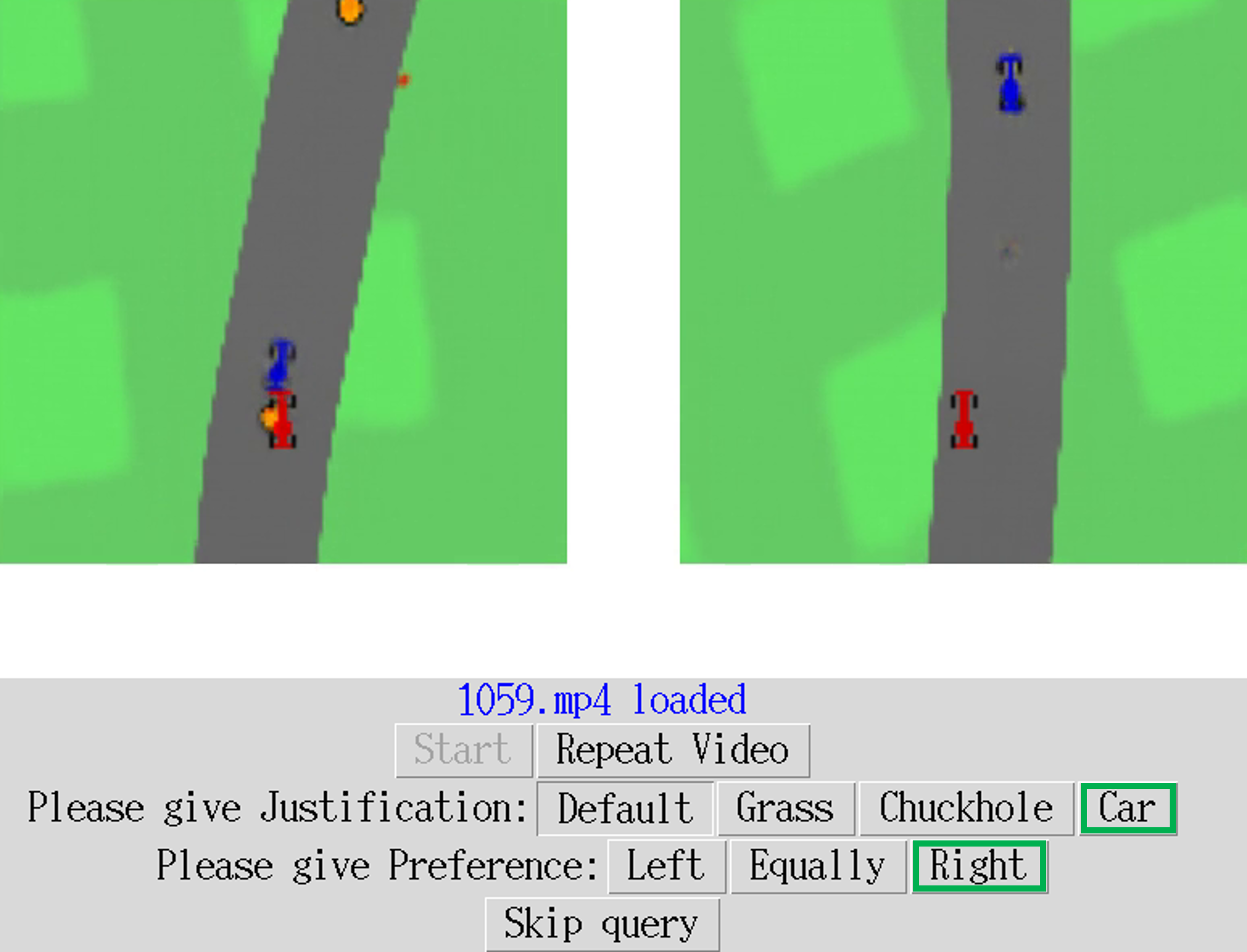}
    \caption{GUI of Obstacle Car Racing with multiple justifications \cite{kazantzidis2026safe}. $J = \mathit{Default}$ is always pre-selected for brevity. The left agent-car contacts both a chuckhole and another car, while the right agent-car stays safe on road. Thus, the correct answer is $J = \mathit{Car}$, $P=2^\mathit{nd}$ (Right).}
    \label{fig:obs_env_ex}
\end{figure}

Finally, for every complete answer from the human, a triple $(\sigma^0,\sigma^1,\mu)$ is stored in $\mathcal{T}$ (Line 16), and $\mathcal{T}$ is used to train the reward model $\hat{r}$, exactly as in Equation \ref{eq:rew_loss} (Line 17). Now a safety objective is encoded in the learned reward model, which prioritises it over other performance objectives (such as driving fast in order to cover more track tiles) or less severe safety objectives. Section \ref{sec:theor_supp} provides support of this claim.

\subsection{Deployment with MPC}
\label{sec:deployment}

Finally, in \textbf{Step 4} (Line 18), an MPC policy $\pi_{\mathit{mpc}}$ can be deployed directly in the real world by planning with the learned dynamics and reward models. Although an RL policy can be learned with $\hat{r}$ in the virtual environment and transferred to the real world, this requires additional training time, whereas MPC tends to work efficiently in practice. In common with prior work \cite{reddy2020learning, rahtz2022safe}, we use a sampling-based approach for MPC planning. Being in the real state $s$ (Line 43), the controller is given the current encoded state $z=f_{\mathit{enc}}(s)$ (Line 24) and the last memory state $h$ (Line 45). The controller generates $S$ sensible random action sequences $\mathcal{A}_{\mathit{mpc}} = \{\mathbf{a}_1, \mathbf{a}_2, \ldots, \mathbf{a}_S\}$, where $\mathbf{a}_i = (a_{i0}, a_{i1}, \ldots, a_{iH})$ represents a sequence of actions over the planning horizon $H$ (Lines 27--28). Executing them with the dynamics model creates the `imagined' trajectories $\{([h_{i0}, z_{i0}],a_{i0}, \ldots, [h_{iH}, z_{iH}],a_{iH}): h_{i0}=h, \ z_{i0}=z, \ i=1,\ldots,S\}$ (Lines 31--32). The predicted return of each trajectory $i$ is: $R_i = \sum_{t=0}^{H}  \hat{r}([h_{it}, z_{it}], a_{it})$ (Lines 33--34) and the action sequence of the highest-return trajectory is chosen: $\mathbf{a}^* = \arg\max_{\mathbf{a}_i \in \mathcal{A}_{mpc}} R_i$ (Lines 35--36). From $\mathbf{a}^* = (a^*_1, a^*_2, \ldots, a^*_H)$, only the first $R$ actions are executed in the real environment $\pi_{\mathit{mpc}}=(a^*_1, a^*_2, \ldots, a^*_R)$ (Lines 37--38, 41, 43), until MPC replans in the same way.

\section{The Impact of Justifications on Safety}
\label{sec:theor_supp}

In this section, we prove that with learning from safety justifications the resulting reward model is biased towards prioritising safe behaviour over other performance objectives. We show this for the single safety justification case, with the extension to multiple justifications being straightforward.

\begin{theorem}
    Reward models that are learned using safety justifications prioritise safe behaviour over other performance objectives.
\end{theorem}

\begin{proof}
Let $\mu \in \{1 - w_s, 1 - w_{\mathit{def}}, 0.5, w_{\mathit{def}}, w_s\} \subset [0,1]$ be the preference label with: 

\begin{itemize}
    \item $w_s$: the weight assigned to preferences with a justification about safety (e.g.\ driving on the grass).
    \item $w_{\mathit{def}}$: the weight assigned to preferences with a default justification (e.g.\ driving faster).
    \item $w_s > w_{\mathit{def}} > 0.5$.
    \item $\mu = 0.5$ when both segments are unsafe, or are safe but the user expresses indifference.
\end{itemize}
Let the dataset $\mathcal{T}$ with triples of the queried segments and the preference labels $(\sigma^0, \sigma^1, \mu)$ be partitioned into:
\begin{itemize}
    \item $\mathcal{T}_{\mathit{safety}}$: samples of queries with at least one unsafe segment (justified by a safety factor).
    \item $\mathcal{T}_{\mathit{def}}$: samples of all other queries with both segments safe (justified by a performance factor unrelated to safety).
\end{itemize}

\noindent Then from Equation \ref{eq:rew_loss}, the cross-entropy loss minimised by the reward model becomes:

\begin{equation}
\begin{split}
\mathcal{L}^{\hat{r}}  =
&  -  \sum_{ (\sigma^0, \sigma^1, \mu_{s}) \in \mathcal{T}_{\mathit{safety}}} 
 \mu_{s} \log 
p[P= 1^{\mathit{st}}]
+ 
(1 - \mu_{s}) 
\log 
p[P= 2^{\mathit{nd}}] \\
& -  \sum_{ (\sigma^0, \sigma^1, \mu_{\mathit{def}}) \in \mathcal{T}_{\mathit{def}}} 
 \mu_{\mathit{def}} \log 
p[P= 1^{\mathit{st}}]
+ 
(1 - \mu_{\mathit{def}}) 
\log 
p[P= 2^{\mathit{nd}}]
\end{split}
\label{eq:rew_loss_2}
\end{equation}

\noindent where $\mu_{s}\in\{1 - w_s, 0.5, w_s\}$ and $\mu_{\mathit{def}}\in\{1 - w_{\mathit{def}}, 0.5, w_{\mathit{def}}\}$. Since $w_s > w_{\mathit{def}} > 0.5$, safety-justified preferences assign labels further from $0.5$ than non-safety ones. Around uncertain predictions, this introduces a larger gradient magnitude for the binary cross-entropy loss, 
$\left|p[P=1^{\mathit{st}}]-\mu\right|$, for samples with a safety justification. Thus, the loss puts more pressure on the reward model to fit such preferences, effectively behaving like an importance-weighted training objective. This means that the learned reward model will prioritise safety over other performance aspects, which leads to the safer behaviour observed in the experiments. \qed
\end{proof}

\section{Evaluation}
\label{sec:evaluation}

\subsection{Experimental Setup}
\label{sec:experimental_setup}

We evaluate DROPJ using the OpenAI Gym Car Racing environment \cite{brockman2016openai} and a new version, Obstacle Car Racing, that extends this with chuckholes and cars, details of which are provided in the Appendix. Except for a baseline, which is an agent trained with traditional RL, we ran real-user experiments to obtain more reliable results. For example, in Step 1 (Figure \ref{fig:method}) if the dataset of real-world trajectories $\mathcal{R}$ consisted only of random trajectories, that would not correspond to realistic past trajectories that we would normally have on a real-world task. Most importantly, in Step 3, simulating the human feedback is not possible, because the oracle-handcrafted reward function of the real environment (which is often used in the literature to simulate near-accurate simulated human feedback) does not correspond to the dream space, from which the queried segments are derived. The world model in Obstacle Car Racing trained for 42 hours in a High-Performance Computing cluster, with more details in the Appendix. Since our method could be evaluated only with real users for trusted results, the above environments were convenient to interact from a typical computer via the keyboard's arrow keys (`play the game'), and the GUIs for eliciting human feedback (Figures \ref{fig:sparse}--\ref{fig:gui_multi}). With this choice of environments we could aptly show the trade-offs among different performance objectives (driving fast) and safety requirements (avoiding grass, chuckholes and cars).

To assess the effectiveness of a method, we use the following metrics: (i) \textit{performance}, evaluated by the return during deployment; (ii) \textit{safety}, using crash rates for the grass (proportion of steps on the grass or kerb over the episode length $L{=}1000$), chuckholes (chuckholes stepped / chuckholes passed), and cars (cars crashed / cars passed); (iii) \textit{human burden}, measured by the number of queries the user answers, and more precisely, the overall time they spend on training; and (iv) \textit{computational cost}, i.e.\ the time the computer takes to run time-intensive computations, while the human has to `wait in the chair'.

All models are evaluated over 10 trials. We run the following algorithms:

\begin{enumerate}

\item \textbf{DRQV2} \cite{yarats2021image, yarats2021mastering}\textbf{;} a model-free RL method, efficient for image-based input, which incorporates augmentation strategies such as random shifts in observation space to regularise the learning process. We allow it to use the oracle reward function and use it as an upper bound.

\item \textbf{ReQueST} \cite{reddy2020learning}\textbf{;} after building the world model, it iterates over providing human feedback, training the reward model, and generating new trajectories via optimisation. The feedback type is \textit{sparse} reward labels (i.e.\ `good' for visiting a new tile, `unsafe' for stepping on grass or kerb and `neutral' for staying on already-visited tiles) to train a classifier-based reward model.

\item \textbf{DROS} (\textbf{D}ream-world \textbf{R}eward learning from \textbf{O}ne-shot \textbf{S}parse labels)\textbf{;} the user, in contrast to ReQueST, generates the dream trajectories (dataset $\mathcal{D}$ in Step 2) with the \textbf{O}ne-shot technique, and, in common with ReQueST, provides \textbf{S}parse labels.

\item \textbf{DROP} (\textbf{D}ream-world \textbf{R}eward learning from \textbf{O}ne-shot \textbf{P}references)\textbf{;} the user uses the \textbf{O}ne-shot technique (as in DROS), but provides \textbf{P}references (instead of sparse labels). When the choice of \textbf{e}quality (indifference) is included in the preference answer, we call the method \textbf{DROPe}.

\item \textbf{DROPJ} (\textbf{D}ream-world \textbf{R}eward learning from \textbf{O}ne-shot \textbf{P}references and \textbf{J}ustifications)\textbf{;} it is as \textbf{DROP}, with the pivotal difference of accompanying preferences with \textbf{J}ustifications.

\end{enumerate}

\noindent The precise protocols with examples based on which the user provided feedback by sparse labels or preferences (with justifications) are explained in the Appendix. Moreover, to create a sufficiently-good and realistic dataset $\mathcal{R}$ of real-world trajectories (assumed to already be in the problem setting), we played the game many times in the real environment ($M=600$ for CR and $M=800$ for OCR --- our strategy for stopping there is presented in Section \ref{sec:world_model_abl}), mostly safely exploring and sometimes taking unsafe actions to represent past failures or controlled unsafe examples. Finally, we used similar architectures and hyperparameters with \cite{ha2018recurrent} for the world model, with \cite{reddy2020learning} for the reward model from sparse labels and MPC, and with \cite{lee2021pebble} for the reward model from preferences. More details are provided in the Appendix. Our code and the full clips for the Figures \ref{fig:obs_env_ex}, \ref{fig:gui_drop_drope_dropj}, \ref{fig:gui_multi} and \ref{fig:mpc_dreams} are available at \url{https://github.com/ilkaza/DROPJ}.

Our evaluation is guided by four questions:
\begin{description}
    \item[\textbf{(Q1)}] \textit{\!\!\!What is the effect of learning with one-shot human-generated trajectories, instead of generating them iteratively via optimisation?} (Section \ref{sec:impact_one_shot})
    \item[\textbf{(Q2)}] \textit{\!\!\!How advantageous is using preferences over sparse labels?} (Section \ref{sec:impact_preferences})
    \item[\textbf{(Q3)}] \textit{\!\!\!What is the effect of safety justifications added to preferences?} \!(Section \ref{sec:impact_justifications})
    \item[\textbf{(Q4)}] \textit{\!\!\!What happens when multiple justifications are involved?} (Section \ref{sec:impact_multi_justifications})
\end{description}

\textbf{(Q1)}-\textbf{(Q3)} can be answered in the CR environment, while \textbf{(Q4)} in the more challenging OCR. We also include an ablation on the world model quality, along with distribution shift considerations (Section \ref{sec:world_model_abl}) and an examination of the robustness of justifications in diverse feedback via synthetic erroneous input, and differing opinions from more labellers (Section \ref{sec:robustness_justif}). We start by seeing the limitation of traditional RL in safety-critical environments (Section \ref{sec:safety_training}).

\subsection{Safety during Training}
\label{sec:safety_training}

We initially clarify that methods 2--5 are completely safe during training, i.e.\ have zero unsafe states. The reason is because all training takes place inside the dream world (learned simulator), with no interactions of the agent with the real world. Instead, DRQV2, which trains with RL in the real environment, suffers a large amount of roughly $50K$ steps in grass in CR, before reaching the upper-bound mean return of around $2800$ during deployment. That shows that traditional RL applied directly in safety-critical real-world domains should expect a large number of safety violations during training. DRQV2's mean return is visible in most of the following figures, such as Figure \ref{fig:impact_oneshot} (left) with the magenta colour, which is the highest possible for the environment. This is expected, since DRQV2 is a state-of-the-art model-free RL algorithm for image-based input and, unlike the other methods, we allow it to make use of the oracle handcrafted reward structure of the environment. Moreover, its grass rate (visible e.g.\ in Figure \ref{fig:impact_oneshot} (right)) is also not zero, since its RL policy was trained to maximise the return (not minimising the grass violations).

\subsection{Impact of the One-shot Technique}
\label{sec:impact_one_shot}

Regarding \textbf{(Q1)}, i.e.\ the impact of learning with one-shot human-generated trajectory segments, rather than generating them iteratively via gradient descent optimisation (as in ReQueST), we compare ReQueST with DROS in order to isolate the effect of the one-shot mechanism from that of using preferences instead of sparse labels (which will be examined in \textbf{(Q2)}). In Figure \ref{fig:impact_oneshot}, we observe that the mean return is higher for DROS across all numbers of sparse label queries (left), while the grass rate is mixed (right). We are not particularly interested in the ups and downs in the grass rate of ReQueST and DROS, since they do not have any direct mechanism to tackle safety during deployment. However, we are interested in the fact that DROS's performance (return) is consistently better than ReQueST's. As we hypothesised, this is because as the human uses their understanding of the task and explores freely in the dream world, more relevant and of a wider range query segments, respectively, can be created. 

\begin{figure}
     \centering
     \begin{subfigure}[b]{0.49\textwidth}
         \centering
         \includegraphics[width=\textwidth]{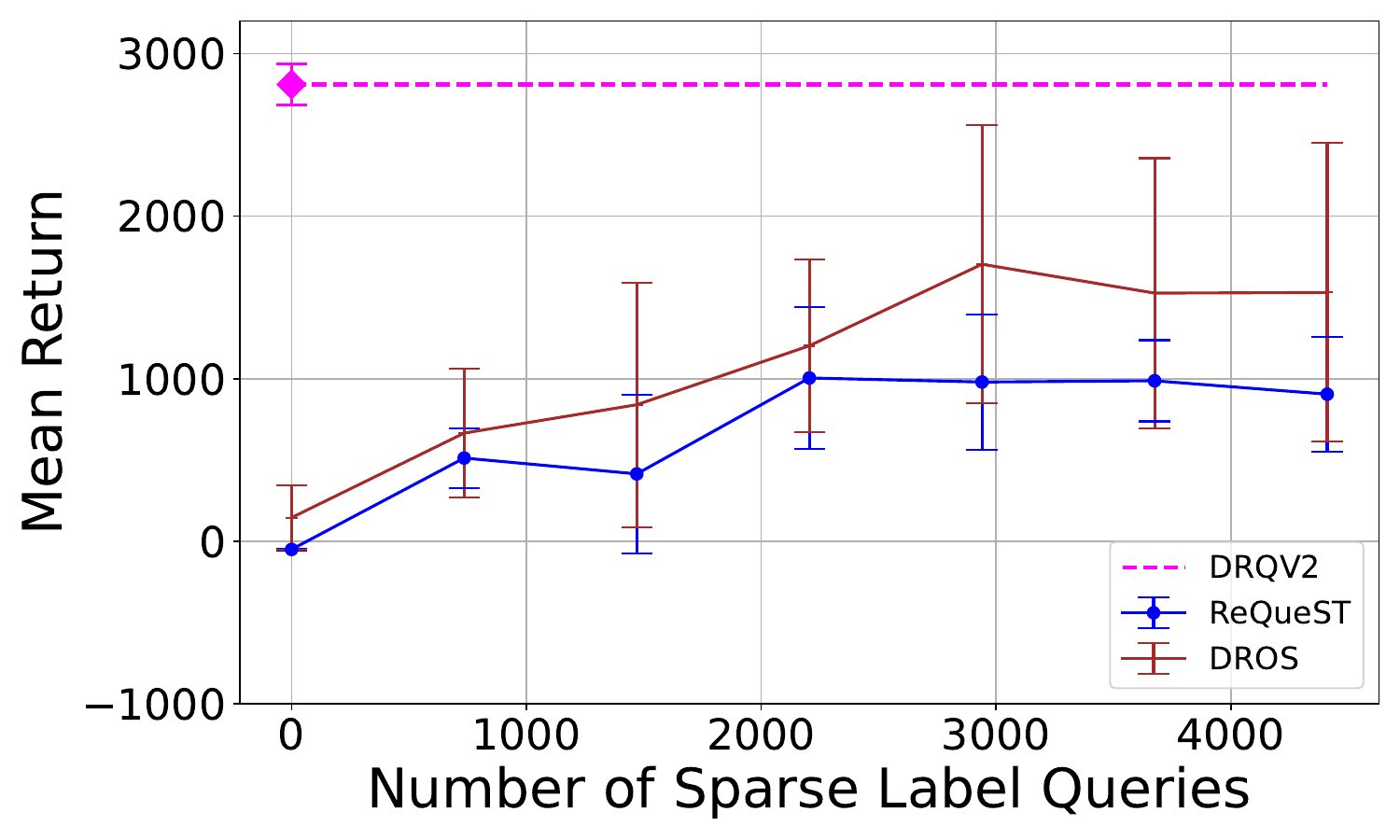}
     \end{subfigure}
    \hfill     
     \begin{subfigure}[b]{0.49\textwidth}
         \centering
         \includegraphics[width=\textwidth]{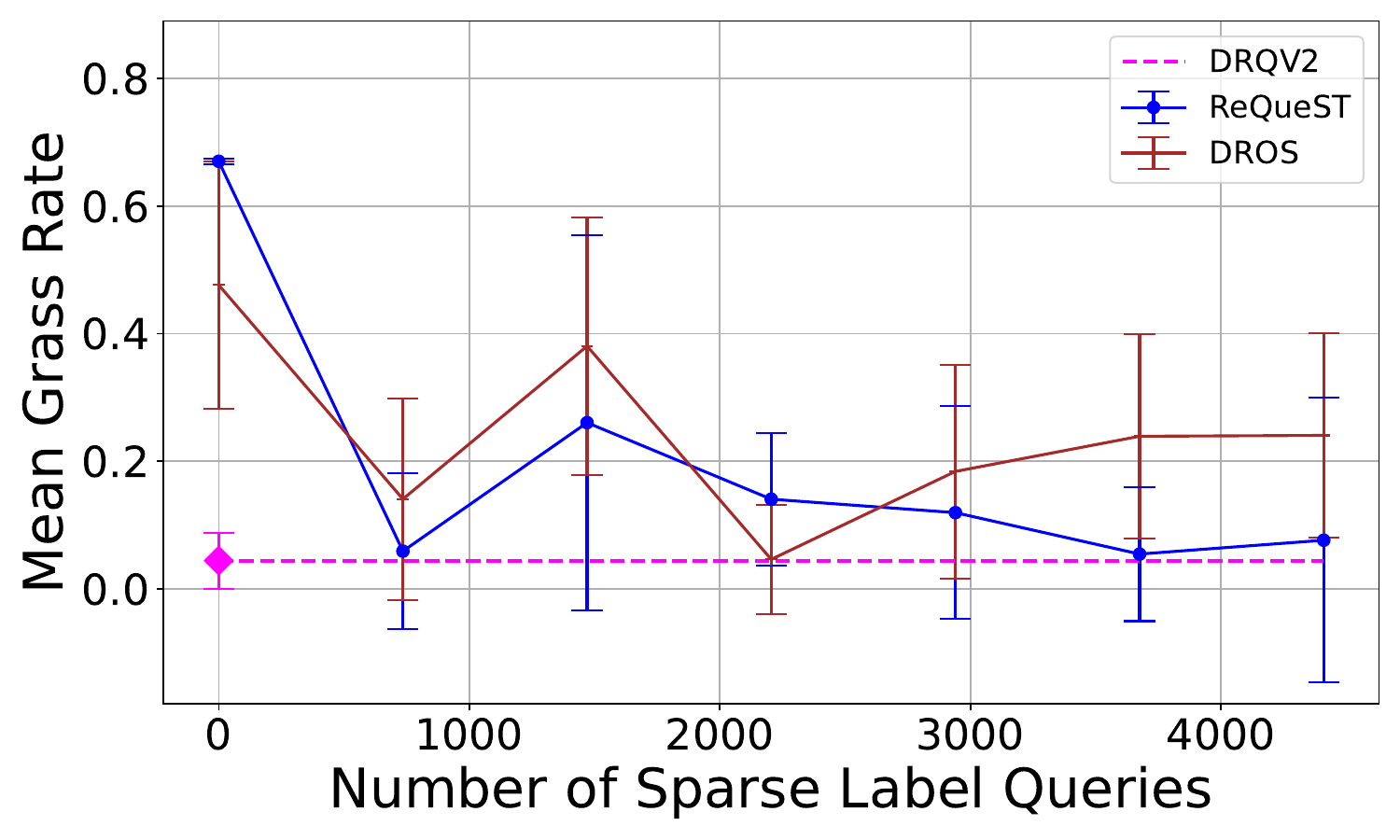}
     \end{subfigure}
        \caption{\textbf{Q1:}  Impact of the one-shot technique \cite{kazantzidis2026safe}.}
        \label{fig:impact_oneshot}
\end{figure}

We confirmed this not only by visual inspection of trajectory segments, but also in more systematic ways. We compared the diversity of DROS and ReQueST segments, using the Determinantal Point Process (Equation 7 from \cite{dai2021diversity}). This method slides a fixed-size window over the sequence of states in a trajectory segment, computes the determinant of the Gram matrix (symmetric matrix of pairwise inner products) formed by the $l_2$-normalised state vectors in each window, and sums these determinants to get the overall diversity score of a trajectory segment. For each set of segments for the two algorithms, the diversity averaged across the segments was:
\[D_{\mathit{DROS}} = 25.64, \quad D_{\mathit{ReQueST}} = 14.38,\] presenting a notable difference. The same was shown by the t-SNE \cite{maaten2008visualizing} of the state representations (Figure \ref{fig:tsne}), where DROS spreads broadly across the entire space, while ReQueST forms tighter clusters. Although there are a few regions covered by ReQueST and not DROS, these likely correspond to rare task situations that the agent does not need to prioritise. In all cases, it is trivial to combine the two approaches if the complexity of the environment permits the high computational cost of the ReQueST approach --- for example, one could first train a robust reward model with the one-shot technique, and then run one or a few iterations of trajectory optimisation and human feedback, to capture any regions that were potentially not covered.

\begin{figure}
        \centering
        \includegraphics[width=0.6\textwidth]{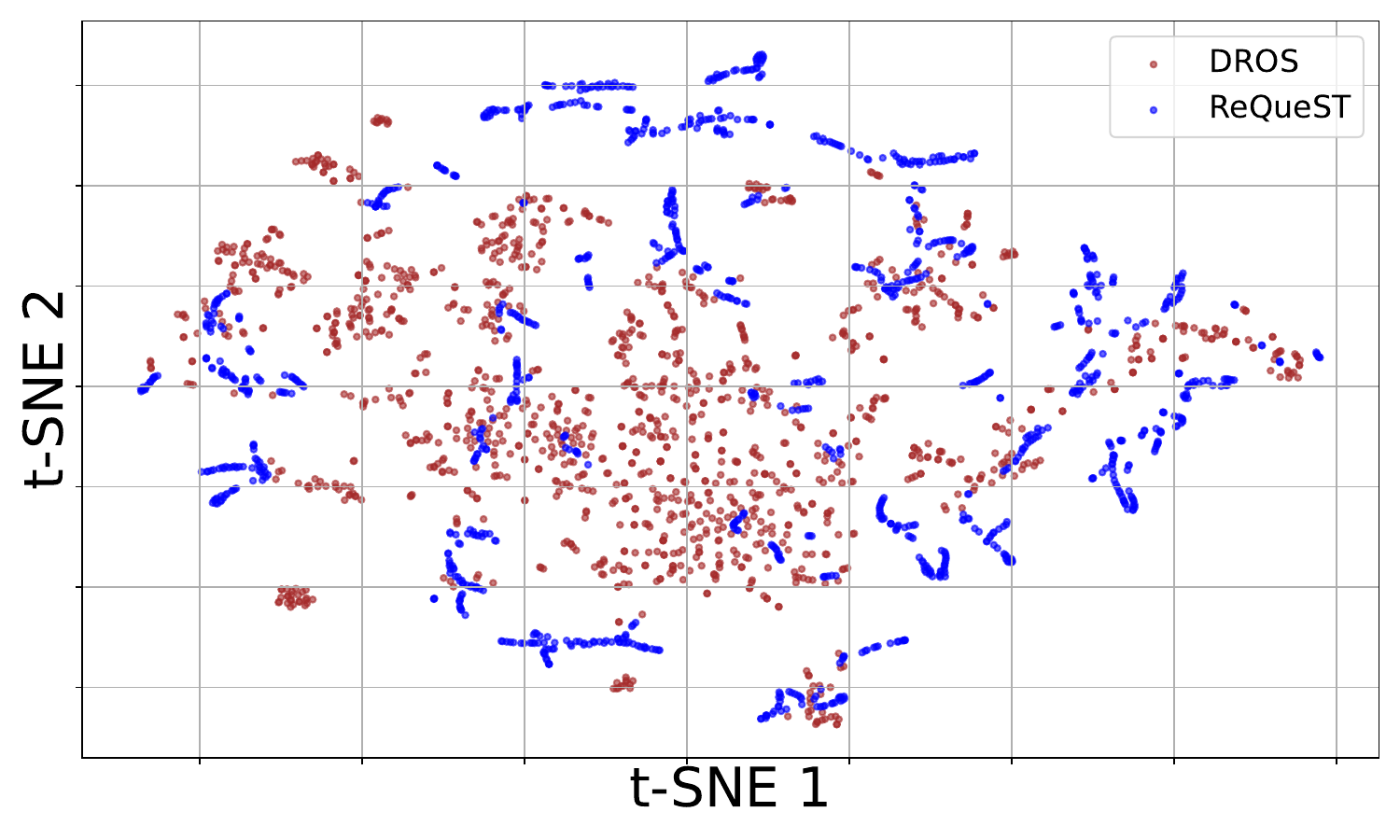}
        \caption[\textbf{Q1:}  Inspecting the diversity of ReQueST versus DROS trajectory segments with the t-SNE of state representations]{\textbf{Q1:}  Inspecting the diversity of ReQueST versus DROS trajectory segments with the t-SNE of state representations \cite{kazantzidis2026safe}.}
        \label{fig:tsne}
\end{figure}

Moreover, for the computational cost in our experiments in CR, we measured that the time ReQueST required to generate \textbf{one trajectory segment of} $\mathbf{50}$ \textbf{steps} was $\mathbf{3}$ \textbf{min and} $\mathbf{13.58}$ \textbf{sec}. In Figure \ref{fig:impact_oneshot}, we see that the ReQueST model with the highest return and lowest crash rate needed $3675$ sparse labels, which is equivalent to $49$ labels in $75$ trajectory segments. Thus, to achieve that, the human had to `wait in the chair' for $75 \times (3~\text{min and } 13.58~\text{sec}) = \mathbf{241~\text{\textbf{min and }} 58.19~\textbf{sec}}$, which is a considerable amount of time\footnote{For the ReQueST model of $2205$ sparse labels (equivalent to $45$ trajectory segments) which has reached the highest return, the computational time is again substantial at $45 \times (3~\text{min and } 13.58~\text{sec}) = \mathbf{145~\text{\textbf{min and }} 11.1~\textbf{sec}}$.}. Even on faster machines, more complex environments may require more gradient steps for trajectory optimisation, compounding the issue, and making it hard to scale efficiently to environments with higher-dimensional action spaces. Instead, with the one-shot technique, which is independent of the action-space dimension, the user needed to draw via the learned simulator (Step 2) only $\mathbf{N=10}$ \textbf{episodes} of the game, lasting around $30$ sec each, i.e.\ spending only $\mathbf{5}$ \textbf{min} overall, to achieve an even better performance.

\subsection{Comparing Preferences with Sparse Labels}
\label{sec:impact_preferences}
To answer \textbf{(Q2)}, i.e.\ how favourable is using human preferences over sparse labels, we should compare DROS (which uses the one-shot technique and sparse labels) with a method that uses the one-shot technique and preferences (Figure \ref{fig:impact_preferences}). As sparse-label queries require less time than preference queries, the x-axis should depict actual feedback time, instead of query counts. Table \ref{tab:times} shows the measured average time per query for each method. Multiplying this rate with the number of queries gives the actual feedback time on the x-axis. We see in Figure \ref{fig:impact_preferences} (left) that after 20 min all preference methods achieve a considerably better performance than DROS for the same amount of human burden. Equivalently, they achieve the same performance with significantly less feedback time. Thus, similarly to other contexts (e.g.\ with perfect simulators), preferences remain efficient when queries are generated from a learned simulator.

\begin{table} [t]
\caption{Average feedback time per query, number of queries and total time until the best model for each method \cite{kazantzidis2026safe}.}
\label{tab:times}
\centering
\begin{tabular}{@{}l@{\hspace{6mm}}l@{\hspace{6mm}}l@{\hspace{6mm}}l@{}}
\toprule
\textbf{Method} & \textbf{Time per Query (s)} & \textbf{Queries} ($K$) & \textbf{Total time (min)} \\
\midrule
ReQueST  & $1.66 \pm 1.38$ & $3675$ & $\simeq 102$ \\
DROS     & $1.66 \pm 1.38$ & $2205$ & $\simeq 61$ \\
DROP     & $7.09 \pm 6.47$ & $500$  & $\simeq 59$ \\
DROPe    & $7.36 \pm 5.45$ & $500$  & $\simeq 62$ \\
DROPJ    & $7.88 \pm 5.83$ & $500$  & $\simeq 66$ \\
\bottomrule
\end{tabular}
\end{table}

\begin{figure}
     \centering
     \begin{subfigure}[b]{0.49\textwidth}
         \centering
         \includegraphics[width=\textwidth]{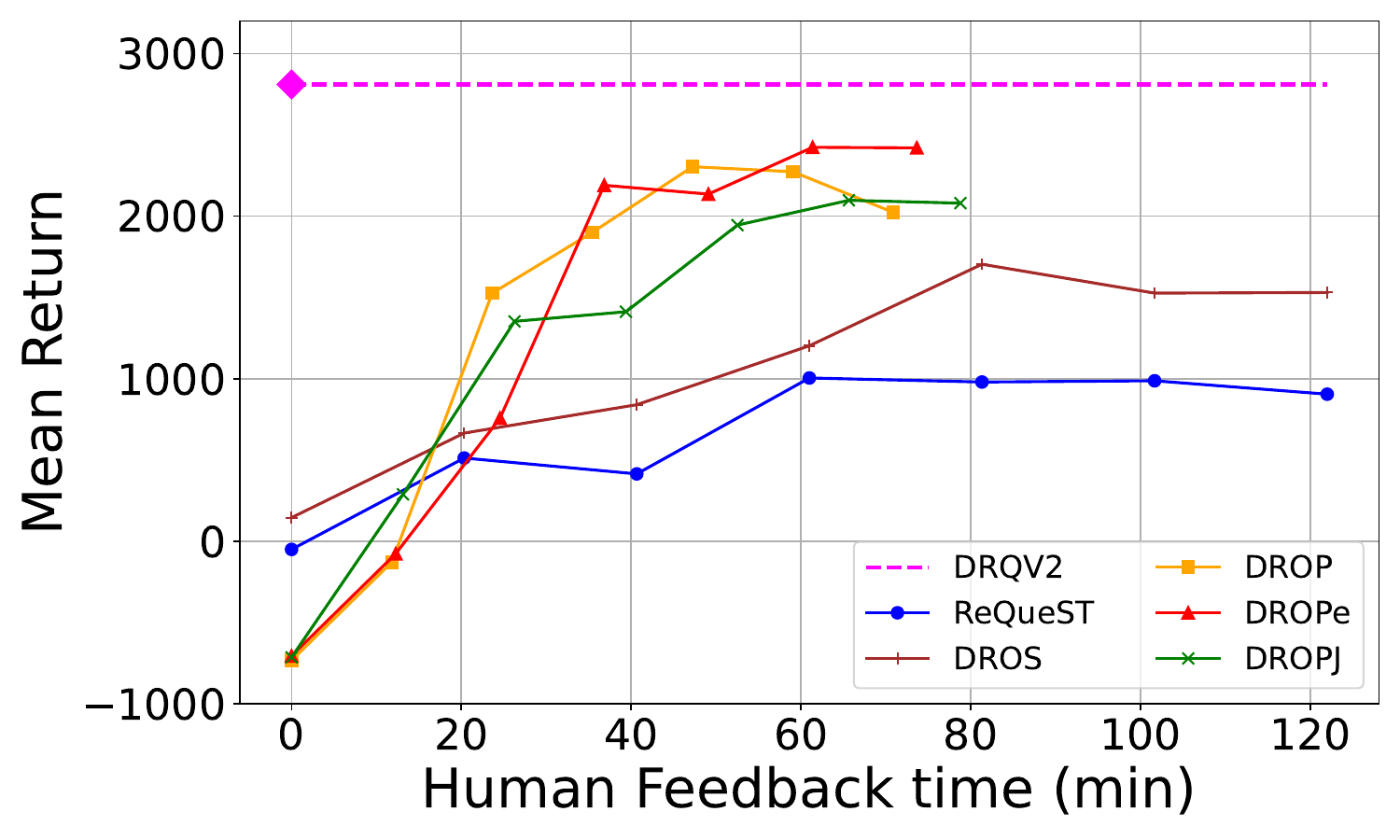}
     \end{subfigure}
    \hfill     
     \begin{subfigure}[b]{0.49\textwidth}
         \centering
         \includegraphics[width=\textwidth]{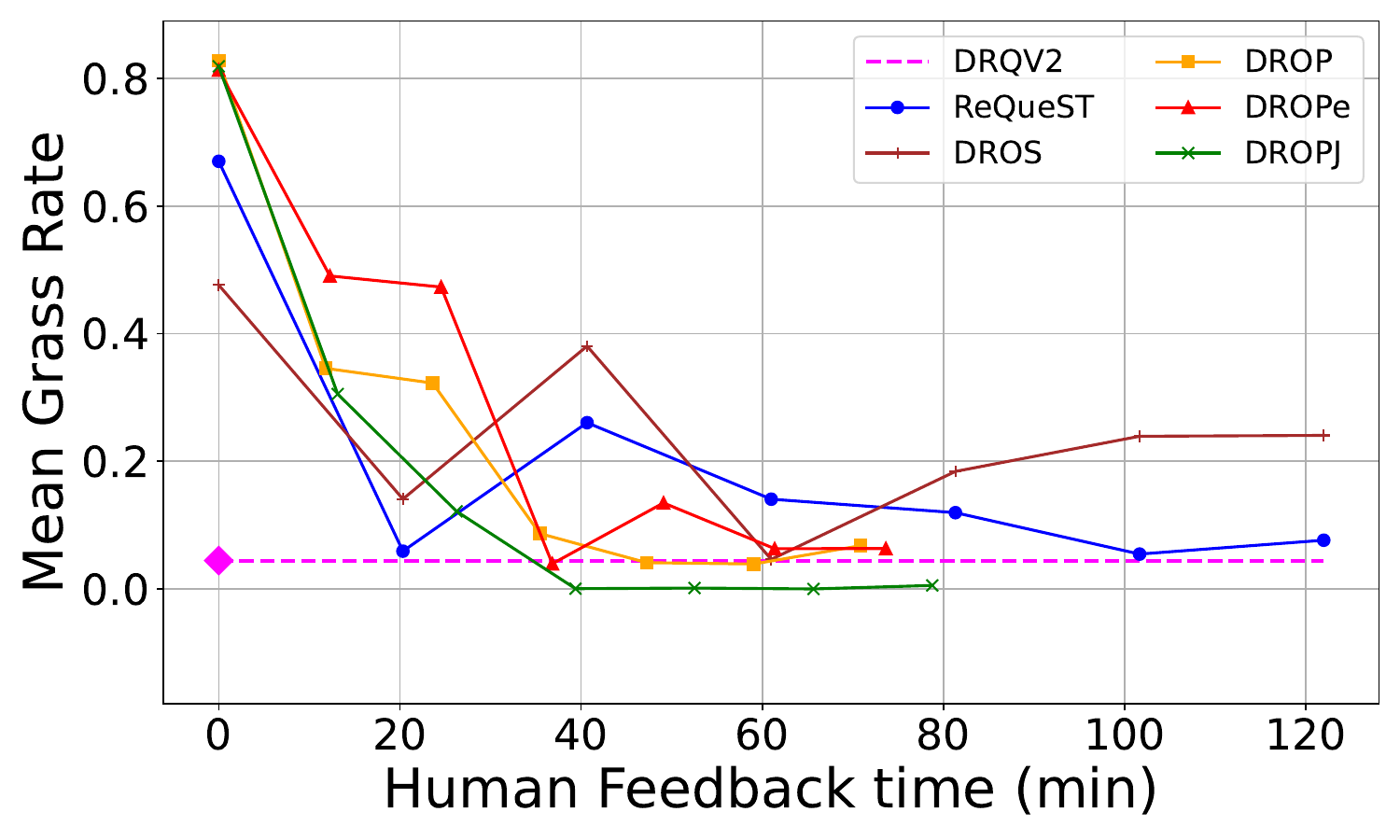}
     \end{subfigure}
        \caption{\textbf{Q2:} Impact of using preferences instead of sparse labels \cite{kazantzidis2026safe}.}
        \label{fig:impact_preferences}
\end{figure}

Regarding the grass rate, we see in Figure \ref{fig:impact_preferences} (right) that the only method which after 40 min stands out is DROPJ. Hence, plain preferences alone do not appear to have an obvious effect on deployment safety compared to other feedback types. This is reasonable, since although they are efficient, they do not have an explicit mechanism for encoding safety requirements. It is the justifications in DROPJ that make the difference in safety, which leads us to \textbf{(Q3)} and \textbf{(Q4)}.

\subsection{Impact of Justifications (the Single Safety Justification Case)}
\label{sec:impact_justifications}

We now investigate \textbf{(Q3)}, i.e.\ the impact of justifications when they accompany preferences. We initially examine this in CR, where the only safety justification refers to grass violations.

We note first, after observing the time per query of DROP, DROPe and DROPJ in Table \ref{tab:times}, that the extra time incurred by justifications was minimally more. Table \ref{tab:times} also shows the required number of queries to yield the best model for each method, taking into account both the return and the grass rate from Figures \ref{fig:impact_oneshot}--\ref{fig:impact_justifications}. If we multiply the time per query with the required number of queries, we can extract in the last column the total time required until the best model. Since for the preference methods (DROP, DROPe, and DROPJ) the number of queries was the same ($K=500$), we can see in total time the very small difference when justifications are added.\footnote{This aligns well with the insights from \cite{kazantzidis2022learning} about justifications, i.e.\ that when a user is used to providing justifications, the latter do not require a considerably higher cognitive effort and overall feedback time.}

\begin{figure}
     \centering
     \begin{subfigure}[b]{0.49\textwidth}
         \centering
         \includegraphics[width=\textwidth]{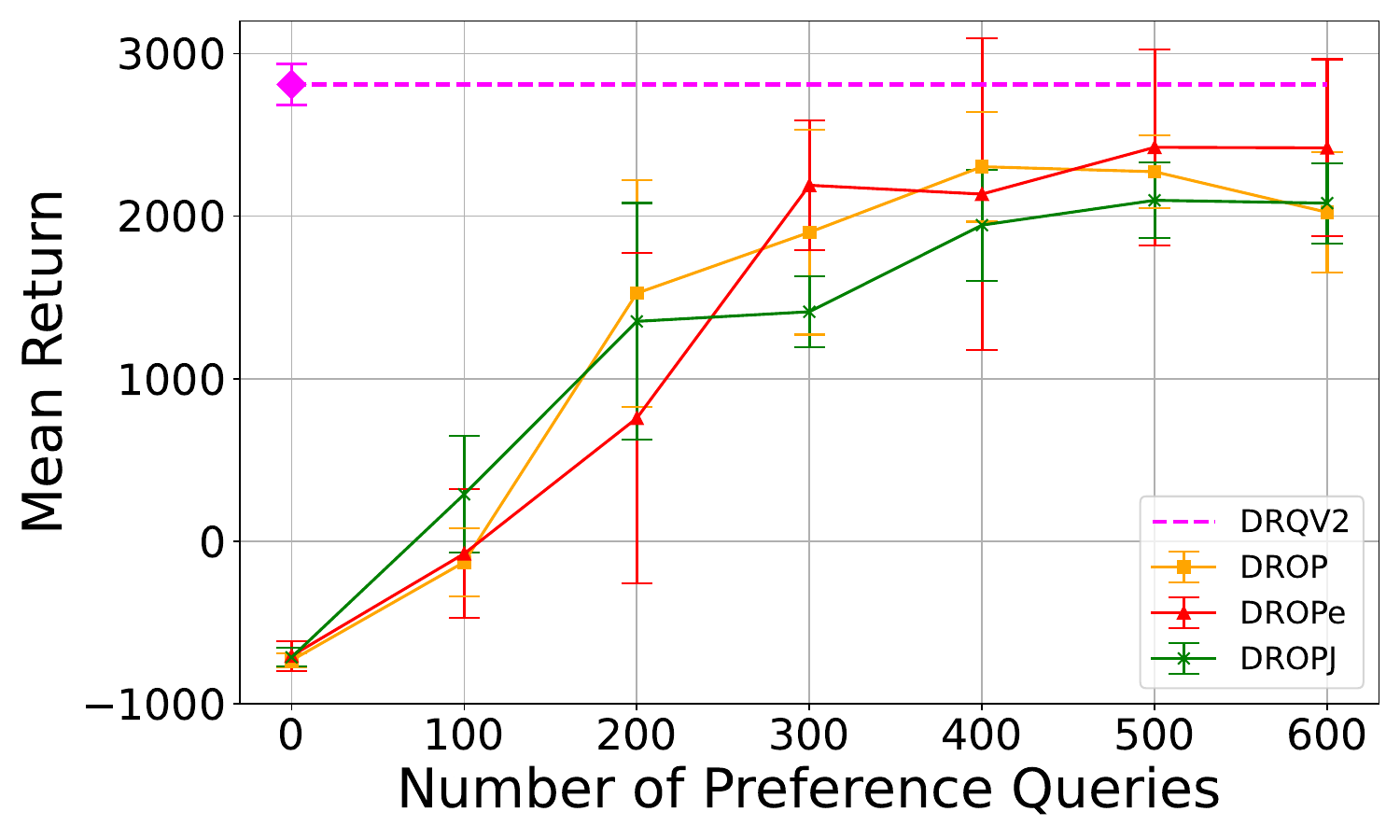}
     \end{subfigure}
    \hfill     
     \begin{subfigure}[b]{0.49\textwidth}
         \centering
         \includegraphics[width=\textwidth]{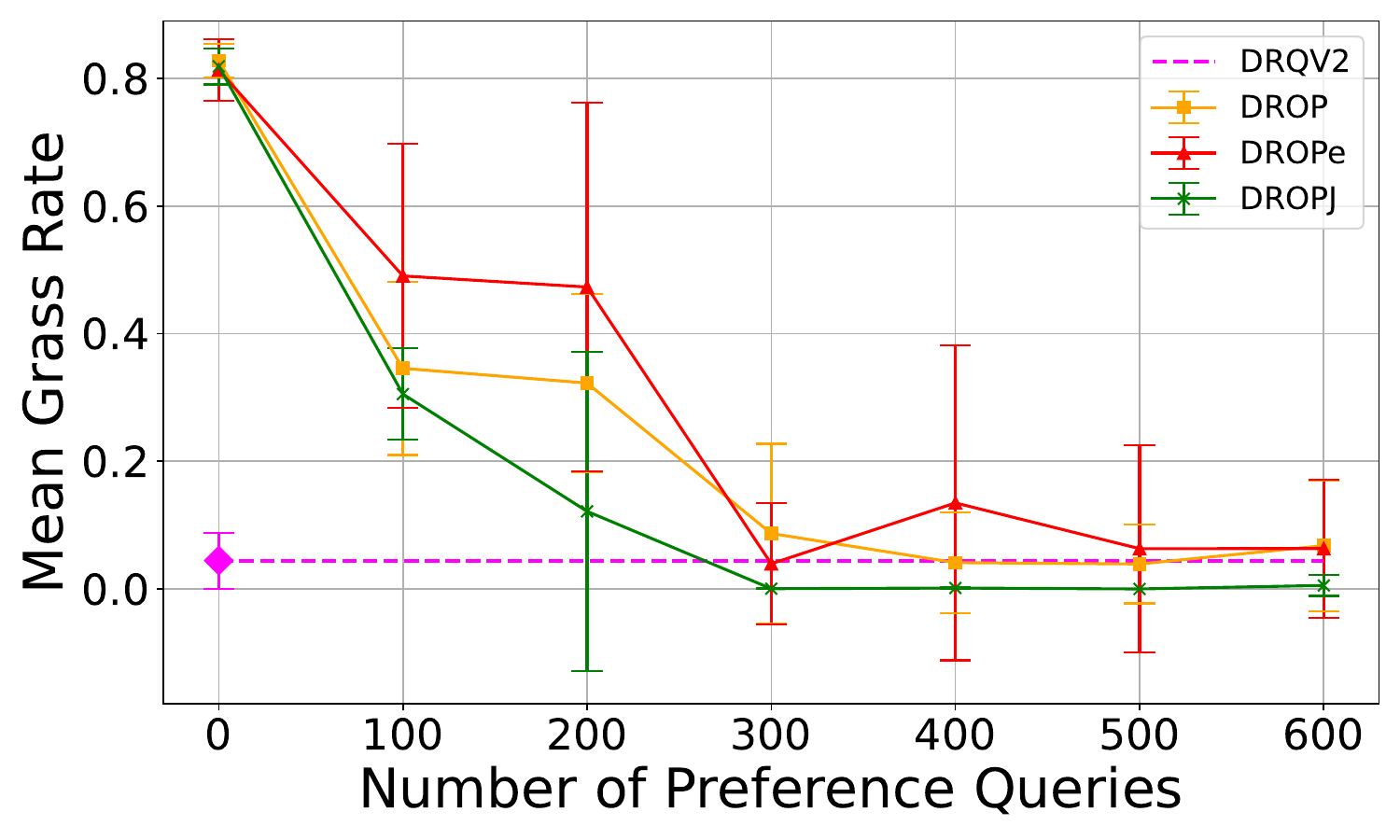}
     \end{subfigure}
        \caption{\textbf{Q3:} Impact of including a safety justification in preferences \cite{kazantzidis2026safe}.}
        \label{fig:impact_justifications}
\end{figure}

Given that, Figure \ref{fig:impact_justifications} shows the effect of the safety justification related to grass, relative to the count of preference queries. We see that although there is a drop-off in the return with the justification, the grass rate for DROPJ after $300$ queries is minimised and significantly less than DROP and DROPe. During the experiments, we observed that the car with DROPJ would carefully slow down on turns and ensure it stayed inside the road. In contrast, in other methods the car would often continue, which could lead to briefly stepping onto the grass or spinning out. Thus, while DROP and DROPe may be more suitable for scenarios like a real car race, DROPJ would be the preferable choice for safety-critical settings, such as city centres. This trade-off is sensible, given that in the reward model of DROPJ avoiding unsafe states was encoded as a priority from the justification queries (\textbf{Step 3}), even at the expense of other metrics.

We also note that although we used the sensible choice of $(w_{\mathit{def}}=0.75, w_{s}=1)$ (maximum weight to safety and reasonable weight to any other default reason/objective) for the models of Figure \ref{fig:impact_justifications}, we experimented with a number of other combinations to verify that safety gains stem strictly from justifications. Keeping as reference the $K=500$ queries, the heatmap in Figure \ref{fig:heatmap} shows the return (left) and the grass rate (right) of several more models trained with justification weights $(w_{\mathit{def}}, w_{s})$. We observe at a glance that good performance and safety requires $w_s \geq w_{\mathit{def}}$. Also, we notice that the combination $(w_{\mathit{def}}=1, w_s=1)$  is effectively equivalent to DROPe. Importantly, for any $w_{\mathit{def}}$, the higher the $w_s$, the better the safety (lower grass rate), providing best results for $w_s=1$. Safety also improves in the last column when $w_s > w_{\mathit{def}}$. The $(w_{\mathit{def}}=0.6, w_{s}=1)$ model gave a few more steps in grass, likely because the general performance started degrading, due to the significant decrease of $w_{\mathit{def}}$ to $0.6$. From further experimentation, we noticed that when a justification weight associated with a specific objective was significantly decreased, this could affect another objective of the problem if they were dependent with each other; this was more evident in the multiple justification case, which we analyse in Section \ref{sec:impact_multi_justifications}.

\begin{figure}
     \centering
     \begin{subfigure}[b]{0.49\textwidth}
         \centering
         \includegraphics[width=\textwidth]{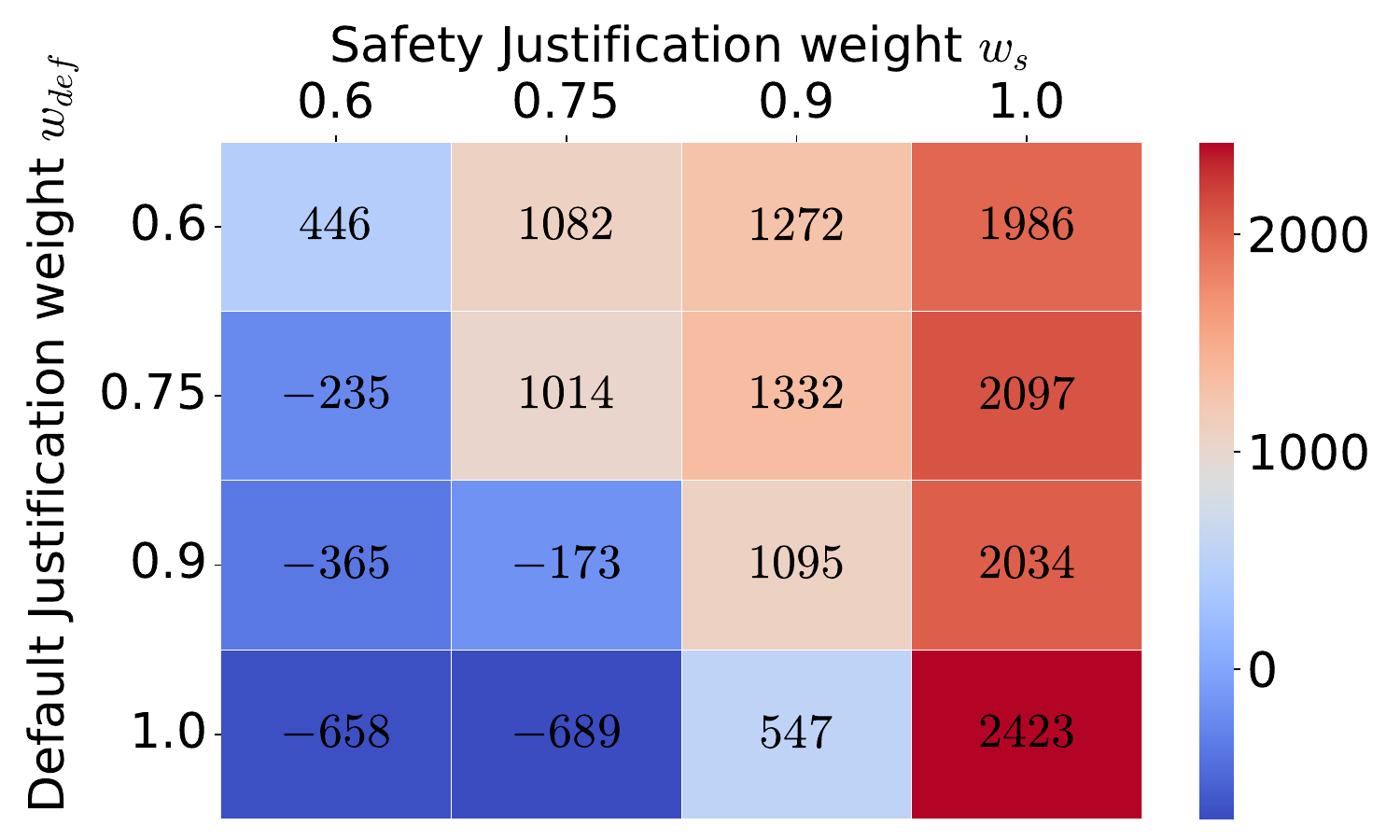}
     \end{subfigure}
    \hfill     
     \begin{subfigure}[b]{0.49\textwidth}
         \centering
         \includegraphics[width=\textwidth]{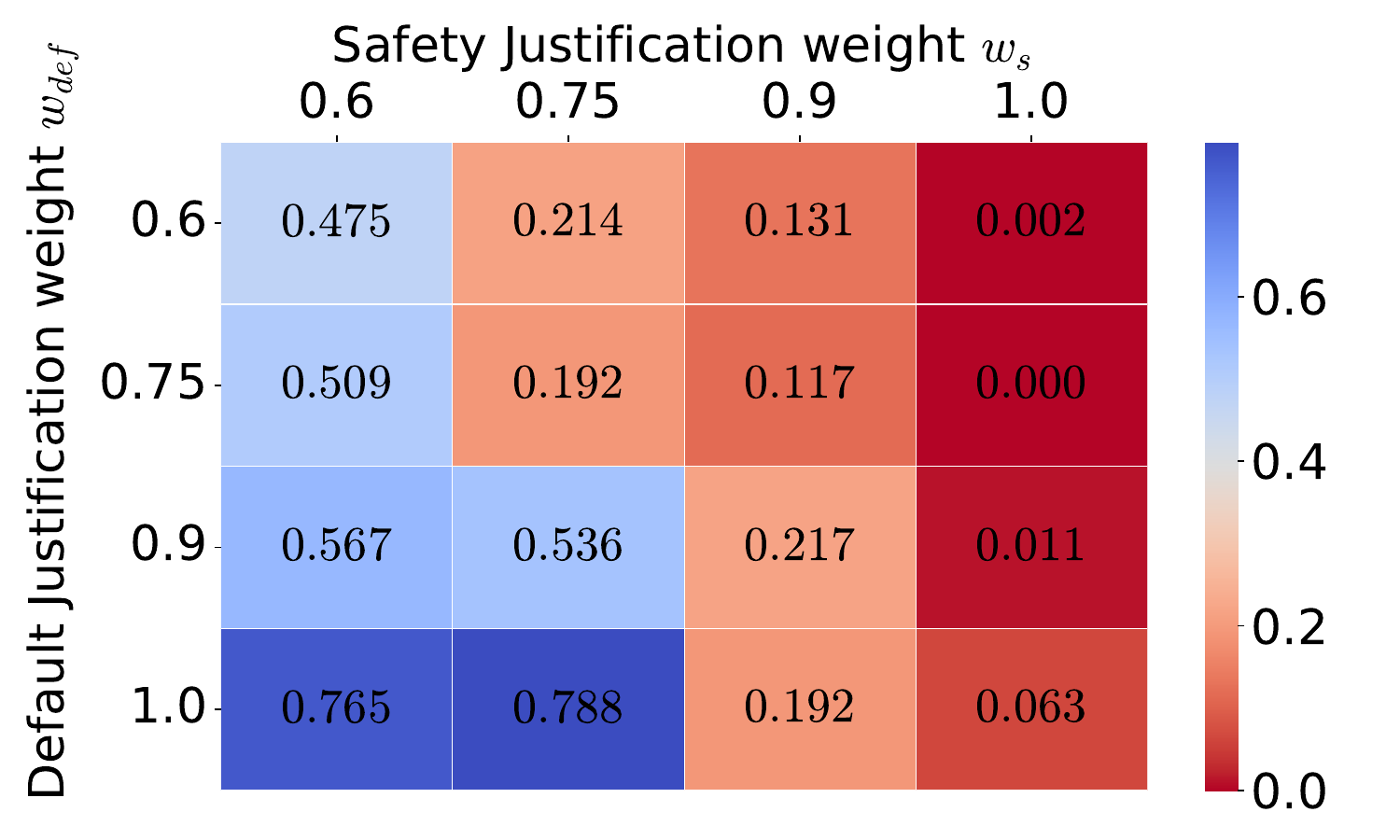}
     \end{subfigure}
        \caption{Mean return (left) and mean grass rate (right) for various $(w_{\mathit{def}},w_s)$ \cite{kazantzidis2026safe}.}
        \label{fig:heatmap}
\end{figure}

\subsection{Significance Tests}
\label{sec:signif tests}

Here, we present the significance tests, related to \textbf{(Q1)}-\textbf{(Q3)} (Sections \ref{sec:impact_one_shot}--\ref{sec:impact_justifications}), for the return and grass rate of the best trained models of each method. Figure \ref{fig:best_models} shows their mean return (left) and mean grass rate (right) distributions, and Table \ref{tab:significance} the $p$-values from the significance tests associated with them. By observing the box-plots and histograms, we saw that the assumptions of homoscedasticity and normality are not clearly satisfied, so we conducted non-parametric tests. First, we use the Mann-Whitney U (Wilcoxon rank-sum) test \cite{mann1947test} between ReQueST and DROS to assess the significance of the one-shot technique. For the return, an one-sided test (expecting lower values for ReQueST) gives $p=0.0605$, indicating weak significance ($p < 0.1$). For the grass rate, no difference is observed. Thus, we conclude that while the one-shot technique improves performance, there is scope for further improvement through more effective feedback types. Comparing for instance DROS with DROPJ, we find a significant difference in both the return and grass rate. Moreover, although DRQV2 has a significant difference in the return compared to DROPe (as expected, given DRQV2 trains with the oracle reward), the difference is not highly significant ($p=0.0376$). Notably, DROPe's return nearly matches the upper bound in some episodes, despite learning in the dream (Figure \ref{fig:best_models}), revealing the power of preferences. Regarding DROP, DROPe and DROPJ, we perform the non-parametric Kruskal–Wallis test \cite{kruskal1952use}, identifying significant differences ($p=0.008$ for the return, and $p=0.0029$ for the grass rate). Post-hoc pairwise Dunn's tests with the Holm-Bonferroni adjustment \cite{dunn1964multiple, holm1979simple} show significant differences in the return between DROPe and DROPJ ($p=0.0058$), and in the grass rate between DROPJ and both DROP ($p=0.0036$) and DROPe ($p=0.0215$), confirming, respectively, the trade-off between performance and safety discussed in \textbf{(Q3)}.

\begin{figure}
     \centering
     \begin{subfigure}[b]{0.49\textwidth}
         \centering
         \includegraphics[width=\textwidth]{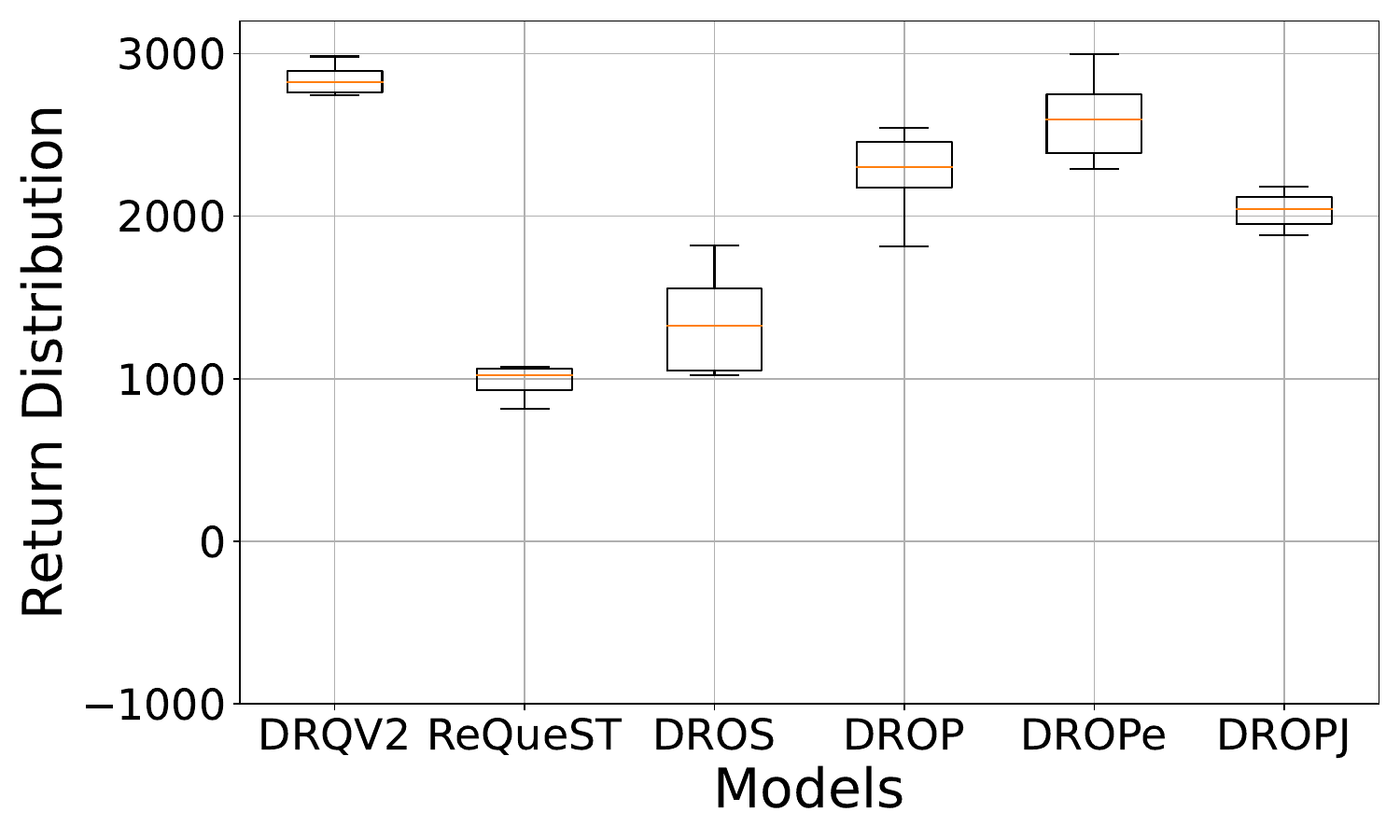}
     \end{subfigure}
    \hfill     
     \begin{subfigure}[b]{0.49\textwidth}
         \centering
         \includegraphics[width=\textwidth]{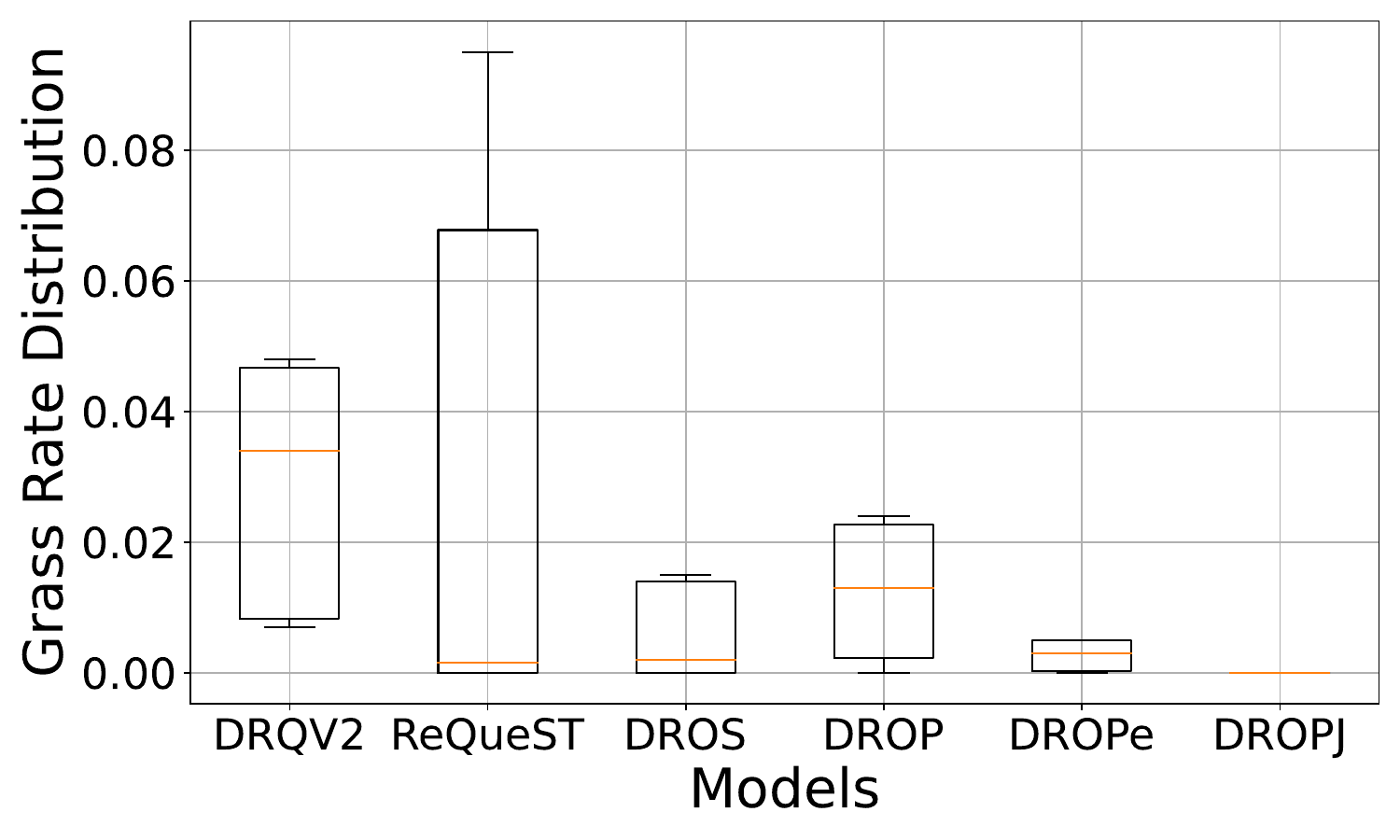}
     \end{subfigure}
        \caption{Return and grass rate distributions of the best models in Car Racing \cite{kazantzidis2026safe}.}
        \label{fig:best_models}
\end{figure}

\begin{table}
\caption{$p$-values for the return and grass rate comparisons from the Mann-Whitney U tests (upper half) and the Dunn's tests (bottom half).}
\label{tab:significance}
\centering
{
\setlength{\tabcolsep}{6mm}
\begin{tabular}{lll}
\toprule
\textbf{Comparison}    & \textbf{Return} & \textbf{Grass rate} \\ \midrule
ReQueST/DROS    & $0.0605$     & $1.0000$      \\
DROS/DROPJ      & $0.0002$     & $0.0149$     \\
DRQV2/DROPe     & $0.0376$     & $0.0170$     \\
\midrule
DROP/DROPe      & $0.1775$     & $0.4905$     \\
DROP/DROPJ      & $0.1624$     & \boldsymbol{$0.0036$}     \\
DROPe/DROPJ     & \boldsymbol{$0.0058$}     & \boldsymbol{$0.0215$}     \\ 
\bottomrule
\end{tabular}
}
\end{table}

\subsection{World Model Ablation and Distribution Shifts}
\label{sec:world_model_abl}

In this section we provide an ablation on the world model quality, along with distribution shift observations during deployment, and relate these to justifications. Figure \ref{fig:loss_curves} shows for CR, the validation loss curves of the two world model components, i.e.\ the encoder (VAE) and the dynamics model (MDN-RNN), trained with a different number $M$ of real-world trajectories $\mathcal{R}$. After $M=600$, there is no significant improvement in either of them. Hence, based on this and the inspection of VAE reconstructions and MDN-RNN predictions, we decided not to extract more trajectories.\footnote{For OCR, we used the same strategy and extracted another 200 trajectories with chuckholes and cars included, in order to handle the environment's higher complexity.}

\begin{figure}
     \centering
     \begin{subfigure}[b]{0.49\textwidth}
         \centering
         \includegraphics[width=\textwidth]{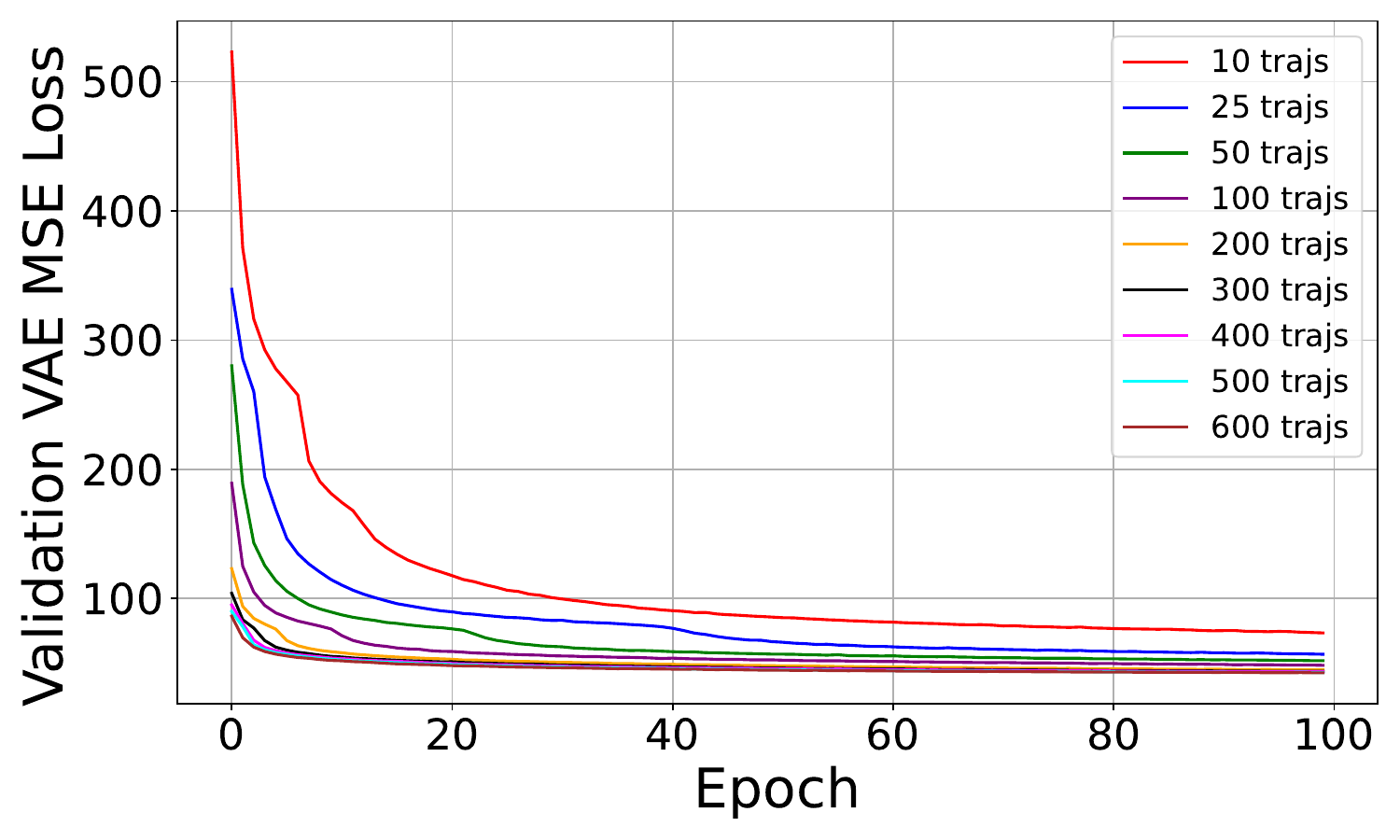}
     \end{subfigure}
    \hfill     
     \begin{subfigure}[b]{0.49\textwidth}
         \centering
         \includegraphics[width=\textwidth]{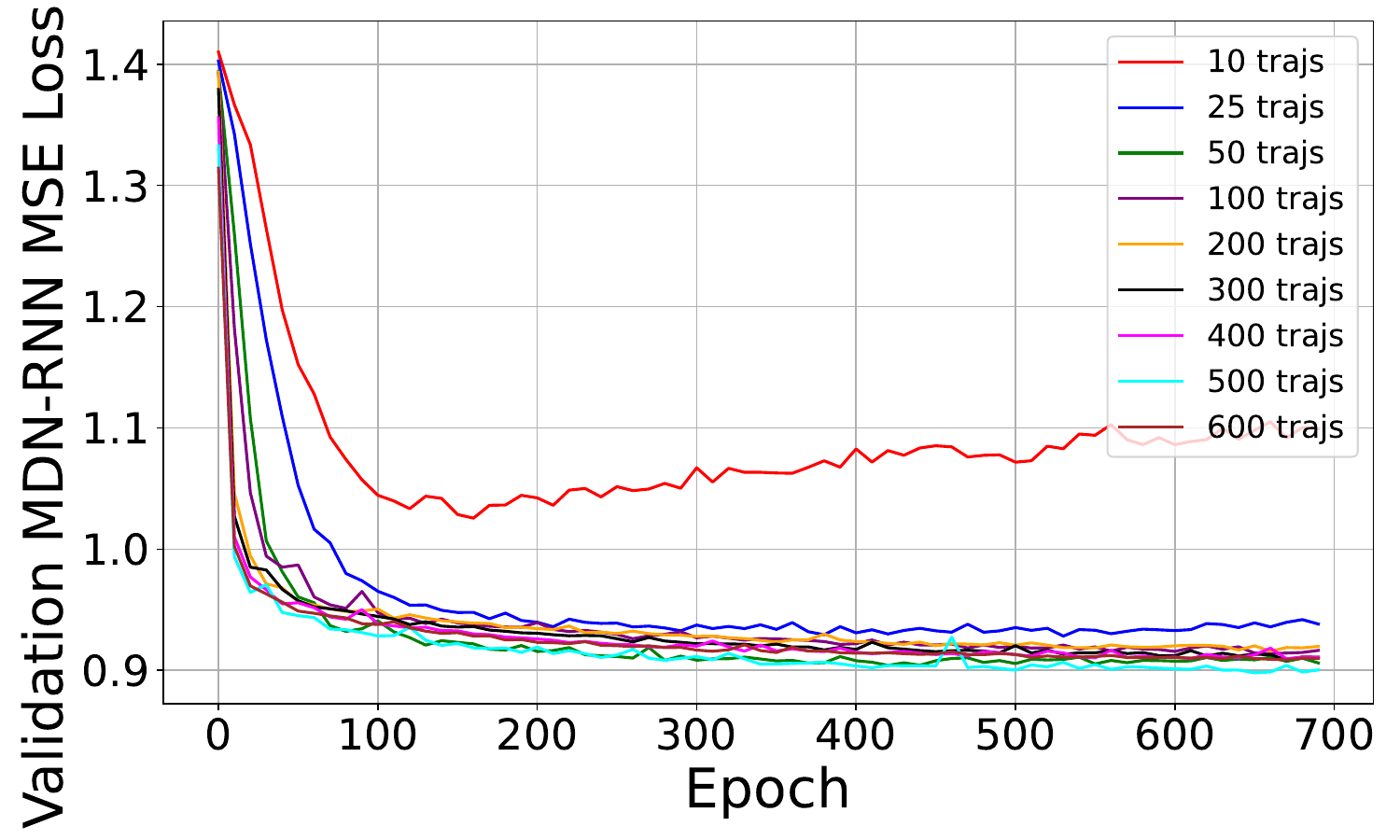}
     \end{subfigure}
    \caption{Encoder (VAE) and dynamics model (MDN-RNN) Validation Mean Squared Error (MSE) curves for different numbers of real-world trajectories \cite{kazantzidis2026safe}.}
    \label{fig:loss_curves}
\end{figure}

However, it would be interesting to see the evaluation of some degraded world models. We chose the world models with a large drop of trajectories at $M=25$ and $M=10$, and repeated, for DROPe and DROPJ, \textbf{Steps 2} (drawing the same number of dream user trajectories) and \textbf{3} (providing the same number of feedback answers). Figure \ref{fig:wm_abl} shows for $M=25$, that, although the return and grass rate degraded, they stayed surprisingly resilient. Even though the world model could barely recreate turns, it learnt well the objective taught from the human feedback of staying on the road and avoiding the grass. Thus, it could perform well even in turns. Additionally, the benefits of justifications (DROPJ) were evident even at $M=25$, where the degradation of return and grass rate was very limited, in contrast to plain preferences (DROPe). Finally, as expected, for only $M=10$, both metrics significantly degraded and the car could only trivially drive, since the world model could not plan successfully.\footnote{We experimented similarly for OCR, but since the objectives were more complicated, we found that more trajectories were normally needed.}

\begin{figure}
     \centering
     \begin{subfigure}[b]{0.49\textwidth}
         \centering
         \includegraphics[width=\textwidth]{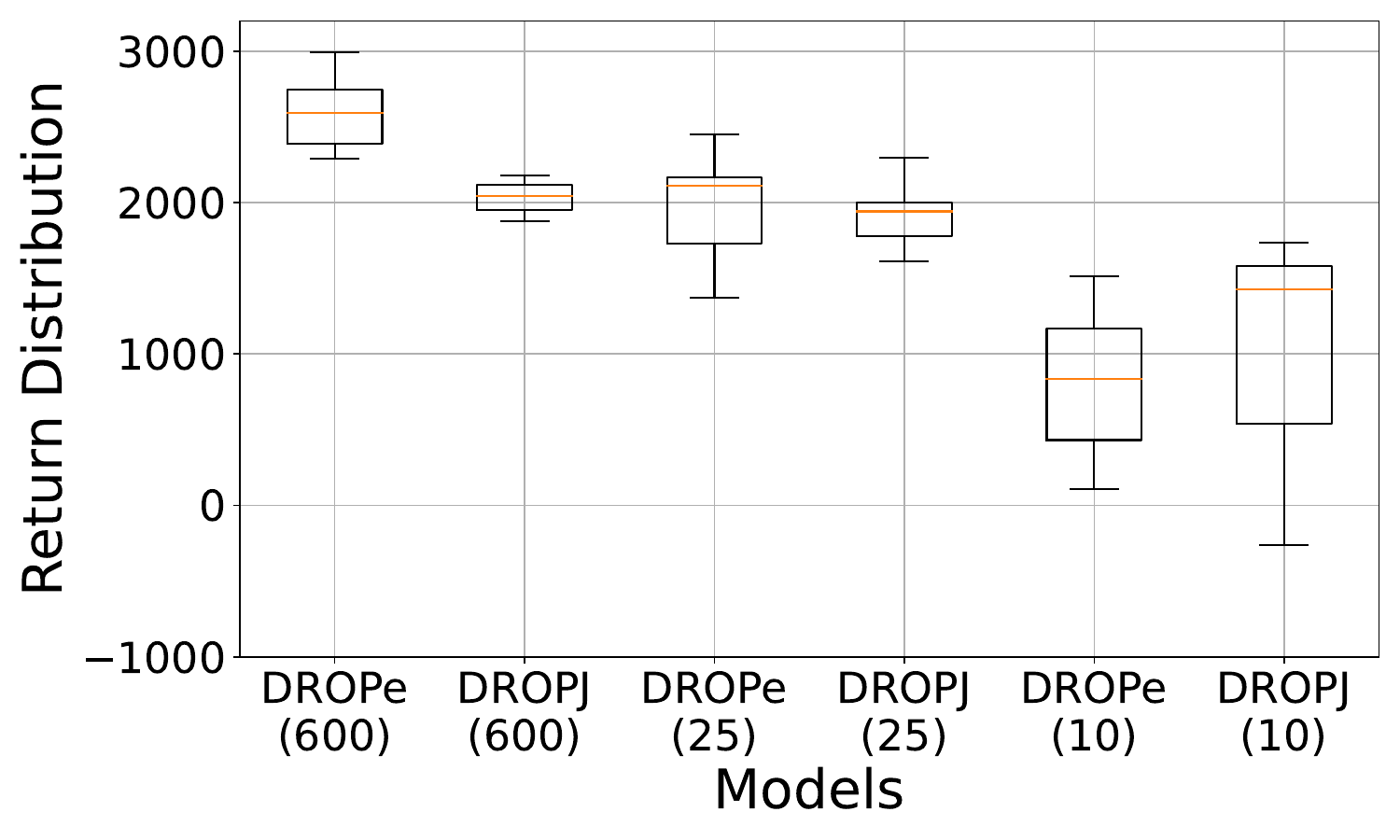}
     \end{subfigure}
    \hfill     
     \begin{subfigure}[b]{0.49\textwidth}
         \centering
         \includegraphics[width=\textwidth]{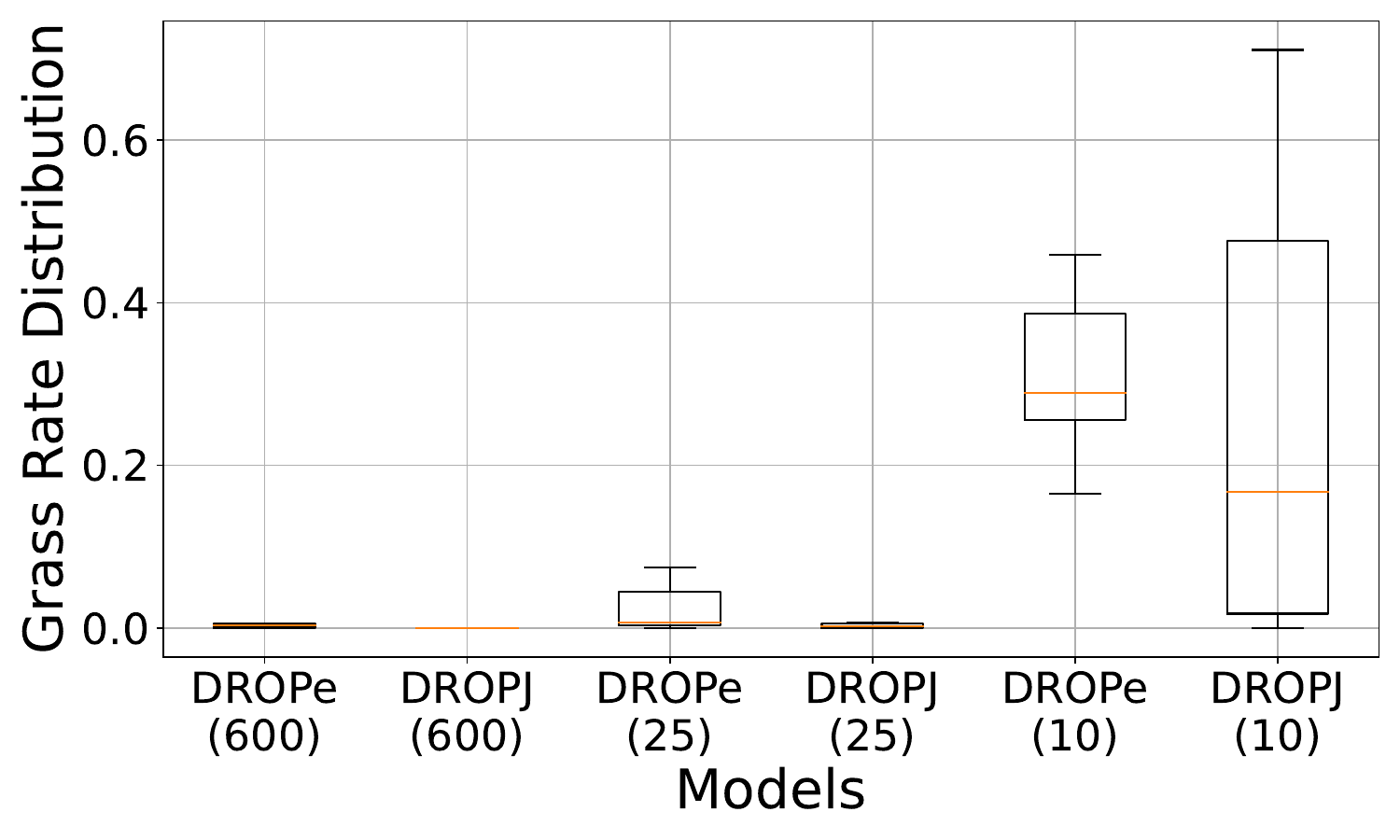}
     \end{subfigure}
    \caption{World model ablation for different numbers of real-world trajectories \cite{kazantzidis2026safe}.}
    \label{fig:wm_abl}
\end{figure}

Reward models from human preferences tend to generalise well by learning human objectives (e.g.\ stay on the road and avoid obstacles). However, if the test distribution is significantly far from the training one (i.e.\ if the deployment environment has significant shifts --- includes elements that were absent from the real-world trajectories training the world model), then the world model cannot always yield accurate MPC plans. In our experiments, the configuration of the track was randomised both in real-world trajectories and in testing episodes, so some findings from other distribution shifts during deployment are the following:

\begin{itemize}
\item  The car, when reversed, could still drive the track clockwise successfully. 
\item  Increasing the road width, there was not a significant degradation in metrics.
\item  Tightening the road width, DROPe could not stay on road without MPC finetuning, but DROPJ (safety-oriented) tried to stay on the road at all costs (even when barely fitting). The same happened most times with very sharp turns, even if the car had to halt.
\item  When the size of the car increased (automatically modifying its internal dynamics), the performance slightly dropped, but again DROPJ would show a consistent advantage over DROPe in trying to stay inside the road.
\item  Interestingly, when only the colour of the grass or road changed, the evaluation remained competitive, but when both changed, there was a significant degradation. This is most likely because the reward model learns to associate grey colour with road where the agent has to stay and green with grass that the agent should avoid. When both cues change simultaneously, the reward model cannot reliably guide any candidate plan.
\end{itemize}

The main conclusion from experimentation is that although performance may drop with significant distribution shifts, safety gains from justifications remain evident. In terms of our framework, as a solution to small distribution shifts, we would typically require only a handful of more real-world examples\footnote{This is also shown to be possible in the large-scale V-JEPA 2 world model \cite{assran2025v}, even for similar tasks.} and human preferences to update the world and reward models (via the same DROPJ pipeline of Figure \ref{fig:method}) and generalise better. Another strategy could be from the outset to augment the initial real-world trajectories with artificially varied conditions we expect during deployment (e.g.\ textures, lightings, colours).

\subsection{Robustness of Justifications in Diverse Feedback}
\label{sec:robustness_justif}

Here, we examine the effect of human error and differing opinions in feedback answers, focusing on justifications. Initially, for the best-trained DROP, DROPe, and DROPJ models (all at $K=500$ queries) in CR, we randomly change the answer of \textit{non-critical queries}, i.e.\ preference pairs where the car remains on the road in both clips. The same queries were modified for all methods, with a substantial 60\% of non-critical queries corrupted. Figure \ref{fig:impact_non} shows that DROPJ is more resistant than DROP and DROPe to non-critical errors. Specifically, while DROP and DROPe degrade significantly, DROPJ exhibits a smaller decrease in return, retains a tighter distribution, and maintains a grass rate near the minimum. This was expected, as errors in non-critical queries for DROPJ only impact the preference label $\mu$ within the range $[1-w_{\mathit{def}}, w_{\mathit{def}}]$ (e.g.\ $[0.25, 0.75]$), while $\mu=1-w_s$ (e.g.\ $0$) and $w_s$ (e.g.\ $1$) is reserved for safety (Table \ref{tab:equivalence}). In contrast, DROP and DROPe extend errors across the entire $\mu$ range of $[0,1]$. However, as shown in Figure \ref{fig:impact_cr}, when \textit{critical queries} (i.e.\ queries where at least in one clip the car goes through an unsafe state) are mislabelled, all methods, including DROPJ, collapse. While DROPJ shows some resistance on the grass rate to an error percentage up to 10\%, it finally breaks after 15\%.

\begin{figure}
    \centering
    \begin{subfigure}[t]{0.49\textwidth}
        \centering
        \includegraphics[width=\textwidth]{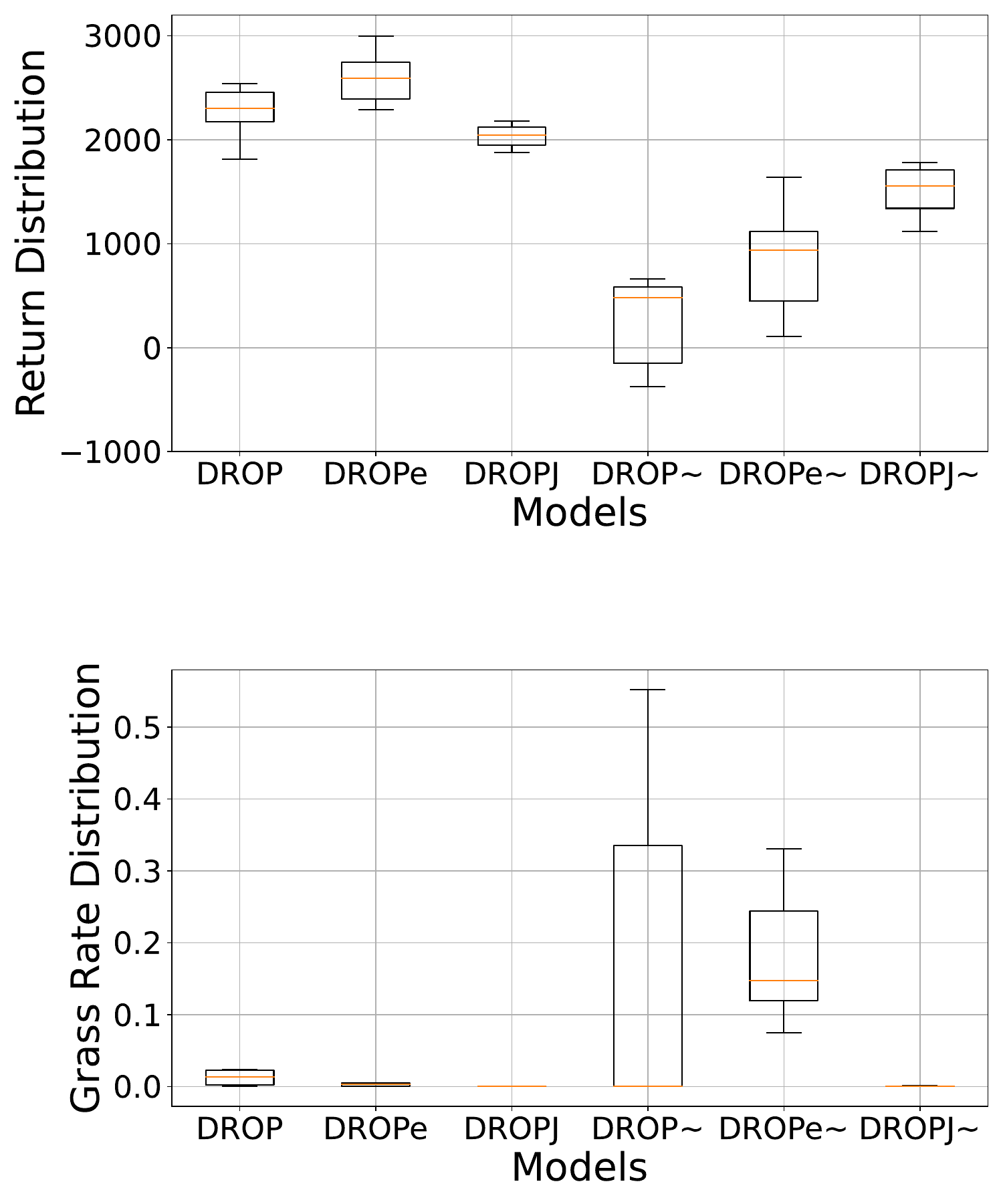}
        \caption{Effect when $60\%$ of answers in non-critical queries is erroneous.}
        \label{fig:impact_non}
    \end{subfigure}
    \hfill
    \begin{subfigure}[t]{0.49\textwidth}
        \centering
        \includegraphics[width=\textwidth]{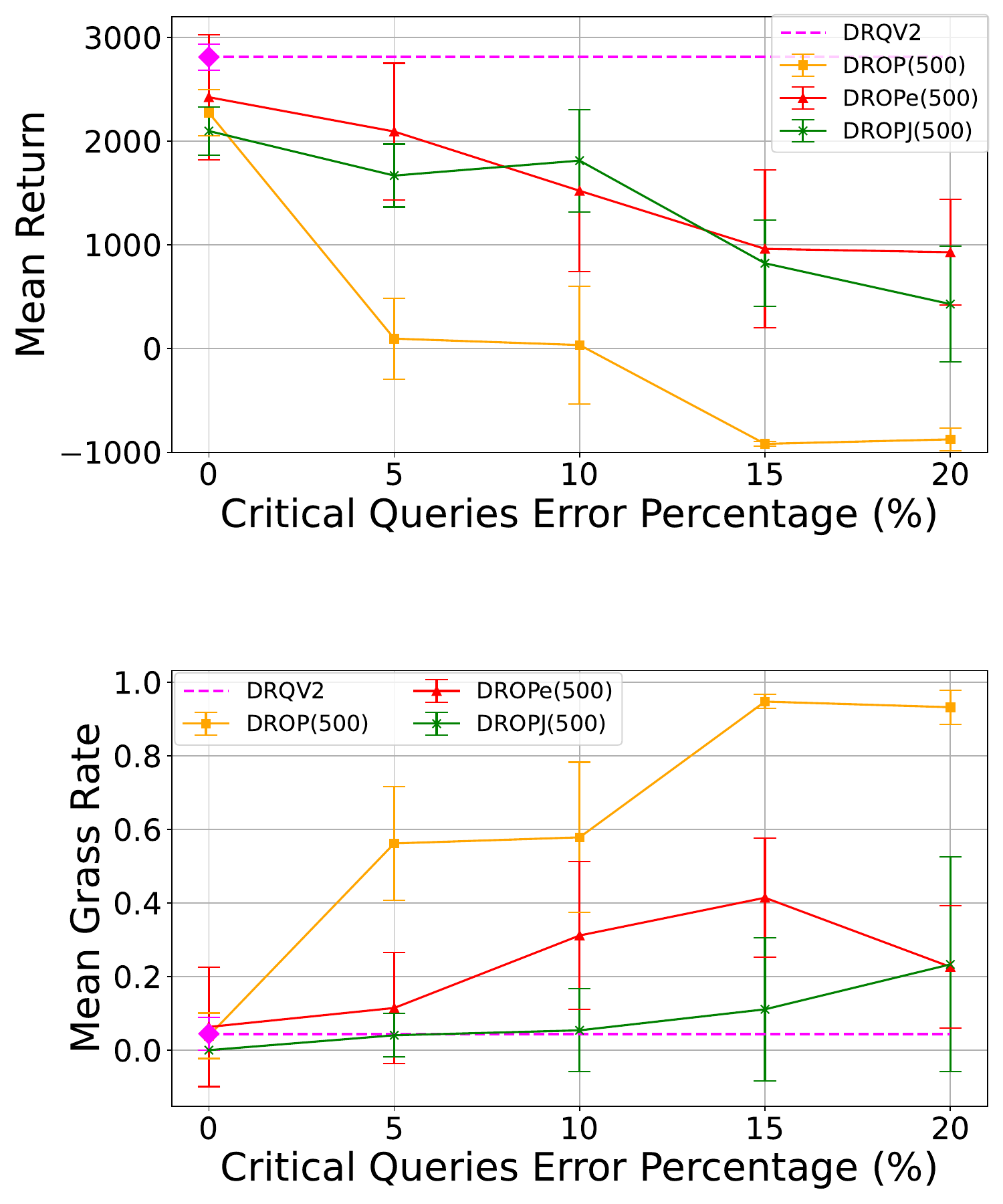}
        \caption{Effect when a percentage of answers in critical queries is erroneous.}
        \label{fig:impact_cr}
    \end{subfigure}
    \caption{Effect on best-trained models via synthetic erroneous feedback \cite{kazantzidis2026safe}.}
    \label{fig:erroneous_queries}
\end{figure}

Moreover, another labeller annotated a subset of queries with DROP and DROPe to test the consistency in feedback time, and all the queries with DROPJ to extend the insights from the previous paragraph on the effect of diverse feedback in relation to justifications. The new labeller had the same good understanding of the feedback protocol with the original labeller, but there was freedom at debatable queries. The average feedback time per query of the new labeller was: $t_{\mathit{DROP}}=6.39 \pm 5.49$, $t_{\mathit{DROPe}}=7.48 \pm 7.24$ and $t_{\mathit{DROPJ}}=8.09 \pm 5.15$. Comparing to the last three rows of Table \ref{tab:times}, these measurements are all close and consistent with the original labeller's trials. Regarding the consistency between annotators' answers on DROPJ, the most disagreements were on clips with hallucinations or ambiguous/debatable cases, such as: in one clip the car goes to grass and in the other it comes inside the road but not fully; a car slightly veers towards the grass but quickly corrects; a car takes a turn narrowly and gets close to the kerb instead of taking it more widely; driving close to the grass; and awkward manoeuvres or zigzag moves. The analytic breakdown of \textbf{all 600 queries} along with their reasons can be summarised as follows:

\begin{itemize}
\item \textbf{172 skipped from both} (due to various reasons detailed in the Appendix, including hallucinations and ambiguous cases)
\item \textbf{325 critical queries}, of which:
    \begin{itemize} 
    \item  287 agreements
    \item 32 one answered and the other skipped, due to:
        \begin{itemize}
        \item hallucinations considered significant from only one
        \item ambiguous/debatable cases
        \end{itemize}
    \item  6 disagreements (could have been also skipped), due to:
        \begin{itemize}
        \item oversights
        \item ambiguous/debatable cases
        \end{itemize}
    \end{itemize}
\item \textbf{103 non-critical queries}, of which:
    \begin{itemize}
    \item 75 agreements
    \item 6 that one answered and the other skipped, due to:
        \begin{itemize}
        \item hallucinations
        \end{itemize}
    \item 22 disagreements, due to
        \begin{itemize}
        \item ambiguous/debatable cases
        \end{itemize}
    \end{itemize}     
\end{itemize}

\noindent Thus, ignoring skipped answers, the \textit{critical disagreement rate} was only $100\times\frac{6}{287+6} \approx  2\%$ and the \textit{non-critical disagreement rate} was $100\times\frac{22}{75+22} \approx  23\%$. The averaged results of the new annotator gave for DROPJ a return of $2086.0 \pm 255.2$ and a grass rate of $0$. Hence, compared to the DROPJ trained from the original labeller (Figure \ref{fig:best_models}), the grass rate was again minimised (although we noticed that during deployment the car would approach the road-grass boundary slightly more often, likely because the second annotator treated edge cases at the boundary as mostly safe, whereas the first annotator either marked them unsafe or skipped them), and the performance remained very close. The Mann-Whitney U test for the return showed no significant difference with $p = 0.677$. So, extending the observations from the synthetic diverse feedback experiments, we have more indications that for a small contradiction in critical, and even a bigger one in non-critical queries in mainly debatable cases, DROPJ still remains robust.

Finally, for DROPJ, we implemented a simple solution for automatic detection of possible wrong answers:

\begin{enumerate}

\item  Use an initial error-free set of answers to train an ensemble of $M$ reward models 

\item During normal annotation, derive for a query the preference prediction likelihood for each member $m$ of the ensemble $\hat{\mu}_{m}=p[P= 1^{\mathit{st}}|\hat{r}_m], m=1,...,M$

\item Estimate the mean $\hat{\mu}_{\mathit{mean}}$ and the uncertainty $\hat{\mu}_{\mathit{std}}$ of the ensemble

\item If $|\mu - \hat{\mu}_{\mathit{mean}}| \leq \tau_{\mathit{mean}} \ \ \text{and} \ \ \hat{\mu}_{\mathit{std}} \leq \tau_{\mathit{std}}$, where $\mu$ is the user's label (Table \ref{tab:equivalence} or Equation \ref{eq:gen_just}), and $\tau_{\mathit{mean}}$ and $\tau_{\mathit{std}}$ fixed thresholds, accept the answer; otherwise, prompt for re-checking.

\end{enumerate}

We preliminarily tested this method, not from the outset of our experiments during the original user annotation, but afterwards by constructing a reward ensemble from a number of correct feedback answers and corrupting intentionally some other answers in order to see whether they could be detected. We found that $\tau_{\mathit{mean}}$ and $\tau_{\mathit{std}}$ values, that detected debatable queries and did not repeatedly prompt for re-checking, fell in the range of 0.01 to 0.1. Generally this range would depend on the precision of the predictor, the domain, and the consistency demanded. However, we saw that clear-cut errors, such as oversights, could be easily identified with high thresholds. Moreover, incrementally adding verified (correct) annotations to the training set and retraining the reward ensemble could further improve its predictive accuracy and error-detection capability.

\subsection{Impact of Multiple Safety Justifications}
\label{sec:impact_multi_justifications}

Finally, we examine the impact when multiple safety justifications are used in OCR (Figure \ref{fig:obst_car}). We note that the time per query with multiple justifications increased slightly compared to the single justification case of DROPJ ($7.88 \pm 5.83$ in Table \ref{tab:times}), to $8.66 \pm 7.58$ with chuckholes included, and $9.07 \pm 7.85$ with chuckholes and cars. These small increases were expected due to the higher complexity of the environment requiring slightly more cognitive effort.

\begin{figure}
    \centering
    \begin{subfigure}[t]{0.49\textwidth}
        \centering
        \includegraphics[width=\textwidth]{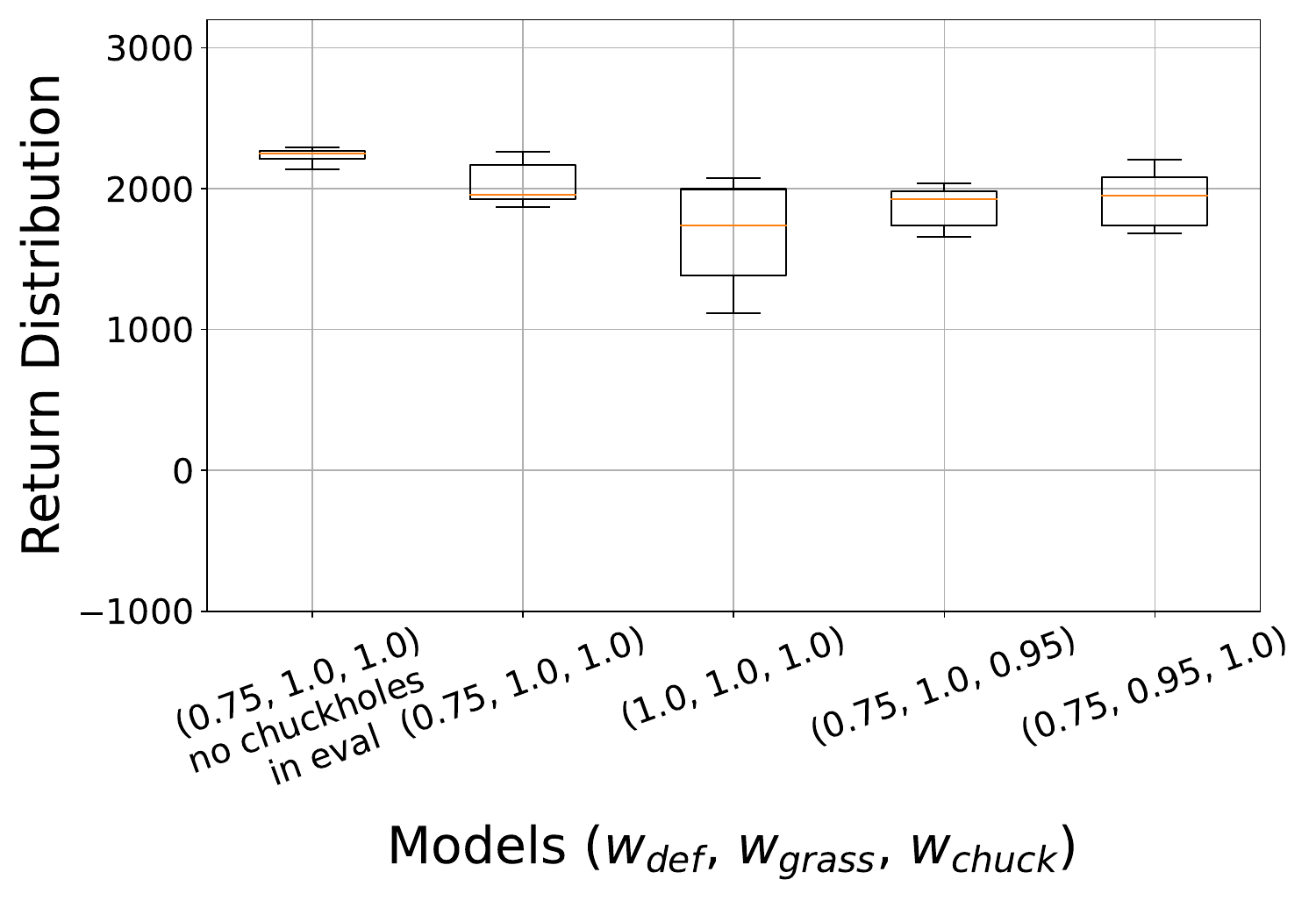}
    \end{subfigure}
    \hfill
    \begin{subfigure}[t]{0.49\textwidth}
        \centering
        \includegraphics[width=\textwidth]{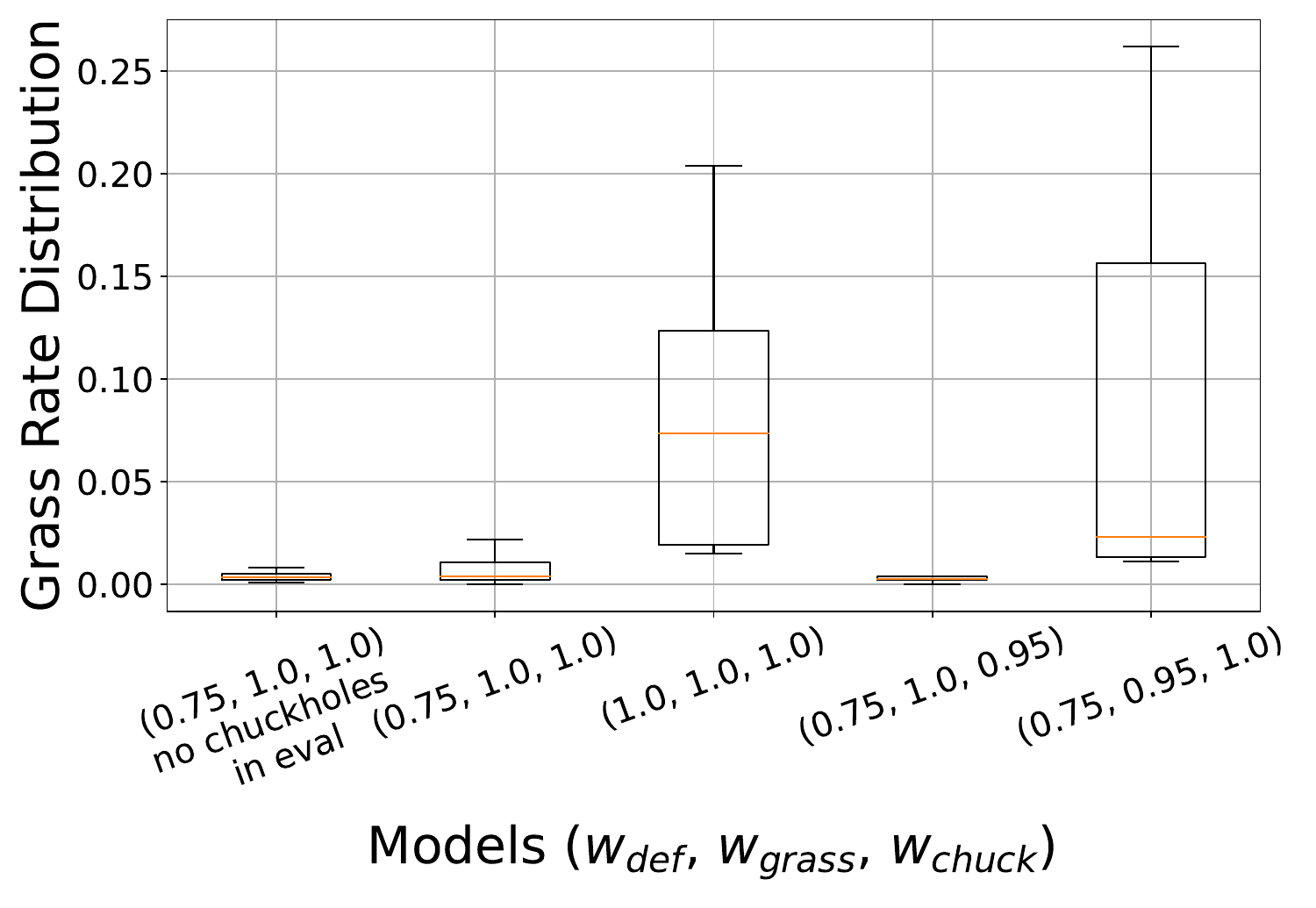}
    \end{subfigure}
    \begin{subfigure}[t]{0.49\textwidth}
        \centering
        \includegraphics[width=\textwidth]{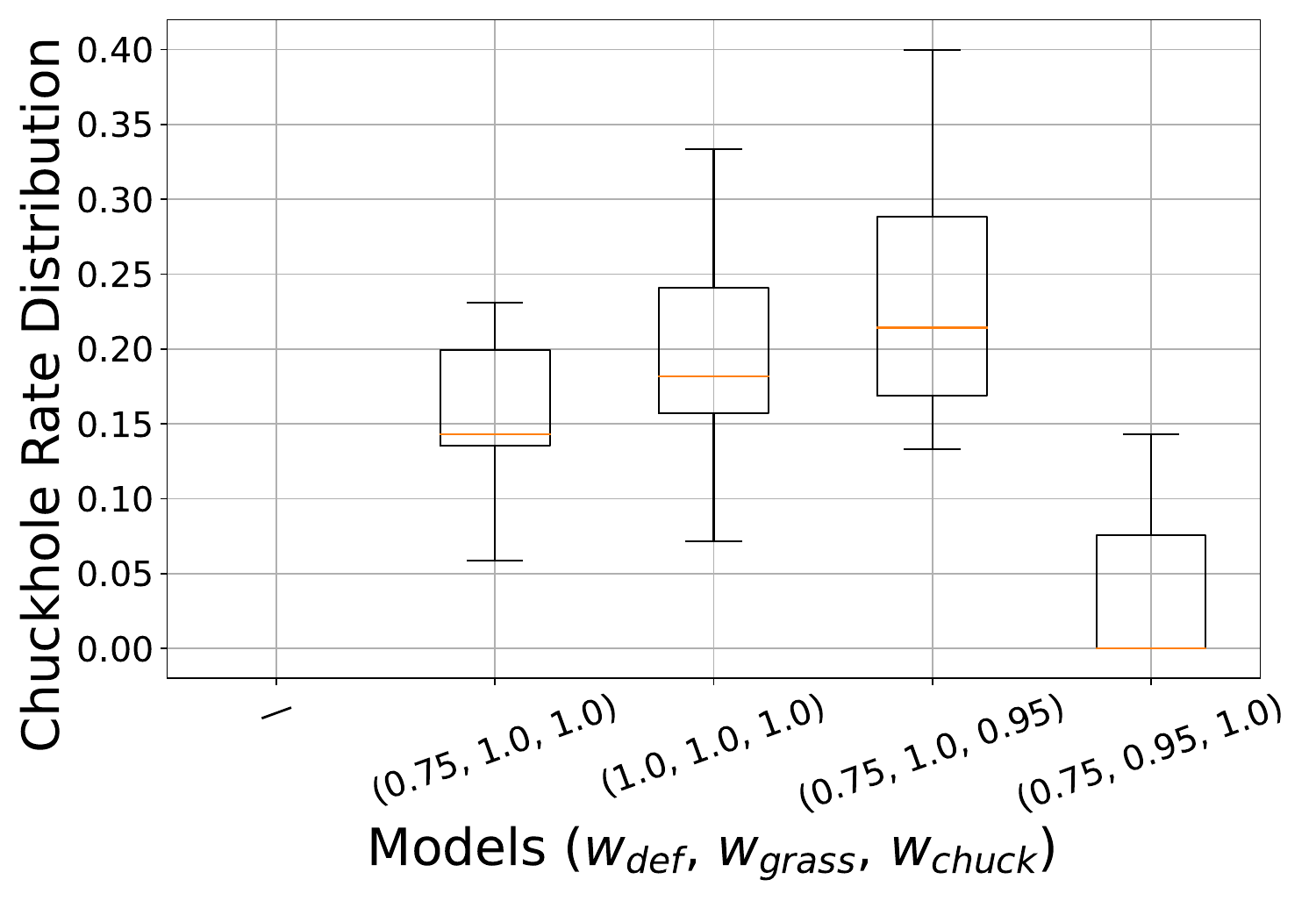}
    \end{subfigure}
    \caption{\textbf{Q4:} Impact of including multiple justifications in preferences (OCR with only chuckholes) \cite{kazantzidis2026safe}.}
    \label{fig:impact_multi}
\end{figure}

Initially, when only chuckholes are added on the road, Figure \ref{fig:impact_multi} shows the metrics distributions of models with justification weights $(w_{\mathit{def}}, w_{\mathit{grass}}, w_{\mathit{chuck}})$. As in \textbf{(Q3)} with a single safety justification, here as well, the combination $(0.75, 1, 1)$, i.e.\ making the safety justification weights $w_{\mathit{grass}}$, $w_{\mathit{chuck}}$ higher than the default weight $w_{\mathit{def}}$, yields better safety (grass and chuckhole rates) compared to the $(1, 1, 1)$ combination. Now, the return of $(0.75, 1, 1)$ is also improved over that of $(1, 1, 1)$, with one reason being that the latter gets stuck more often in chuckholes, and consequently loses time to cover more tiles. Also, the same $(0.75, 1, 1)$ model, although trained with chuckholes included, can still be used efficiently when no chuckholes appear during deployment, with, however, a few grass violations (possibly due to world model imperfections, such as correct reconstruction and prediction of the kerb). Additionally, keeping $(0.75, 1, 1)$ as reference, the grass rate can be minimised with a combination such as $(0.75, 1, 0.95)$, i.e.\ making $w_{\mathit{grass}}$ higher than $w_{\mathit{chuck}}$. However, with such a combination a trade-off is expected in the chuckhole rate, which increases. We noticed during the trials, that $(0.75, 1, 0.95)$ almost never leaves the road, but it more often steps on chuckholes. Conversely, $(0.75, 0.95, 1)$ does the opposite. It nearly minimises the chuckhole rate, but at the cost of more steps on the grass. Interestingly, with this combination the car tends to keep a safe distance from the chuckholes, even if that means it has to step more on the grass. Finally, as we mentioned in Section \ref{sec:impact_justifications} and Figure \ref{fig:heatmap}, when a justification weight decreases imprudently, although it will likely favour an objective it is not associated with, it can also end up deteriorating objectives that its associated objective influences. For example, decreasing  $w_{\mathit{grass}}$ or $w_{\mathit{chuck}}$ significantly (e.g. to 0.6), it would have an immediate effect in the return; either because the car would barely consider the grass and it would get `lost in the forest', or because it would do the same with chuckholes and frequently get stuck on them.

When both chuckholes and cars are included, Figure \ref{fig:impact_multi_car} shows the metrics distributions of models with justification weights $(w_{\mathit{def}}, w_{\mathit{grass}}, w_{\mathit{chuck}}, w_{\mathit{car}})$. In this case, considering the additional complexity of the environment, both $(1, 1, 1, 1)$ and $(0.75, 1, 1, 1)$ achieve competitive returns (slightly lower than $2000$), maintaining balanced and diverse crash rates --- neither combination dominates. For instance, $(0.75, 1, 1, 1)$ exhibits a lower car rate, but higher grass and chuckhole rates. We kept as reference $w_{\mathit{def}}=1$, since the default justification seemed to contribute, not only in performance, but also indirectly in safety this time. The reason is that preference answers disfavoured clips showing the car approaching obstacles from a distance, even if not colliding with them within the clip. As expected, the same $(1, 1, 1, 1)$ model again performs well when evaluated with no obstacles, with only a few grass violations, most likely due to world model imperfections (e.g.\ in kerb). As before, adjusting the justification weights can improve the desired safety aspects. For instance, when compared to the reference $(1, 1, 1, 1)$ model, the $(1, 1, 0.98, 0.98)$ is more careful with the safety aspect of avoiding grass steps (grass rate almost zeros) than chuckholes and cars (their rates increase). The opposite happens for $(1, 0.98, 1, 1)$, where only very few chuckholes and no cars are contacted. As before, the agent-car tends to keep a safe distance from chuckholes and cars, even if it has to drive more often in the grass. The previous observations show that we can control the original balance of safety metrics as desired by adjusting accordingly the justification weights, giving greater or lesser emphasis to specific safety requirements.

\begin{figure}
    \centering
    
    \begin{subfigure}[b]{0.49\textwidth}
        \centering
        \includegraphics[width=\textwidth]{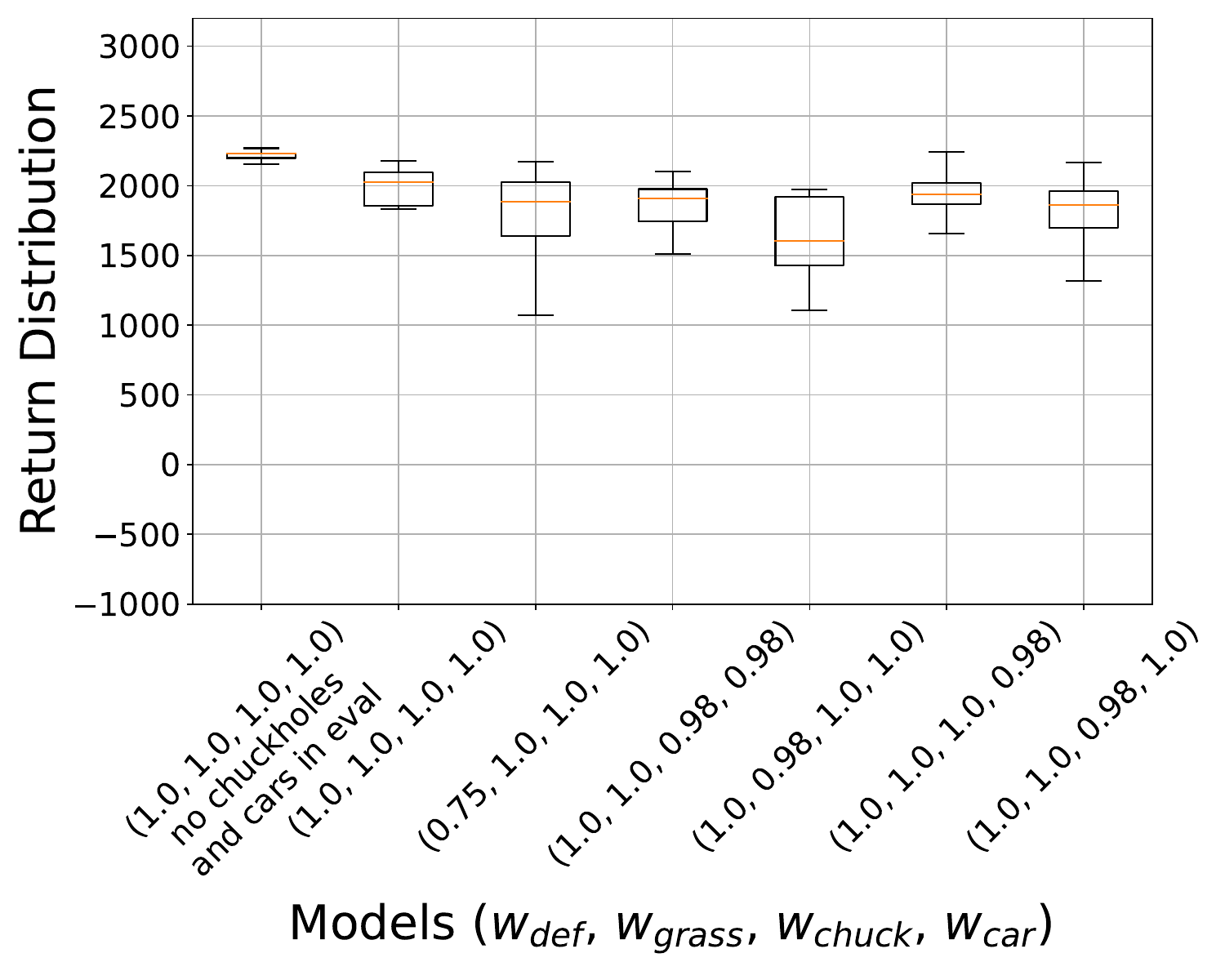}
    \end{subfigure}
    \hfill
    \begin{subfigure}[b]{0.49\textwidth}
        \centering
        \includegraphics[width=\textwidth]{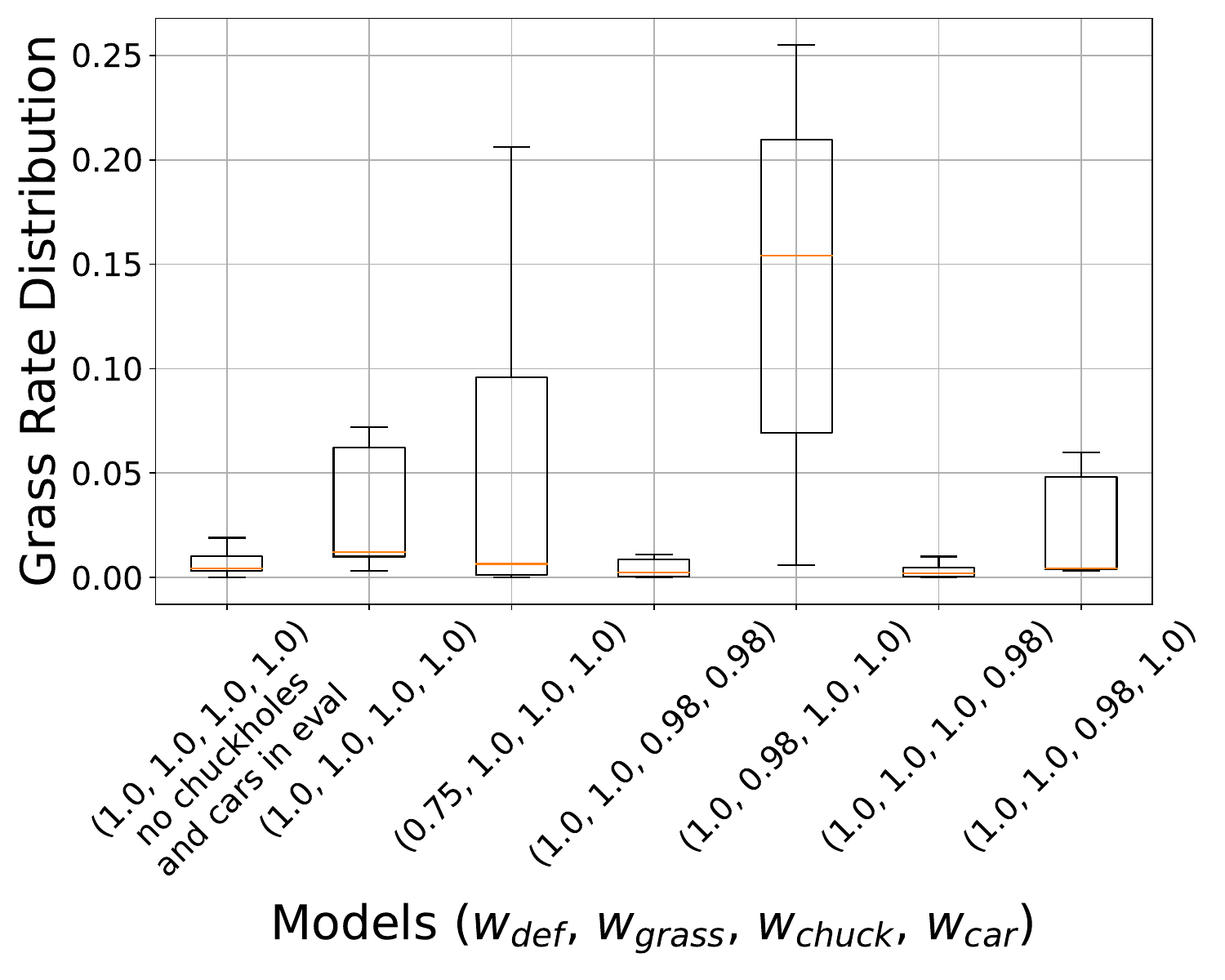}
    \end{subfigure}
    
    \vspace{0.6cm}
    
    \begin{subfigure}[b]{0.49\textwidth}
        \centering
        \includegraphics[width=\textwidth]{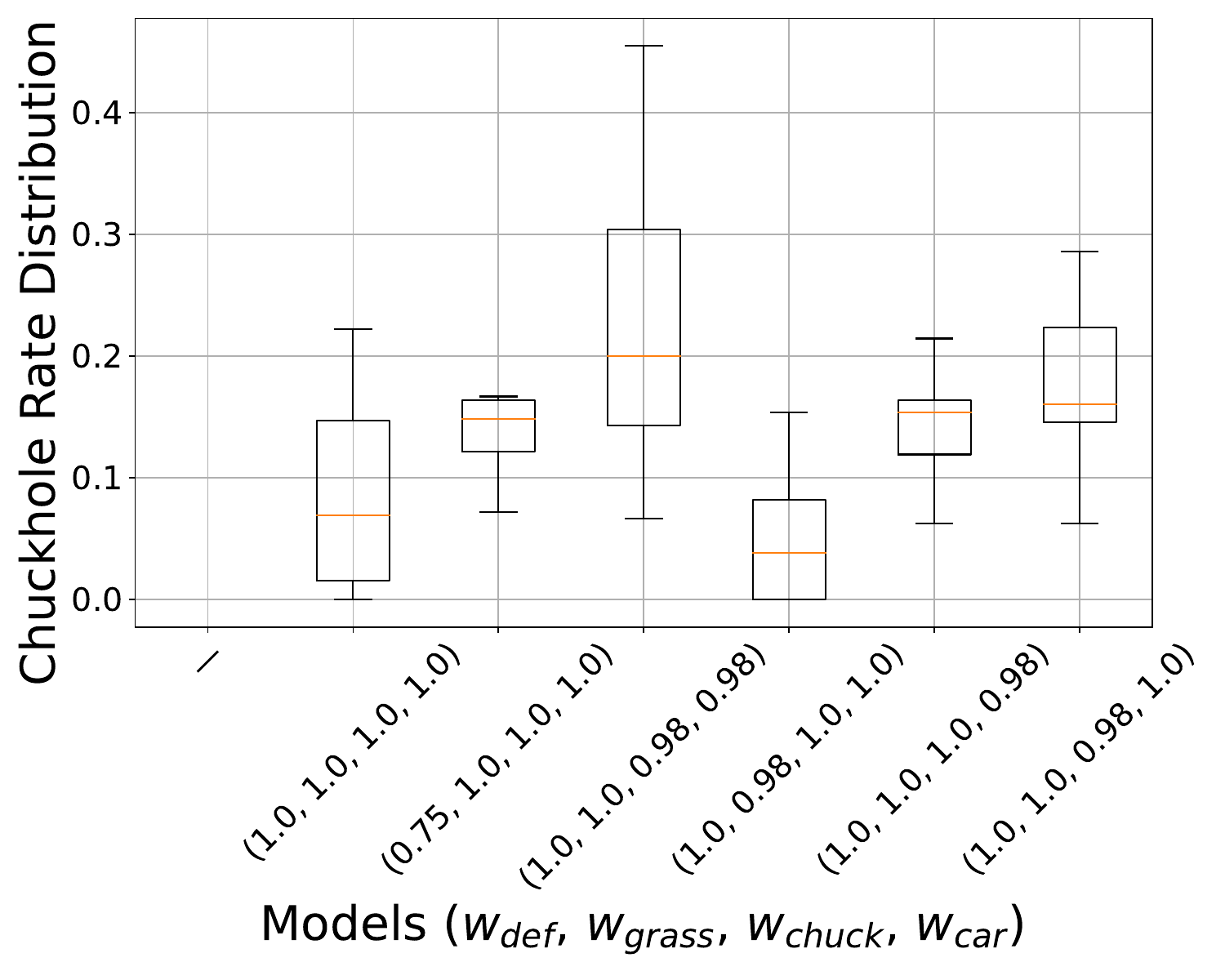}
    \end{subfigure}
    \hfill
    \begin{subfigure}[b]{0.49\textwidth}
        \centering
        \includegraphics[width=\textwidth]{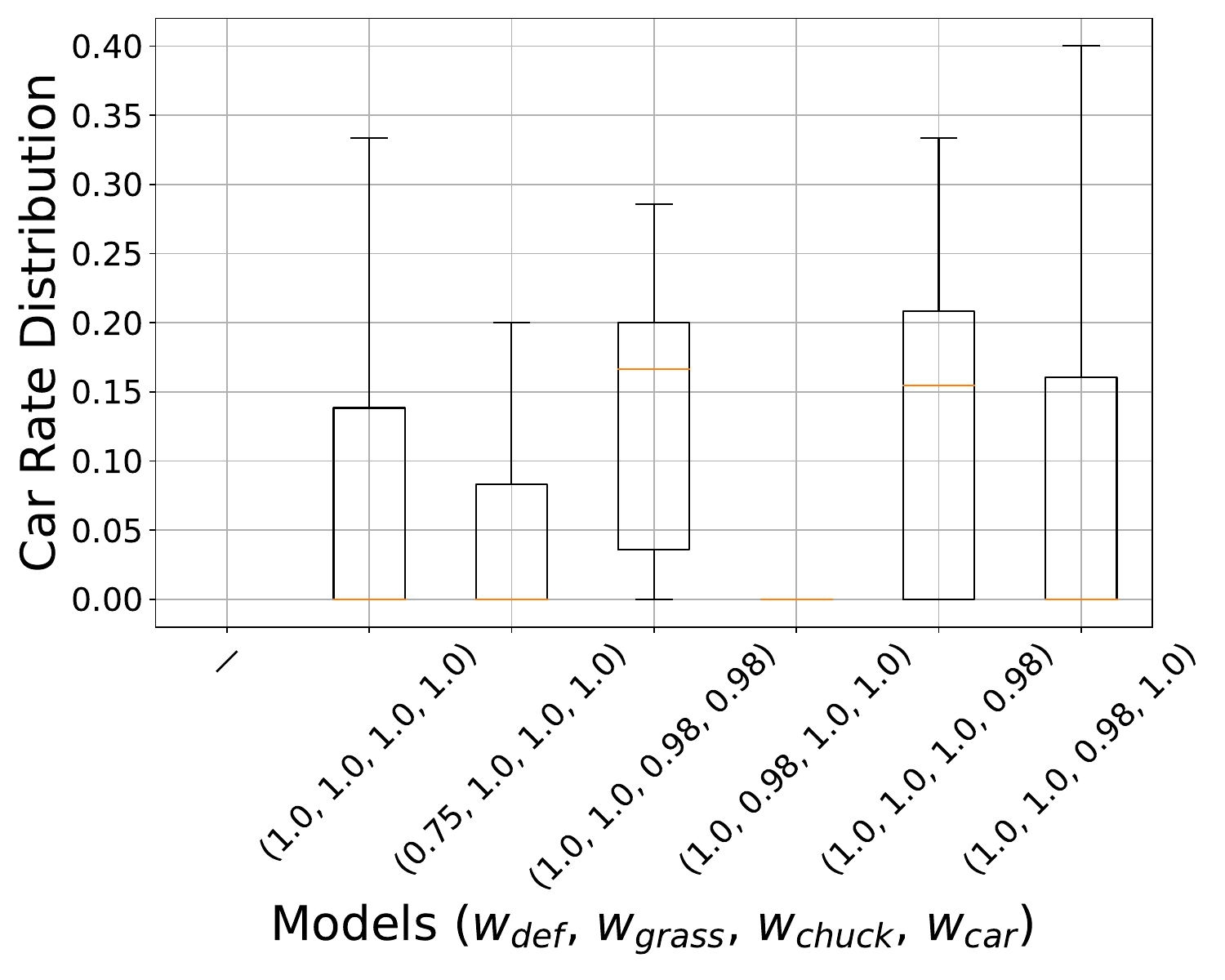}
    \end{subfigure}
    
    \caption{\textbf{Q4:} Impact of including multiple justifications in preferences (OCR with both chuckholes and cars) \cite{kazantzidis2026safe}.}
    \label{fig:impact_multi_car}
\end{figure}

However, as mentioned before, justification weights should be used with consideration if all safety aspects are critical or one affects another. For instance, with $(1, 1, 1, 0.98)$, as expected, the grass rate (given higher priority) is minimised, and the less-considered car rate increases. However, unexpectedly, the chuckhole rate also slightly increases. From inspection, this happens not because the agent-car considers chuckholes less, but because of the \textit{carry-over effect}, where a collision with a scripted-car often pushes the agent-car towards a nearby chuckhole. This is more evident in $(1, 1, 1, 0.75)$, i.e.\ when decreasing the car weight considerably, where the more frequent collisions with scripted-cars drag the agent-car off its route. The same effect, to a smaller degree, is observed in $(1, 1, 0.98, 1)$ (and more on $(1, 1, 0.75, 1)$)\footnote{Figure \ref{fig:obst_car} was extracted at $(1, 1, 0.75, 1)$, where the agent-car, safely passes the blue car, but while returning centrally, does not plan for the chuckhole ahead, due to the considerable decrease of $w_{\mathit{chuck}}$.}, where although the grass rate slightly decreases and the chuckhole rate increases, the latter causes a slight increase in the car rate as well (e.g.\ scripted-cars bump into the slowed-down agent-car in a chuckhole more often). 

Thus, according to our overall experimentation in OCR, if the safety aspects are independent, justifications can provide a straightforward way to prioritise one or another safety aspect. However, if there are dependencies among safety aspects, justifications can still consider one safety aspect more than another, but the justification weights should be set with care to avoid the carry-over effect. We aspire justification weights to evolve into standardised values for common tasks, while still allowing fine-tuning by human operators before training to account for environment-specific characteristics. Alternatively, in future work one could possibly employ evolution strategies such as \cite{hansen2001completely} to formulate an optimisation problem with soft constraints, where justification weights are learnable.

\section{Conclusion}
\label{sec:conclusion}

We presented DROPJ, a method that builds a world model from prior real-world trajectories and uses human input as the learning source. DROPJ maximises safety both during training and deployment, given unavailable environment dynamics and rewards. Using real-user experiments, we found that generating simulated trajectories from a user through a learned world model significantly improves the computational cost during training, as well as the performance during deployment. We showed that using preferences on dream queries, compared to other human feedback types, significantly enhances performance and can lead to a near-optimal policy in the environments tested. Most importantly, we saw how by using safety justifications to accompany preferences, we can improve safety or prioritise desired safety aspects associated with them. The only concern of DROPJ is its reliance on past real-world examples in order to build a robust world model, and (as generally in human-centred methods) on qualified/capable users to create more examples via the learned simulator or provide accurate feedback. Sometimes having these and in abundance, may be challenging beyond large industry labs. Some natural next steps in this research would be applying DROPJ in real hardware, such as a lab-based robot navigation task, or employing justifications with a world model in a different domain that uses learning from human preferences.

\begin{credits}
\subsubsection{\ackname} This work was supported by UKRI [EP/S024298/1] and EPSRC [EP/Y003187/1]. The authors also acknowledge the use of the IRIDIS High-Performance Computing Facility, and associated support services at the University of Southampton.
\end{credits}

\pdfbookmark[1]{Appendix}{appendix}
\section*{Appendix}
\label{app:dropj-app}

\subsubsection{Obstacle Car Racing (OCR)} The OCR environment (Figure \ref{fig:obst_car}) extends Car Racing (CR) \cite{brockman2016openai} with chuckholes and cars added along the track. While these could be spawned entirely randomly, we kept a minimum distance between two cars or two chuckholes in order to afford the typically very high demands in data and time to train a sufficiently-good world model able to plan \cite{assran2025v}, and avoid highly congested scenarios where safety metrics (crash rates) would be difficult to interpret. When at least one car wheel contacts a chuckhole its velocity is abruptly reduced to reflect a slowdown effect (more like stepping on glue), but the car continues moving due to momentum. The agent-car follows the learned MPC policy to drive, but the non-agent cars use a scripted controller to stay on the road and maintain a relatively low speed. Specifically, they identify the closest track tile (patch) and steer proportionally towards it, reducing the angular difference between its heading and the tile direction. Also, they compute the gas based on the difference between the current and the target speed, and slow down on sharp turns by estimating angular changes between consecutive tiles.

To create the offline dataset $\mathcal{R}$, we generated by playing the game from the keyboard $M=600$ real-world trajectories (episodes) for CR and OCR only with chuckholes, and $M=800$ for OCR with chuckholes and cars. The episode length was $L=1000$. For CR the returned states were $96\times96$ pixels, but included the score for the 12 pixels at the bottom, which were cropped, as well as 6 pixels from the left and right side (not useful as they normally show just the grass), giving the state space $\mathcal{S}=\mathbb{R}^{84\times84\times3}$ of the RGB images. For OCR, as it was more complex, we made the state space $\mathcal{S}=\mathbb{R}^{128\times128\times3}$ from the start. In Step 2, we generated by playing the game in the learned simulator $N=10$ dream user trajectories for CR, and $N=60$ for OCR to acquire slightly more diversity given the higher complexity. The generation of one trajectory needed only 30s.

Training used early stopping based on convergence criteria and dynamically adjusted the epochs and learning rate to avoid overfitting. All methods used the same world model based on \cite{ha2018recurrent} and \cite{reddy2020learning}, reimplemented from TensorFlow to PyTorch with a few improvements in the architecture and hyperparameters after experimentation (Table \ref{tab:vae_configuration} for the VAE and Table \ref{tab:mdnrnn} for the dynamics model). The full world model was trained at a High-Performance Computing cluster using a single NVIDIA Tesla V100 GPU in overall (encoder+dynamics model) around 10 hours for CR, 33 hours for OCR only with chuckholes, and 42 hours for OCR including chuckholes and cars. The VAE needed around 80\% of the overall time. For DRQV2 we used the same implementation and hyperparameters as in \cite{yarats2021mastering}.

\begin{table}[t]
\caption{VAE model architecture and hyperparameters \cite{kazantzidis2026safe}.}
\centering
{
\setlength{\tabcolsep}{5mm}
\begin{tabular}{lll}
\toprule
\multicolumn{1}{l}{\textbf{Arch./Hyperparameters}} & \multicolumn{1}{l}{\textbf{CR}} & \multicolumn{1}{l}{\textbf{OCR}} \\ \midrule

Input size & $84\times84\times3$ & $128\times128\times3$ \\ 
\midrule

\multirow{6}{*}{\makecell[l]{Encoder \\CNN Layers}}
    & \multicolumn{2}{c}{(filters, kernel, stride, padding)} \\
    & (24, 4$\times$4, 2, 0)          & (64, 4$\times$4, 2, 1, 0) \\
    & (48, 4$\times$4, 2, 0)          & (128, 4$\times$4, 2, 1) \\
    & (96, 4$\times$4, 1, 0)          & (256, 4$\times$4, 2, 1) \\
    & (192, 4$\times$4, 1, 0)         & (512, 4$\times$4, 2, 1) \\
    & (384, 4$\times$4, 1, 0)         & (1024, 4$\times$4, 2, 1) \\ \midrule

\multirow{2}{*}{\makecell[l]{Linear Layers for\\$\mu$ and $\sigma$}}

    & 1536 to $d=32$                                      & 16384 to $d=64$ \\
    & $d=32$ to 1536 
    & $d=64$ to 16384 
    \\ \midrule

\multirow{6}{*}{\makecell[l]{Decoder Transpose\\CNN Layers}}
& \multicolumn{2}{c}{(filters, kernel, stride, padding)} \\
    & (192, 4$\times$4, 1, 0)          & (512, 4$\times$4, 2, 1) \\
    & (96, 4$\times$4, 1, 0)           & (128, 4$\times$4, 2, 1) \\
    & (48, 5$\times$5, 2, 0)           & (64, 4$\times$4, 2, 1) \\
    & (24, 5$\times$5, 2, 0)           & (32, 4$\times$4, 2, 1) \\
    & (3, 4$\times$4, 2, 0)            & (3, 4$\times$4, 2, 1) \\ 

\midrule

KL Divergence Threshold 
    & 0.5                                                 & 2 \\
Optimiser           
    & Adam                        & Adam \\
Mini-batch size       
    & 32                                                  & 32 \\
Learning Rate        
    & 10\textsuperscript{-4} $\rightarrow$ 10\textsuperscript{-5} (dyn)      & 10\textsuperscript{-4} \\ \bottomrule

\end{tabular}
}
\label{tab:vae_configuration}
\end{table}

\begin{table}[h!]
\caption{Dynamics model (MDN-RNN) hyperparameters \cite{kazantzidis2026safe}.}
\centering
{
\setlength{\tabcolsep}{6mm}
\begin{tabular}{ll}
\toprule
\textbf{Hyperparameters}            & \textbf{Values} \\ \midrule

Number of LSTM layers  & 1 \\
Number of LSTM units $m$ & 256 (1024 for OCR) \\
Gaussian Mixtures                            & 5  (7 for OCR)             \\
Temperature $\tau$ &    1.0 \\ 
Optimiser                            & Adam              \\
Mini-batch size                           & 32              \\
Maximum Gradient Value                   & 1.0 \\
Learning rate                        & 10\textsuperscript{-3} $\rightarrow$ 10\textsuperscript{-4} (dyn)             \\
\bottomrule
\end{tabular}
}
\label{tab:mdnrnn}
\end{table}

\pdfbookmark[2]
  {Reward Model from Sparse Reward Labels}
  {app-sparse-reward}
\subsection*{Reward Model from Sparse Reward Labels}
\label{app:rew_sparse}
For both ReQueST and DROS, the user annotated with sparse reward labels segments of length $k=50$. We used the same classifier-based reward model from \cite{reddy2020learning} reimplemented in PyTorch, with the hyperparameters shown in Table \ref{tab:sparse_model}. The feedback was given from a real user, instead of a simulated one using a ground-truth reward model trained from a lot of expert and random real-world demonstrations. That allowed more reliable comparisons. Figure \ref{fig:sparse} shows the GUI that the user used to provide sparse labels. The action shown represents $[\mathit{steer}, \mathit{gas}, \mathit{brake}]$ with the ranges $\mathit{steer}\in[-1.0, 1.0]$, $\mathit{gas} \in [0.0, 1.0]$ and $\mathit{brake}\in[0.0, 1.0]$. In the Figure's example, the car had the wheel turned completely right ($1.0$) on the last action, and the car is already inside the grass. So the correct answer is \textit{Grass/Kerbs} (unsafe). Patches (tiles) are not clearly visible, which is reasonable given the simulator (world model) is not perfect. However, based on the gas of an action, and clues such as closer proximity to a turn (meaning the car is moving forward) the label \textit{New Tile} (good) was chosen. When the scene does not change (e.g.\ when the gas is not pressed at all or the brake is pressed consistently) and the car is inside the road, the answer should be \textit{Road} (neutral), meaning that the car is on a previously-visited tile. The user also has the option \textit{Skip query} to skip a label when a query is not clear, e.g.\ when the frame contains hallucinations (artefacts) \cite{zhao2025rejecting} due to the dynamics model imperfections.

\begin{table}[t]
\caption{Reward model from sparse reward labels \cite{kazantzidis2026safe}.}
\centering
{
\setlength{\tabcolsep}{5mm}
\begin{tabular}{ll}
\toprule
\textbf{Hyperparameters}            & \textbf{Values} \\ \midrule
Reward constants & $R_{\mathit{good}} = 10, \ R_{\mathit{unsafe}} = -1, \ R_{\mathit{neutral}} = 0$ \\
Number of FNN layers & 5 \\
Number of FNN units & 256 \\
Number of ensemble networks & 4 \\
Mini-batch size                      & 32              \\
Optimiser                            & Adam              \\
Learning rate                        & 10\textsuperscript{-3}             
       
\\ \bottomrule
\end{tabular}
}
\label{tab:sparse_model}
\end{table}

\begin{figure*}
    \centering
    \includegraphics[width=0.6\textwidth]{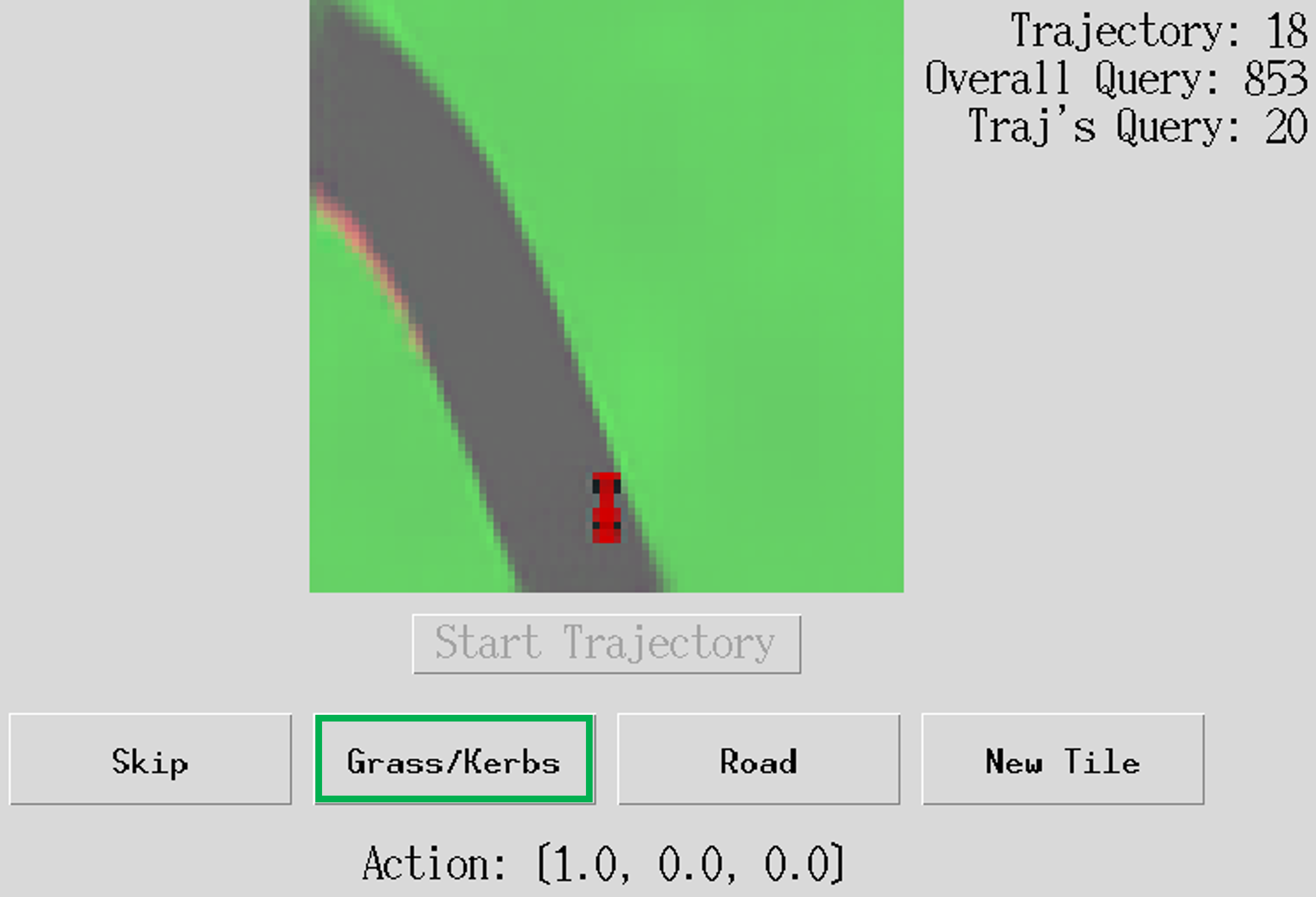}
    \caption{GUI of the sparse reward labels feedback in ReQueST and DROS \cite{kazantzidis2026safe}.}
    \label{fig:sparse}
\end{figure*}

\pdfbookmark[2]
  {Reward Model from Preferences (and Justifications)}
  {app-preferences}
\subsection*{Reward Model from Preferences (and Justifications)}
\label{app:rew_pref}

From the dream user trajectories in $\mathcal{D}$, we sampled uniformly at random pairs of trajectory segments of length $k=20$. For the preference reward model we adapted the implementation from \cite{lee2021pebble}. The addition was the justifications with weights $w_s$, $w_{\mathit{def}}$ in CR and $w_{\mathit{car}}$, $w_{\mathit{chuck}}$, $w_{\mathit{grass}}$, $w_{\mathit{def}}$ in OCR, which shape the preference label $\mu$ differently. For CR the number of preference queries was $K=500$ until the best model (Table \ref{tab:times}), which was increased in OCR to $K=1000$ with added chuckholes, and $K=1500$ with added chuckholes and cars, to compensate for the higher complexity and the more queries that had to be skipped due to more world model hallucinations. Other hyperparameters are given in Table \ref{tab:pref_mod}.

\begin{table}[t]
\caption{Reward model from preferences and justifications \cite{kazantzidis2026safe}.}
\centering
{
\setlength{\tabcolsep}{1.5mm}
\begin{tabular}{ll}
\toprule
\textbf{Hyperparameters}            & \textbf{Values} \\ \midrule
Justification weights & $w_s = 1, \ w_{\mathit{def}} = 0.75$ and Figure \ref{fig:heatmap} (CR) \\
 & $w_{\mathit{car}}$, $w_{\mathit{chuck}}$, $w_{\mathit{grass}}$, $w_{\mathit{def}}$: Figures \ref{fig:impact_multi} and \ref{fig:impact_multi_car} (OCR)\\
Number of FNN layers & 3 \\
Number of FNN units & 256 \\
Number of ensemble networks & 3 \\
Mini-batch size                      & 128              \\
Optimiser                            & Adam              \\
Learning rate                        & 3$\times10$\textsuperscript{-4}             
       
\\ \bottomrule
\end{tabular}
}
\label{tab:pref_mod}
\end{table}

Figure \ref{fig:gui_drop_drope_dropj} shows the GUIs of DROP, DROPe and DROPJ, all on the same example of Figure \ref{fig:clip} for the single safety justification case. On the left clip the car has barely moved, and on the right the car has taken a nice turn. With DROP (Figure \ref{fig:gui_drop}), the correct answer is \textit{Right} ($\mu=0$). Similarly, a positive preference would have been given to a clip that: the car was staying on the road, while in the other clip the agent was going to, staying or coming from the grass; the car was covering a bigger distance than the car of the other clip; the car was taking a difficult (sharp) turn, while the other car did not show an interesting behaviour. With DROPe (Figure \ref{fig:gui_drope}), the answer to the example is again \textit{Right} ($\mu=0$). However, in contrast to DROP, the \textit{Equally} button is now activated for the following cases: in both clips the car is going from the road to the grass; in both clips the car is staying on the grass; in both clips the car is coming from the grass to the road; in one clip the car is staying on the grass and in the other the car is going from the road to the grass (both unsafe); in both clips the cars stay on the road and have a similar behaviour, i.e.\ covering roughly the same distance and none of them taking a difficult turn. Finally, DROPJ (Figures \ref{fig:gui_dropj1}--\ref{fig:gui_dropj3}) first makes two queries about the safety on each clip. The answer in the example is \textit{Yes} for both clips, since in both the car stays safe inside the road. If one answer was \textit{Yes} and the other \textit{No}, then there would not be a third question about the preference ($\mu$ would automatically become 1 or 0). Similarly, there would not be a third question if the answer started with \textit{No}/\textit{No}, since if they are both unsafe, the preference is essentially `equally bad' ($\mu=0.5$). Going to the third question when both are safe, the answer of DROPJ is always the same with the one of DROPe and for the same reasons; e.g.\ in the example it is \textit{Right} (Figure \ref{fig:gui_dropj3}). However, now the label becomes $\mu=0.25$, `communicating' to the reward model that this example did not involve unsafe behaviour. The user can also repeat the clips if needed with \textit{Repeat Video}. Finally, the option \textit{Skip query} is available to all algorithms when the answer is unclear. Such cases are when: there are hallucinations from imperfections in the dream; one clip shows the car coming from the grass to the road, while the other going from the road to the grass (both are unsafe in reality, so neither do we want to express a preference for the first, nor mark them as equal, since the latter would be unintuitive and could mislead the reward model); one clip shows the car coming from the grass to the road, while the other one staying on the grass (for the same reason as before). For DROP, \textit{Skip query} is also pressed for all the cases that DROPe or DROPJ would choose the \textit{Equally} answer. 

\begin{figure}
\centering

\setcounter{subfigure}{0}
\begin{subfigure}[t]{0.49\textwidth}
    \centering
    \includegraphics[width=\textwidth]{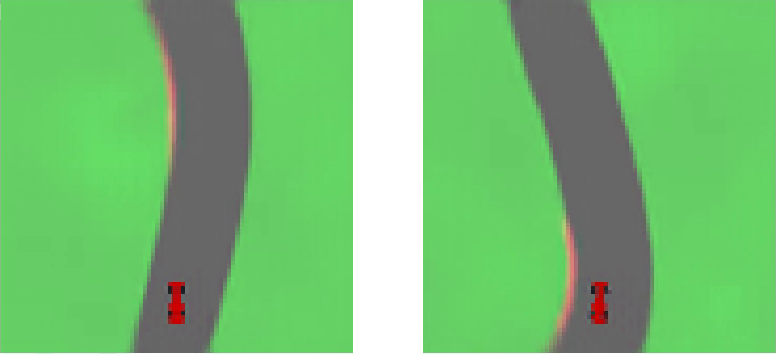}
    \caption{Final frame of the two compared clips. On the left, the car has barely moved; and on the right, it has safely taken a nice turn.}
    \label{fig:clip}
\end{subfigure}
\setcounter{subfigure}{3}
\begin{subfigure}[t]{0.49\textwidth}
    \centering
    \includegraphics[width=\textwidth]{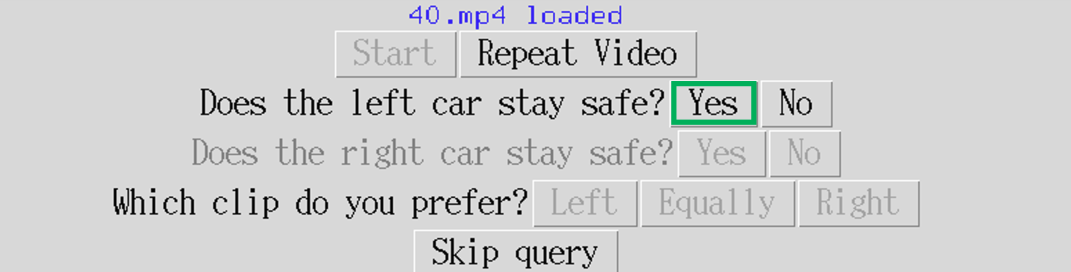}
    \caption{DROPJ --- first query about safety (left clip). Ans: \textit{Yes}.}
    \label{fig:gui_dropj1}
\end{subfigure}

\vspace{0.25cm}

\setcounter{subfigure}{1}
\begin{subfigure}[t]{0.49\textwidth}
    \centering
    \includegraphics[width=\textwidth]{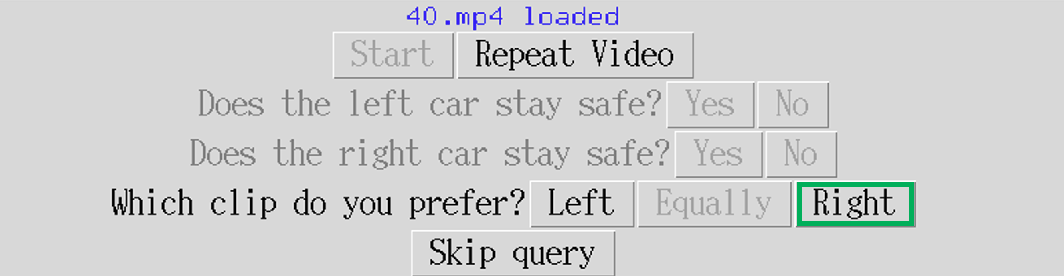}
    \caption{DROP (\textit{Equally} button is deactivated). Ans: \textit{Right} $\rightarrow$ $\mu=0$.}
    \label{fig:gui_drop}
\end{subfigure}
\setcounter{subfigure}{4}
\begin{subfigure}[t]{0.49\textwidth}
    \centering
    \includegraphics[width=\textwidth]{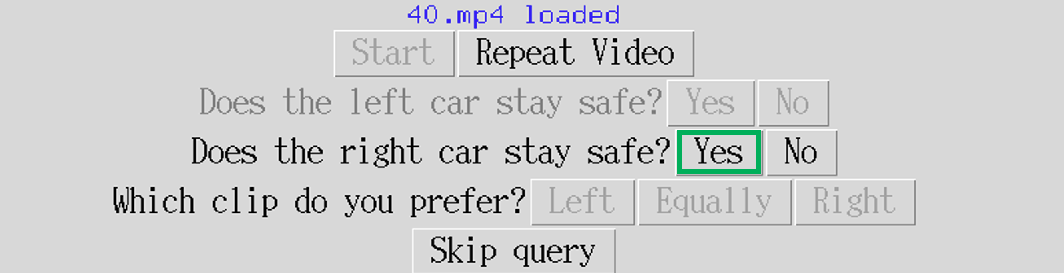}
    \caption{DROPJ --- second query about safety (right clip). Ans: \textit{Yes}.}
    \label{fig:gui_dropj2}
\end{subfigure}

\vspace{0.25cm}

\setcounter{subfigure}{2}
\begin{subfigure}[t]{0.49\textwidth}
    \centering
    \includegraphics[width=\textwidth]{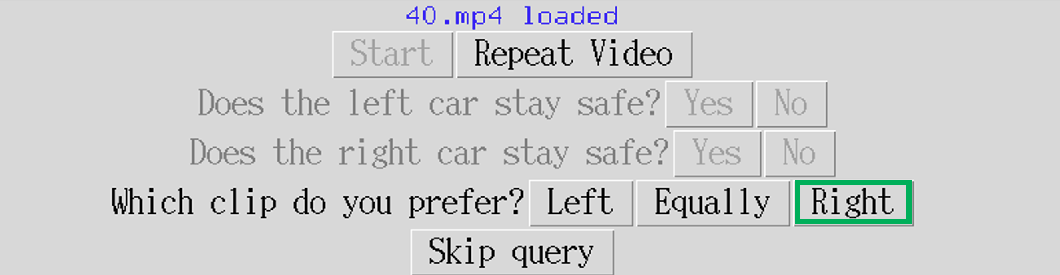}
    \caption{DROPe (\textit{Equally} button is included). Ans: \textit{Right} $\rightarrow$ $\mu=0$.}
    \label{fig:gui_drope}
\end{subfigure}
\setcounter{subfigure}{5}
\begin{subfigure}[t]{0.49\textwidth}
    \centering
    \includegraphics[width=\textwidth]{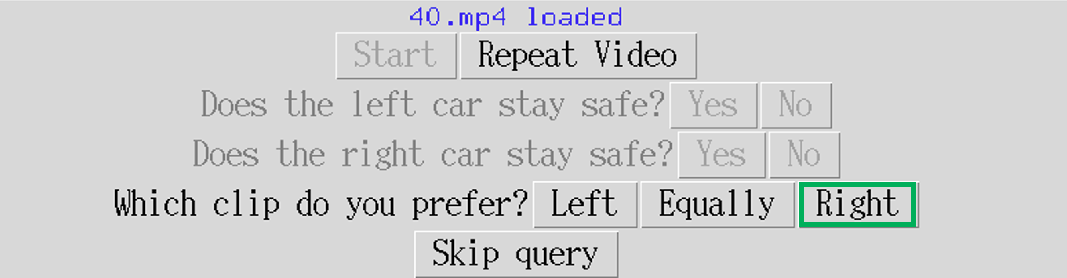}
    \caption{DROPJ --- third query for preference (since both \textit{Yes}). Ans: \textit{Right} $\rightarrow$ $\mu=0.25$.}
    \label{fig:gui_dropj3}
\end{subfigure}

\caption{GUIs in Car Racing of DROP (b), DROPe (c) and DROPJ (d--f) on the same preference query (a) for the single safety justification case \cite{kazantzidis2026safe}.}
\label{fig:gui_drop_drope_dropj}

\end{figure}

Finally, Figure \ref{fig:gui_multi} shows the GUI of DROPJ when multiple justifications are used in OCR. Here, a logic combining different activated justifications (e.g.\ a car shown to be both on top of a chuckhole and hitting another car) to shape $\mu$ could perhaps give interesting results. However, we continued with the reasonable logic (Equation \ref{eq:gen_just}) of shaping $\mu$ based on the most severe active justification. In this case, for efficiency in terms of user burden, we specify immediately the most severe justification's name and then give the preference. In the example, we answer \textit{Chuckhole} due to the right clip, and \textit{Left} since the agent-car stayed safe in that clip, making $\mu=w_{\mathit{chuck}}$. The opposite way of specifying first the preference and then the justification would have also been possible.

\begin{figure}
    \centering
    \includegraphics[width=0.6\textwidth]{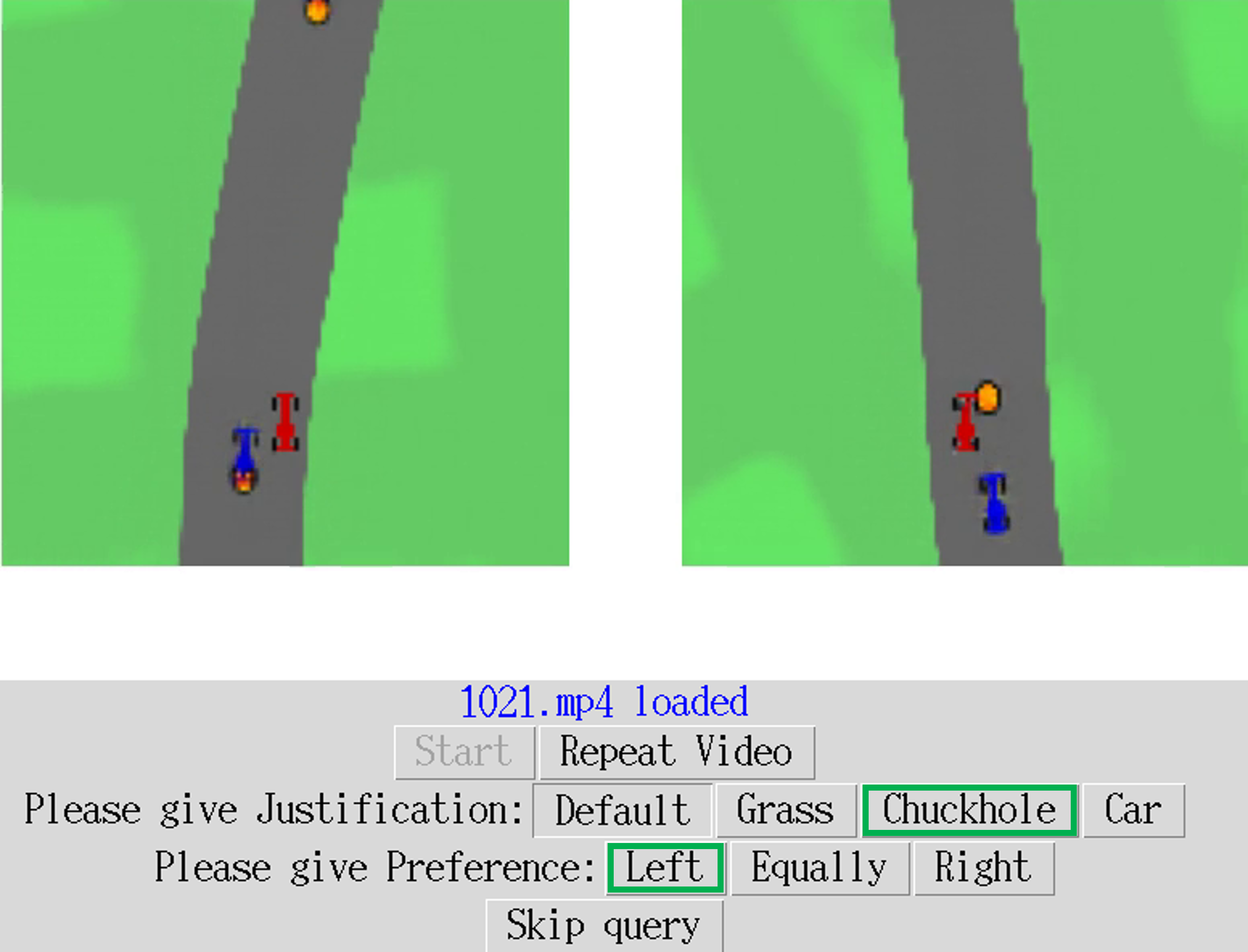}
    \caption{GUI in Obstacle Car Racing of DROPJ for the multiple justifications case \cite{kazantzidis2026safe}. \textit{Default} is always preselected for brevity (to often save a click). On the left, the agent-car (red) passed safely the blue car; and on the right, it passed the car but stepped on the chuckhole. Ans: (\textit{Chuckhole}, \textit{Left}) $\rightarrow$ $\mu=w_{\mathit{chuck}}$.}
    \label{fig:gui_multi}
\end{figure}

\pdfbookmark[2]
  {Model Predictive Control}
  {app-mpc}
\subsection*{Model Predictive Control}

After learning the reward model, we deploy immediately the agent with MPC \cite{garcia1989model}. We used the same sample-based MPC approach with \cite{reddy2020learning} and \cite{rahtz2022safe} described in Step 4 (Section \ref{sec:dropj} and Algorithm \ref{alg:dropj}). Given the computational resources, for all algorithms we generated $S=15$ sensible random action sequences. They include some predefined sequences such as drive straight, drive forward while turning left or right, brake, and two-stage randomly generated combinations of these, all within the horizon of the sequence (for example for a horizon of 15, a sequence could be: drive forward while turning left for 10 steps, and just drive straight for another 5 steps). A short experimentation was needed to tune the horizon of each method. ReQueST and DROS were found to operate best with $H=25$ and DROP, DROPe and DROPJ with $H=15$. Also, all methods re-planned every $R=4$ ($R=3$ generated similar results for most methods). Figure \ref{fig:mpc_dreams} shows the end of the MPC dreams of a set of sensible random action sequences of horizon $H=15$. The sequences are succinctly described in the figure: inside the brackets it is the action $[\mathit{steer}, \mathit{gas}, \mathit{brake}]$ with the ranges $\mathit{steer}\in[-1.0, 1.0]$, $\mathit{gas} \in [0.0, 1.0]$ and $\mathit{brake}\in[0.0, 1.0]$ and the number next to the brackets indicates the number of steps that that action was taken in the sequence. The action sequence of the trajectory highlighted in red has the highest predicted return, so the first $R=4$ actions of that sequence were taken, until MPC re-planned. We also found it slightly beneficial to reset (zero) the LSTM's memory every 50 steps, which can have several explanations, such as preventing the accumulation of errors or managing non-stationarity.

\begin{figure}
    \centering
    \includegraphics[width=1\textwidth]{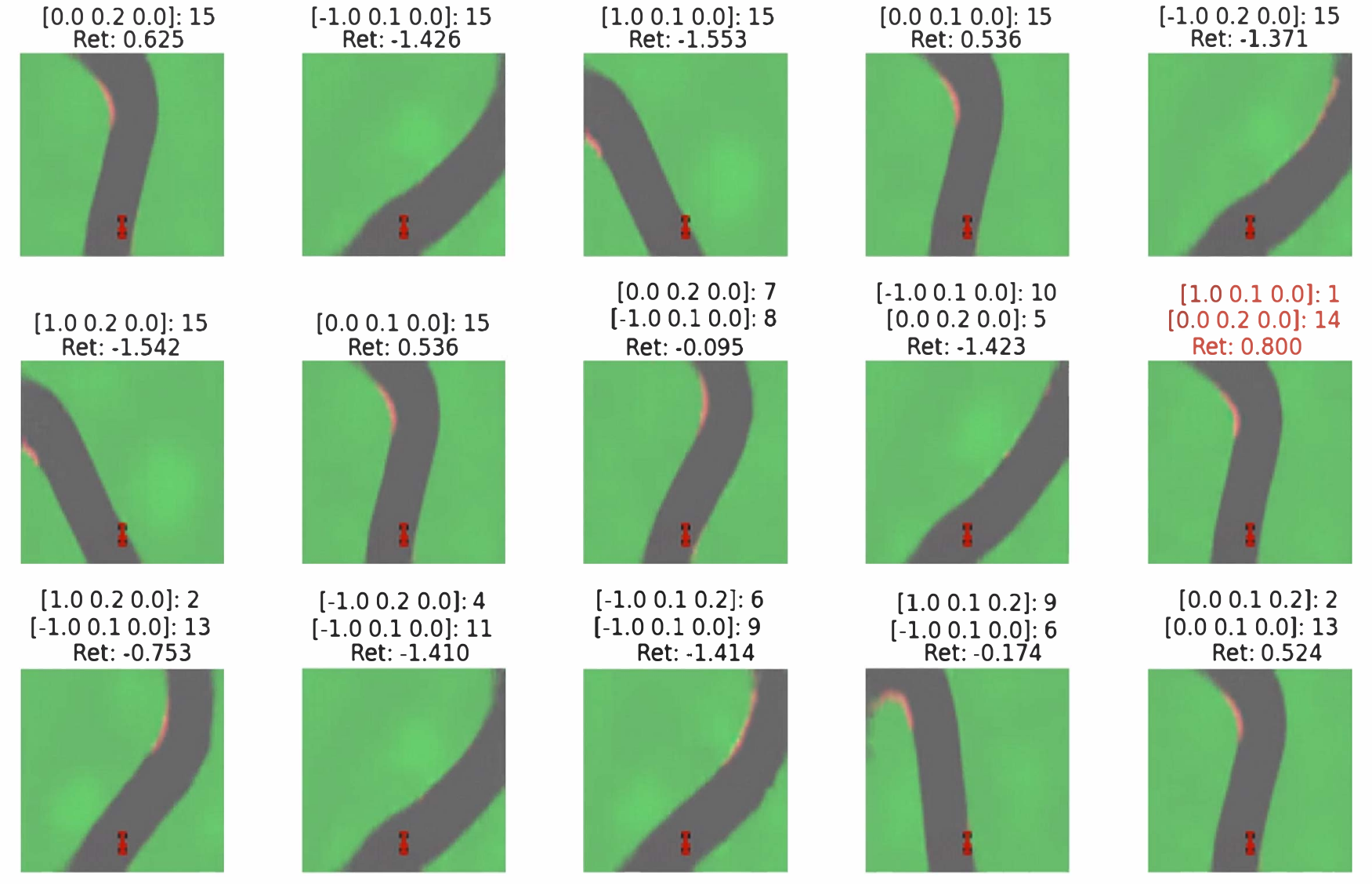}
    \caption{MPC dreams of a set of sensible random action sequences \cite{kazantzidis2026safe}.}
    \label{fig:mpc_dreams}
\end{figure}

%
%
\bibliographystyle{splncs04}
\bibliography{mybibliography}
%




\end{document}